\documentclass[11pt,a4paper]{article}

\usepackage{fullpage}		 
\usepackage[utf8]{inputenc} 
\usepackage[T1]{fontenc}    
\usepackage{lmodern}			 
\usepackage{hyperref}       
\usepackage{url}            
\usepackage{booktabs}       
\usepackage{amsthm}			 
\usepackage{amsmath}			 
\usepackage{amssymb}			 
\usepackage{amsfonts}       
\usepackage{mathtools}      
\usepackage{nicefrac}       
\usepackage{microtype}      
\usepackage{subfiles}       
\usepackage{bbold}          
\usepackage{xcolor}         
\usepackage{bm}             
\usepackage{enumitem}		 
\usepackage[numbers]{natbib}  
\usepackage[ruled,vlined]{algorithm2e} 
\usepackage{algorithmic}

\usepackage{xpatch}
\makeatletter
\xpatchcmd{\@thm}{\thm@headpunct{.}}{\thm@headpunct{}}{}{}
\makeatother

\makeatletter
\def\th@plain{%
  \thm@notefont{}
  \itshape 
}
\def\th@definition{%
  \thm@notefont{}
  \normalfont 
}
\makeatother

\newtheorem{theorem}{Theorem}
\newtheorem{definition}{Definition}
\newtheorem{lemma}[theorem]{Lemma}

\newtheorem{remark}{Remark}

\definecolor{myred}{HTML}{880000}
\definecolor{mygreen}{HTML}{008800}
\definecolor{myblue}{HTML}{000088}
\definecolor{linkblue}{HTML}{0000BB}

\hypersetup{
  colorlinks,
  linkcolor={myred},
  citecolor={myblue},
  urlcolor={linkblue}
}

\newcommand{\R}{\mathbb{R}}
\newcommand{\E}{\mathbb{E}}
\newcommand{\Eins}{\mathrm{\textbf{1}}}
\newcommand{\gauss}{\mathcal{N}}
\renewcommand{\O}{\mathcal{O}}
\renewcommand{\P}{\mathbb{P}}
\newcommand{\x}{\mathbf{x}}
\newcommand{\X}{\mathbf{X}}
\newcommand{\A}{\mathbf{A}}
\renewcommand{\S}{\mathcal{S}}

\title{A Continuous-Time Mirror Descent Approach to Sparse Phase Retrieval}
\author{
  Fan Wu,
  Patrick Rebeschini \\
	Department of Statistics, University of Oxford
}

\begin{document}

\maketitle

\begin{abstract}%
	We analyze continuous-time mirror descent applied to sparse phase retrieval, which is the problem of recovering sparse signals from a set of magnitude-only measurements. We apply mirror descent to the unconstrained empirical risk minimization problem (batch setting), using the square loss and square measurements. We provide a convergence analysis of the algorithm in this non-convex setting and prove that, with the hypentropy mirror map, mirror descent recovers any $k$-sparse vector $\mathbf{x}^\star\in\mathbb{R}^n$ with minimum (in modulus) non-zero entry on the order of $\| \mathbf{x}^\star \|_2/\sqrt{k}$ from $k^2$ Gaussian measurements, modulo logarithmic terms. This yields a simple algorithm which, unlike most existing approaches to sparse phase retrieval, adapts to the sparsity level, without including thresholding steps or adding regularization terms. Our results also provide a principled theoretical understanding for Hadamard Wirtinger flow \cite{WR20}, as Euclidean gradient descent applied to the empirical risk problem with Hadamard parametrization can be recovered as a first-order approximation to mirror descent in discrete time.
\end{abstract}

\section{Introduction}
\label{section:introduction}
Mirror descent \cite{NY83} is becoming increasingly popular in a variety of settings in optimization and machine learning. One reason for its success is the fact that mirror descent can be adapted to fit the geometry of the optimization problem at hand by choosing a suitable strictly convex potential function, the so-called mirror map. Mirror descent has been extensively studied for convex problems, and it is amenable to a general convergence analysis in terms of the Bregman divergence associated to the mirror map, e.g.\ \cite{AB10, ABL13, BT03, B15, CT93, NJLS09, S11}. 
There is a growing literature considering mirror descent in non-convex settings, e.g.\ \cite{DG19, DL15, GL13, GLZ16, KN19, KBTB15, MM10, ZH18, ZMBBG17, ZMBBG20}, and we contribute to this literature by analyzing continuous-time mirror descent in the non-convex problem of sparse phase retrieval.

Recently, there has been a surge of interest in investigating continuous-time solvers in a variety of settings in machine learning, for instance, in connection to implicit regularization, e.g.\ \cite{ADT20, AKT19, SPR18}, learning neural networks, e.g.\ \cite{CB18, MMM19, RV18, SS18}, and, more in general, to understand the foundations of algorithmic paradigms and to provide design principles for discrete-time algorithms used in practice, e.g.\ \cite{AW20, GWS20, RB12, VKR20, WJ98}. Convergence analyses for continuous-time algorithms are typically simpler and more transparent than those for their discrete-time counterparts, as they allow to focus on the main properties of the algorithm.

Phase retrieval is the problem of recovering a signal from the (squared) magnitude of a set of linear measurements. Such a task arises in many applications such as optics \cite{W63}, diffraction imaging \cite{BDPFSSV07} and quantum mechanics \cite{C06}, where detectors are able to measure intensities, but not phases. Due to the missing phase information, exploiting additional prior information often becomes necessary to ensure that the problem is well-posed. Common forms of prior information include assumptions on sparsity, non-negativity or the magnitude of the signal \cite{F82, JEH16}. Other approaches include introducing redundancy by oversampling random Gaussian measurements or coded diffraction patterns \cite{CLS15, CC15}.

Numerous strategies have been developed to exploit sparsity. One approach is to confine the search to the low-dimensional subspace of sparse vectors, either via a preliminary support recovery step as in the alternating minimization algorithm SparseAltMinPhase \cite{NJS15} or by updating the current estimated support in a greedy fashion as in GESPAR \cite{SBE14}. Another approach relies on the introduction of thresholding steps to enforce sparsity. This approach is typically found in non-convex optimization based algorithms such as thresholded Wirtinger flow (TWF) \cite{CLM16}, sparse truncated amplitude flow (SPARTA) \cite{WZGAC18}, compressive reweighted amplitude flow (CRAF) \cite{ZWGC18}, and sparse Wirtinger flow (SWF) \cite{YWW19}. Sparsity can also be promoted by augmenting the objective function with a regularization term. This is the approach taken in convex relaxation based methods such as compressive phase retrieval via lifting (CPRL) \cite{OYDS12} and SparsePhaseMax \cite{HV16}, which include an $\ell_1$ penalty, but also in PR-GAMP \cite{SR15}, which is an algorithm based on generalized message passing that uses a sparsity inducing prior. Hadamard Wirtinger flow (HWF) \cite{WR20} is an algorithm which performs gradient descent on the unregularized empirical risk using the Hadamard parametrization. This parametrization has been recently used to recover low-rank structures in sparse recovery \cite{H17, VKR19, ZYH19} and matrix factorization \cite{ACHL19, GWBNS17, LMZ18} under the restricted isometry property.

With the exception of HWF, the aforementioned methods rely on restricting the search to sparse vectors, thresholding steps or adding regularization terms to enforce sparsity. On the other hand, HWF does not require thresholding steps or added regularization terms to promote sparsity. Further, it has been empirically observed that HWF can recover sparse signals from a number of measurements comparable to those required by PR-GAMP, in particular, requiring fewer measurements than existing gradient based algorithms.
Despite these benefits, the work introducing HWF in \cite{WR20} has two main limitations: on the one hand, a full theoretical understanding of the algorithm that can explain its (convergence) behavior and, in particular, the reason why it adapts to the signal sparsity, is lacking. On the other hand, the algorithm has been empirically shown to have a sublinear convergence phase to the underlying signal, which would seem to lead to improved sample complexity at the expense of an increase in computational cost.

\subsection{Our contributions}
In this work, we provide a theoretical analysis of unconstrained mirror descent in continuous time applied to the unregularized empirical risk with the square loss for the problem of sparse phase retrieval with square measurements. With the hypentropy mirror map \cite{GHS20}, we prove that mirror descent recovers any $k$-sparse vector $\x^\star\in\R^n$ with minimum non-zero entry (in modulus) on the order of $\|\x^\star\|_2 / \sqrt{k}$ from $\widetilde{\O}(k^2)$ Gaussian measurements, where $\widetilde{\O}$ hides logarithmic terms. To the best of our knowledge, this is the first result on continuous-time solvers for (sparse) phase retrieval.

This provides a simple first-order method that relies neither on thresholding steps, nor on regularized objective functions. Without requiring knowledge of the sparsity level $k$, mirror descent \emph{adapts} to the sparsity level via the geometry defined by the hypentropy mirror map. This mirror map is parametrized by $\beta$, which is the only parameter in mirror descent and regulates the magnitude that off-support variables can attain while the algorithm runs. Our analysis shows that $\beta$ should be chosen smaller than a quantity depending on the signal size $\|\x^\star\|_2$ and the ambient dimension $n$. In particular, tuning of $\beta$ does \emph{not} require knowledge (or estimation) of the sparsity level $k$. We remark that estimating the signal size $\|\x^\star\|_2$ is easily done by considering the average observation size \cite{CLS15, WGE17}. We initialize mirror descent following the same initialization scheme proposed in \cite{WR20} for HWF. This initialization is independent of $\beta$ and only requires knowledge of a single coordinate on the support, which can be estimated from $\widetilde{\O}(k^2)$ Gaussian measurements \cite{WR20}. This initialization is much simpler than the schemes typically necessary for other non-convex formulations such as the spectral initialization in CRAF or the orthogonality-promoting initialization in SPARTA.

It was observed in \cite{VKR20} that gradient descent with the Hadamard parametrization can be seen as a first-order approximation to mirror descent with the hypentropy mirror map. Since HWF consists of running vanilla gradient descent on the unregularized empirical risk under the Hadamard parametrization, it can be treated as a discrete-time first-order approximation to the mirror descent algorithm we analyze. Hence, our work provides a principled theoretical understanding for HWF and addresses the first of the two limitations mentioned above on the analysis given in \cite{WR20}. Further, our investigation also reveals the connection between the initialization size in HWF (which corresponds to the mirror map parameter $\beta$ up to a squareroot) and the convergence speed of the algorithm. In particular, we first have an initial warm-up period, after which convergence towards the true signal $\x^\star$ is linear, up to a precision determined by the parameter $\beta$ and the signal dimension $n$. By choosing a sufficiently small initialization, any desired accuracy can be reached before entering the final sublinear stage, which is the main reason for the slow convergence of HWF observed in \cite{WR20}.

The property that enables mirror descent to deal with the non-convexity of the objective function is a weaker version of \emph{variational coherence} \cite{ZMBBG17, ZMBBG20}. While the variational coherence property as defined in \cite{ZMBBG17, ZMBBG20} precludes the existence of saddle points and is not satisfied in the sparse phase retrieval problem, we show that the defining inequality is satisfied along the trajectory of mirror descent, which is what allows us to establish the convergence analysis.

The literature on mirror descent is vast and growing, and a full overview is outside the scope of our work. 
Our contribution adds, in particular, to existing results on mirror descent in non-convex settings, which in general either require specific assumptions \cite{DL15, ZMBBG17, ZMBBG20}, guarantee convergence only to stationary points \cite{DD18, DG19, GL13, GLZ16, ZH18}, or are tailored to specific problems in the online learning setting \cite{KN19, KBTB15, MM10}, for instance. Recently, the connection between sparsity and mirror descent equipped with the hypentropy mirror map was established in \cite{VKR20}. While the aforementioned connection has been shown for linear models and kernels, which leads to a convex problem, our analysis demonstrates that this connection also extends to the non-convex problem of sparse phase retrieval.

\section{Preliminaries}
\label{section:preliminaries}
We first introduce some notation. We use boldface letters for vectors and matrices, normal font for real numbers, and, typically, uppercase letters for random and lowercase letters for deterministic quantities. For the clarity of the analytical results, this paper focuses on sparse phase retrieval with real signal and measurement vectors. Nevertheless, the algorithm also works in the complex case. 

\subsection{Mirror descent}
\label{subsection:mirror_descent}
We first give a brief description of the unconstrained mirror descent algorithm; more details can be found in \cite{B15}. The key object defining the geometry of the algorithm is the \textit{mirror map}.
\begin{definition}
Let $\mathcal{D}\subset \R^n$ be a convex open set. We say that $\Phi:\mathcal{D}\rightarrow \R$ is a mirror map if it is strictly convex, differentiable and its gradient takes all possible values, i.e.\ $\{\nabla \Phi(\x):\x\in\mathcal{D}\} = \R^n$.
\end{definition}
We consider unconstrained mirror descent, i.e.\ $\mathcal{D} = \R^n$. Let $F:\R^n\rightarrow \R$ be a (possibly non-convex) function, for which we seek a global minimizer. Mirror descent is characterized by the mirror map $\Phi$ and an initial point $\X(0)$, and is, in continuous time, defined by the identity \cite{NY83}
\begin{equation}
\label{eq:mirror_descent_continuous}
\frac{d}{dt}\X(t) = -\left(\nabla^2\Phi(\X(t))\right)^{-1}\nabla F(\X(t)).
\end{equation}
An important quantity in the analysis of mirror descent is the \textit{Bregman divergence} associated to a mirror map $\Phi$, which is given by 
\begin{equation*}
D_{\Phi}(\x, \mathbf{y}) = \Phi(\x) - \Phi(\mathbf{y}) - \nabla \Phi(\mathbf{y})^T(\x-\mathbf{y}).
\end{equation*}
The following equality can be derived from a quick calculation:
\begin{equation}
\label{eq:breg_derivative}
\frac{d}{dt}D_{\Phi}(\x',\X(t)) = \left\langle\nabla F(\X(t)), \x' - \X(t) \right\rangle,
\end{equation}
where $\x'\in\R^n$ is any reference point. In particular, when the objective function $F$ is convex, equation (\ref{eq:breg_derivative}) shows that mirror descent monotonically decreases the Bregman divergence to any minimizer of $F$. In non-convex settings, the inner product in (\ref{eq:breg_derivative}) was used to define the notion of variational coherence, which is the assumption under which convergence of a stochastic version of mirror descent towards a minimizer of $F$ has been shown \cite{ZMBBG17, ZMBBG20}.

\subsection{Sparse phase retrieval}
\label{subsection:sparse_phase_retrieval}
 The goal in sparse phase retrieval is to reconstruct an unknown $k$-sparse vector $\x^\star\in\R^n$ from a set of quadratic measurements $Y_j = (\A_j^T\x^\star)^2$, $j=1,...,m$, where the measurement vectors $\A_j\sim \gauss(0,\mathbf{I}_n)$ are i.i.d.\ and observed.

A well-established approach to estimating the signal $\x^\star$ is based on non-convex optimization \cite{CLM16, WZGAC18, YWW19, ZWGC18}. In particular, given observations $(\A_1,Y_1),...,(\A_m,Y_m)$, the goal becomes minimizing the (non-convex) empirical risk
\begin{equation}\label{eq:objective_function}
F(\x) = \frac{1}{4m}\sum_{j=1}^m\left((\A_j^T\x)^2 - Y_j\right)^2.
\end{equation}
It is worth mentioning that a different, amplitude-based risk function has also been considered \cite{WZGAC18, ZWGC18}. However, in that case the objective function becomes non-smooth, as the terms $(\A_j^T\x)^2$ are replaced with $|\A_j^T\x|$, and the analysis via mirror descent appears more challenging.

As a non-convex function, the function $F$ in (\ref{eq:objective_function}) could potentially have many local minima and saddle points, and even global minima different from $\x^\star$. It has been shown that if we have $m\ge 4k-1$ Gaussian measurements, then, with high probability, $\x^\star$ is (up to a global sign) the sparsest minimizer of $F$, that is $\{\pm \x^\star\}=\operatorname{argmin}_{\x:\|\x\|_0\le k} F(\x)$ \cite{LV13}. In order to tackle these difficulties, previous methods such as SPARTA \cite{WZGAC18} and CRAF \cite{ZWGC18} employed a sophisticated spectral or orthogonality-promoting initialization scheme, which produces an initial estimate close enough to the signal $\x^\star$, followed by thresholded gradient descent updates, which confine the iterates to the low-dimensional subspace of $k$-sparse vectors. 

\section{The algorithm}
\label{section:algorithm}
We consider unconstrained mirror descent in continuous time given by (\ref{eq:mirror_descent_continuous}) and applied to the objective function (\ref{eq:objective_function}), equipped with the mirror map \cite{GHS20}
\begin{equation}
\label{eq:mirror_map}
\Phi(\x) = \sum_{i=1}^n\left(x_i \operatorname{arcsinh}\left(\frac{x_i}{\beta}\right) - \sqrt{x_i^2 + \beta^2}\right),
\end{equation}
for some parameter $\beta>0$. A discussion on the choice of the parameter $\beta$ is given in Section \ref{section:main_result}. The Hessian is given by the diagonal matrix with entries
$
\nabla^2\Phi(\x)_{ii} = (x_i^2 + \beta^2)^{-1/2}.
$

For the initialization, following the approach outlined in \cite{WR20}, we set
\begin{align}
\label{eq:initialization}
X_i(0) = \begin{cases}
\frac{1}{\sqrt{3}} \cdot \sqrt{\frac{1}{m}\sum_{j=1}^mY_j} \qquad & i=i_0\\
0 & i\neq i_0
\end{cases}
\end{align}
for a coordinate $i_0$ on the support of the signal $\x^\star$, i.e.\ $x^\star_{i_0} \neq 0$. Here, the term $\sqrt{\frac{1}{m}\sum_{j=1}^mY_j}$ is an estimator for the magnitude $\|\x^\star\|_2$, see e.g.\ \cite{CLS15, WGE17}. By Lemma 1 of \cite{WR20}, it is possible to estimate a coordinate in the true support with high probability from $\widetilde{\O}(k(x^\star_{max})^{-2})$ Gaussian measurements, where $x^\star_{max} = \max_i |x^\star_i|/\|\x^\star\|_2$. We develop theory for mirror descent with $\widetilde{\O}(k^2)$ samples, which is the worst case of the bound $\widetilde{\O}(k(x^\star_{max})^{-2})$, as $\x^\star$ is a $k$-sparse vector, and hence $x^\star_{max} \ge 1/\sqrt{k}$.

In order to gain some intuitive understanding of why the initialization (\ref{eq:initialization}) is suitable, it will be helpful to consider the limiting case $m=\infty$. The following explanation to motivate the choice of initialization is taken from \cite{WR20}. The gradient of the empirical risk is given by
\begin{equation}
\label{eq:gradient}
\nabla F(\x) = \frac{1}{m}\sum_{j=1}^m \left((\A_j^T\x)^2 - (\A_j^T\x^\star)^2\right)(\A_j^T\x)\A_j.
\end{equation}
A straightforward calculation yields the following expression for the population gradient, which is defined as the expectation of $\nabla F(\x)$ and corresponds to the limiting case $m=\infty$.
\begin{equation*}
\nabla f(\x) = \E[\nabla F(\x)] = \left(3\|\x\|_2^2 - \|\x^\star\|_2^2\right) \x - 2(\x^T\x^\star) \x^\star.
\end{equation*}
We see that, if $m=\infty$, mirror descent has three types of fixed points: a local maximum at $\x^{(1)} = \mathbf{0}$, saddle points at any $\x^{(2)}$ satisfying $\|\x^{(2)}\|_2^2 = \frac{1}{3}\|\x^\star\|_2^2$ and $(\x^{(2)})^T\x^\star = 0$, and the global minima $\x^{(3)} = \pm \x^\star$. Although the landscape might be less well-behaved if $m$ is finite, this consideration provides an intuitive explanation for the choice of initialization (\ref{eq:initialization}), namely, that this initialization is suitably far away from the saddle points of the population risk $f$.

Compared to initialization schemes typically employed in existing non-convex optimization based approaches to sparse phase retrieval, such as the spectral initialization used in CRAF \cite{ZWGC18} and the orthogonality promoting initialization used in SPARTA \cite{WZGAC18}, the initialization we use is much simpler: this method requires only a single coordinate on the support, while the aforementioned initialization schemes require estimating the full support followed by a spectral or orthogonality-promoting scheme.

Finally, the fact that the function $f$ has saddle points means that variational coherence as defined in \cite{ZMBBG17, ZMBBG20} is not satisfied in our problem. Further, the results in \cite{ZMBBG17, ZMBBG20} are formulated for the online setting, where by design one has continuous access to independent observations, while we consider batch mirror descent with a fixed number of observations.

\section{Main result}
\label{section:main_result}
In this section, we present the main result of this paper. We show that continuous-time mirror descent recovers any $k$-sparse signal $\x^\star\in\R^n$ with $x^\star_{min} \ge \Omega(1/\sqrt{k})$ with high probability from $\widetilde{\O}(k^2)$ Gaussian measurements, where we denote $x^\star_{min} := \min_{i:x^\star_i\neq 0} |x^\star_i|/\|\x^\star\|_2$. Since $\x^\star$ cannot be distinguished from $-\x^\star$ using phaseless measurements, we consider, for any vector $\x\in \R^n$, the distance $\operatorname{dist}(\x,\x^\star) = \min\{\|\x-\x^\star\|_2, \|\x+\x^\star\|_2\}$. To simplify the presentation in what follows, we assume that the initialization in (\ref{eq:initialization}) satisfies $x^\star_{i_0}>0$, with which we will show convergence to $\x^\star$; otherwise, we can show that the algorithm converges to $-\x^\star$.

The following lemma characterizes the relationship between the Bregman divergence $D_{\Phi}(\x^\star,\x)$ associated to the hypentropy mirror map (\ref{eq:mirror_map}) and the $\ell_2$ norm $\|\x-\x^\star\|_2$. For a vector $\x\in \R^n$ and a subset of coordinates $\mathcal{S}\subset \{1,...,n\}$, we write $\x_{\mathcal{S}} = (x_i)_{i\in \mathcal{S}}\in \R^{|\mathcal{S}|}$.
\begin{lemma} \label{lemma1}
Let $\x^\star\in \R^n$ be any $k$-sparse vector with  $x^\star_{min} \ge c/\sqrt{k}$ for some constant $c>0$. Let $\mathcal{S}=\{1\le i\le n:x^\star_i\neq 0\}$ be its support, and let $\Phi$ be as in (\ref{eq:mirror_map}) with parameter $\beta>0$.
\begin{itemize}
\item For any vector $\x\in \R^n$, we have
\begin{equation}\label{eq:lemma1lb}
\|\x-\x^\star\|_2^2 \le 2\sqrt{\max\{\|\x\|_{\infty}^2,\|\x^\star\|_\infty^2\}+\beta^2} \cdot D_{\Phi}(\x^\star,\x).
\end{equation}
\item Let $\x\in \R^n$ be any vector with $x_ix^\star_i\ge 0$ (no mismatched sign) and $|x_i| \ge \frac{1}{2}|x^\star_i|$ for all $i=1,...,n$. We have
\begin{equation}\label{eq:lemma1ub}
D_{\Phi}(\x^\star, \x) \le \frac{\sqrt{k}}{c\|\x^\star\|_2}\|\x_{\mathcal{S}}-\x^\star_{\mathcal{S}}\|_2^2 + \|\x_{\mathcal{S}^c}\|_1.
\end{equation}
\end{itemize}
\end{lemma}
The bound in (\ref{eq:lemma1lb}) shows that when a vector $\x$ is close to $\x^\star$ in terms of the Bregman divergence $D_{\Phi}$, then $\x$ is also close to $\x^\star$ in the $\ell_2$ sense. This means that if we are interested in convergence with respect to the $\ell_2$ norm, we can consider the Bregman divergence $D_{\Phi}$ as a proxy, and we write $\operatorname{dist}_{\Phi}(\x^\star,\x) = \min\{D_{\Phi}(\x^\star,\x), D_{\Phi}(-\x^\star,\x) \}$. The bound in (\ref{eq:lemma1ub}) shows that, for certain vectors $\x\in\R^n$ of interest, the Bregman divergence $D_{\Phi}$ can also be upper bounded in terms of a combination of the $\ell_1$ and $\ell_2$ norms. Note that, because of the assumption $|x_i| \ge \frac{1}{2}|x^\star_i|$, the bound in (\ref{eq:lemma1ub}) does not depend on the parameter $\beta$.
Details of the proof can be found in Appendix \ref{appendix:b}.

We can now formulate our main result. The constants $c,c_1,c_2,c_3, c_4$ and $c_5$ mentioned in the following theorem are universal constants, and explicit expressions for these constants are given in the proof in Appendix \ref{appendix:d}.
\begin{theorem}\label{theorem}
Let $\x^\star\in \R^n$ be any $k$-sparse vector with $x^*_{min} \ge c/\sqrt{k}$ for some constant $c>0$, and let $\mathcal{S}=\{1\le i\le n:x^*_i\neq 0\}$ be its support. There exist constants $c_1,c_2,c_3,c_4,c_5>0$ such that the following holds. Let $m\ge c_1\max\{k^2\log^2n,\;\log^5n\}$, and let $\X(t)$ be given by the continuous time mirror descent equation (\ref{eq:mirror_descent_continuous}) with mirror map (\ref{eq:mirror_map}) and initialization (\ref{eq:initialization}) with $\beta\le c_2 \|\x^\star\|_2/n^4$. Let $\delta = c_3n\sqrt{\frac{\beta}{\|\x^\star\|_2}}$ and $T_2(\delta) = \inf\{t>0: \operatorname{dist}_{\Phi}(\x^\star, \X(t))\le 2\delta \|\x^\star\|_2)\}$.

Then, with probability at least $1-c_4n^{-10}$, there is a $T_1(\beta)\le c_5k\log\frac{\|\x^\star\|_2}{\beta}\cdot \log (k\log \frac{\|\x^\star\|_2}{\beta}) \cdot \|\x^\star\|_2^{-3}$ such that
\begin{equation}
\label{thm:convergence}
\frac{\operatorname{dist}_{\Phi}(\x^\star,\X(t))}{\|\x^\star\|_2} \le \frac{6\sqrt{k}}{c}\cdot \exp\left(-\frac{c\|\x^\star\|_2^3}{4\sqrt{k}} (t - T_1(\beta))\right) \quad \text{for all } T_1(\beta)\le t\le T_2(\delta).
\end{equation}
Further, for all $t\le T_2(\delta)$, we have
\begin{equation}
\label{thm:off_support}
\|\X_{\mathcal{S}^c}(t)\|_1 \le \delta \|\x^\star\|_2.
\end{equation}
\end{theorem}
The high-level idea of the proof of Theorem \ref{theorem} is as follows. First, considering the limiting case $m=\infty$, i.e.\ assuming we had access to the population gradient $\nabla f$, we show that mirror descent is variationally coherent along its trajectory, namely $\frac{d}{dt}D_{\Phi}(\x^\star, \X(t)) = \langle \nabla f(\X(t)), \x^\star - \X(t)\rangle < 0$. In order to prove convergence for a finite $m$, we show that, if $m\ge \widetilde{\O}(k^2)$, then the empirical gradient $\nabla F$ is sufficiently close to its expectation $\nabla f$ using concentration results for Lipschitz functions and bounded random variables. A detailed proof can be found in Appendix \ref{appendix:d}. 

\paragraph{Convergence of mirror descent} The bound (\ref{thm:convergence}) in Theorem \ref{theorem} indicates that the convergence of mirror descent (measured by the Bregman divergence $D_{\Phi}$) can be described as follows: in an initial warm-up period, the Bregman divergence decreases to $\sqrt{k}\|\x^\star\|_2$ (up to constants). Then, convergence is linear up to a precision determined by the mirror map parameter $\beta$ and the dimension of the signal $n$. This behavior can be explained as follows. The initial warm-up period is caused by the fact that, as manifested in the proof of Theorem \ref{theorem}, the initial Bregman divergence $D_{\Phi}(\x^\star, \X(0))$ scales like $\sqrt{k}\log\frac{\|\x^\star\|_2}{\beta}$.
The following linear convergence stage corresponds to variables $X_i(t)$ on the support being fitted; to establish linear convergence, we crucially use the bound (\ref{eq:lemma1ub}) of Lemma \ref{lemma1} along with the fact that the second term $\|\X_{\mathcal{S}^c}(t)\|_1$ is negligibly small compared to $\|\X_{\mathcal{S}}(t)-\x_{\mathcal{S}}^\star\|_2^2$. 

\paragraph{Role of the mirror map parameter}
In Theorem \ref{theorem}, the role of the parameter $\beta$ is to ensure that off-support variables stay sufficiently small, cf.\ (\ref{thm:off_support}). In Theorem \ref{theorem}, we require $\beta\le \O(\|\x^\star\|_2/n^4)$. Note that both this requirement and the bound (\ref{thm:off_support}) are pessimistic and not sharp in general. The important property is that the bound on $\|\X_{\mathcal{S}^c}(t)\|_1$ depends polynomially on the parameter $\beta$.  
The price we pay for choosing a small $\beta$ is a longer warm-up period, whose length scales logarithmically in the parameter $\beta$ (see the definition of $T_1(\beta)$). 
In practice, we would simply choose a very small $\beta$ (e.g.\ $10^{-10}$), as the improvement in precision up to which we have linear convergence scales polynomially in $\beta$, while the price we pay in terms of a longer warm-up period only scales logarithmically in $\beta$. A similar trade-off between statistical accuracy and computaional cost with respect to the choice of initialization has been previously observed in \cite{VKR19}.

\paragraph{Scaling with signal magnitude}
When analyzing the convergence speed of continuous-time mirror descent equipped with the hypentropy mirror map for sparse phase retrieval, $t\|\x^\star\|_2^3$ is a natural quantity to consider. Recall that in sparse phase retrieval, the goal is to recover a signal $\x^\star$ from a set of phaseless measurements $\{(\A_j^T\x^\star)^2\}$. This problem is equivalent to the alternative problem of recovering the vector $a\x^\star$ from observations $\{a^2(\A_j^T\x^\star)^2\}$, for any $a\neq 0$. However, in the alternative problem, $\X(t)$ is not replaced by $a\X(t)$, and $\frac{d}{dt}\X(t)$ not by $a\frac{d}{dt}\X(t)$. 
If we replace $\x^\star$ by $a\x^\star$, the natural choice for the mirror map parameter becomes $a\beta$, as our results depend on the parameter $\beta$ only via the ratio $\beta/\|\x^\star\|_2$. Recalling the initialization (\ref{eq:initialization}), we see that also $\X(0)$ is replaced by $a\X(0)$ in the alternative problem. This means that, in the definition of $\frac{d}{dt}\X(t)$ (\ref{eq:mirror_descent_continuous}), the inverse Hessian $(\nabla^2\Phi(\X(t)))^{-1}$ is multiplied by $a$, while the gradient $\nabla F(\X(t))$ is multiplied by $a^3$, cf.\ (\ref{eq:gradient}). Hence, in the alternative problem formulation, $\frac{d}{dt}\X(t)$ is replaced by $a^4\frac{d}{dt}\X(t)$, which makes $t\|\x^\star\|_2^3$ the right quantity to consider for the convergence speed to stay unchanged.

When considering the algorithm in discrete time, this suggests that the step size should scale like $\|\x^\star\|_2^{-3}$. A similar observations has been made in the case of gradient descent for phase retrieval \cite{MWCC18}, where the step size scales as $\|\x^\star\|_2^{-2}$. We have an extra factor $\|\x^\star\|_2^{-1}$ because of the mirror map.
\begin{remark}[On sample complexity]
Up to logarithmic term, the sample complexity in Theorem \ref{theorem} matches that of existing results \cite{CLM16, NJS15, OYDS12, WZGAC18, YWW19}. We typically have $k^2\log^2 n > \log^5 n$ in regimes of interest, so that our sample complexity bound reads $\O(k^2\log^2n)$; the factor $\log^2 n$ is likely an artifact of the our proof technique, and we expect that it is possible to improve the bound to $\O(k^2\log n)$. The empirical results of \cite{WR20} suggest that the sample complexity of HWF, which is closely related to mirror descent as we will see in the next section, depends on the maximum signal component $\max_i |x^\star_i|/\|\x^\star\|_2$. The main bottleneck to establishing such a dependence in our theory seems to be the dependency between the estimates $\mathbf{X}(t)$ and the measurement vectors $\{\A_j\}$, and is likely to require tools different from the ones we use to prove Theorem \ref{theorem}.
\end{remark}

\section{Connection with Hadamard Wirtinger flow}
\label{section:connection_with_hwf}
In discrete time, it has been shown that the exponentiated gradient algorithm with positive and negative weights (EG$\pm$) \cite{KW97} without normalization is equivalent to mirror descent equipped with the hypentropy mirror map \cite{GHS20}. This equivalence has been used in \cite{VKR20} to recover results on implicit regularization in linear models using tools from the mirror descent literature. Since HWF performs Euclidean gradient descent on the empirical risk with Hadamard parametrization, HWF can be recovered as a discrete-time first-order approximation to the mirror descent algorithm we analyzed in Section \ref{section:main_result}. Hence, the convergence behavior established in Theorem \ref{theorem} for continuous-time mirror descent might guide the development of analogous guarantees for HWF. In particular, our analysis suggests a principled approach to address the slow convergence of HWF pointed out in \cite{WR20}.

First, consider the following version of the exponentiated gradient algorithm in continuous time:
\begin{equation}
\label{eq:exponentiated_gradient_continuous}
\begin{gathered}
\X(t) = \mathbf{U}(t) - \mathbf{V}(t) \\
\frac{d}{dt}\mathbf{U}(t) = -\mathbf{U}(t)\odot \nabla F(\X(t)), \qquad \frac{d}{dt}\mathbf{V}(t) = \mathbf{V}(t)\odot \nabla F(\X(t)) 
\end{gathered}
\end{equation}
with initialization, writing $\hat{\theta} = (\sum_{j=1}^mY_j/m)^{\frac{1}{2}}$ for the estimate of the signal size $\|\x^\star\|_2$, 
\begin{align}\label{eq:initialization2}
U_i(0) = \begin{cases} \frac{\hat{\theta}}{2\sqrt{3}} + \sqrt{\frac{\hat{\theta}^2}{12}+ \frac{\beta^2}{4}} \quad &i=i_0 \\ \frac{\beta}{2} & i\neq i_0 \end{cases}, \qquad V_i(0) = \begin{cases} -\frac{\hat{\theta}}{2\sqrt{3}} + \sqrt{\frac{\hat{\theta}^2}{12}+ \frac{\beta^2}{4}} \quad &i=i_0 \\ \frac{\beta}{2} & i\neq i_0 \end{cases},
\end{align}
where $i_0$ is defined as in (\ref{eq:initialization}) and the notation $\odot$ denotes the elementwise Hadamard product.
Similar to the discrete case \cite{GHS20}, a brief computation shows that the exponentiated gradient algorithm EG$\pm$ (\ref{eq:exponentiated_gradient_continuous}) with initialization (\ref{eq:initialization2}) is equivalent to mirror descent (\ref{eq:mirror_descent_continuous}) with initialization (\ref{eq:initialization}). We provide the details in Appendix \ref{appendix:a}. In particular, this reveals that the parameter $\beta$ in the hypentropy mirror map can be interpreted as the initialization size (with a factor $\frac{1}{2}$) in the exponentiated gradient formulation.

In discrete time, the exponentiated gradient algorithm (\ref{eq:exponentiated_gradient_continuous}) reads
\begin{equation}
\label{eq:exponentiated_gradient_discrete}
\begin{gathered}
\X(t) = \mathbf{U}(t) - \mathbf{V}(t) \\
\mathbf{U}(t+1) = \mathbf{U}(t)\odot \exp\left(-\eta\nabla F(\X(t))\right), \qquad \mathbf{V}(t+1) = \mathbf{V}(t)\odot \exp\left(\eta\nabla F(\X(t))\right),
\end{gathered}
\end{equation}
with the same initialization (\ref{eq:initialization2}), where $\eta>0$ is the step size. Noting that $e^x \approx 1 + x$, the update (\ref{eq:exponentiated_gradient_discrete}) can be approximated by (with the step size $\eta$ rescaled by a factor $4$)
\begin{equation}
\label{eq:hwf}
\begin{gathered}
\X(t) = \mathbf{U}(t)\odot \mathbf{U}(t) - \mathbf{V}(t)\odot \mathbf{V}(t) \\
\mathbf{U}(t+1) = \mathbf{U}(t)\odot \left(\Eins_n - 2\eta\nabla F(\X(t))\right), \qquad \mathbf{V}(t+1) = \mathbf{V}(t)\odot \left(\Eins_n + 2\eta\nabla F(\X(t))\right),
\end{gathered}
\end{equation}
where $\Eins_n\in \R^n$ denotes the vector of all ones. This is exactly the update of HWF with slightly different initial values $U_{i_0}(0)$ and $V_{i_0}(0)$ in (\ref{eq:initialization2}) compared to \cite{WR20}. Note that $U_i$ and $V_i$ in (\ref{eq:hwf}) correspond to the square root of $U_i$ and $V_i$ in (\ref{eq:exponentiated_gradient_discrete}), respectively. The trajectory of HWF with these two initializations is essentially identical, and we only report the results using the initialization (\ref{eq:initialization2}).

The convergence of HWF observed in \cite{WR20} matches the behavior suggested by Theorem \ref{theorem}: first, the estimate $\X(t)$ barely changes (in the $\ell_2$ sense) during the initial warm-up period, which is followed by a stage of linear convergence towards the signal $\x^\star$, after which the convergence slows down. Further, Theorem \ref{theorem} implies that the precision up to which convergence is linear is controlled by the mirror map parameter $\beta$ or, equivalently, by the initialization size in HWF. This means that, by choosing $\beta$ sufficiently small, we can avoid the final stage where convergence is slow, which is the stage mainly responsible for the high number of iterations needed to reach a given precision $\epsilon$.

In the following, we present simulations showing how the parameter $\beta$ affects the convergence of HWF. We note that the trajectories of EG$\pm$ (\ref{eq:exponentiated_gradient_discrete}) and HWF (\ref{eq:hwf}) are essentially identical, so we only show the results for HWF. We run HWF as proposed in \cite{WR20}, with initialization given by the square root of the values in (\ref{eq:initialization2}).
The index $i_0$ in (\ref{eq:initialization2}) is estimated by choosing the largest instance in $\{\sum_{j=1}^mY_jA_{ji}^2\}$ as proposed in \cite{WR20}. For the step size $\eta$, we follow \cite{WR20} and choose $\eta = 0.1$. As discussed in Section \ref{section:main_result}, we would set $\eta = 0.1 / \|\x^\star\|_2^3$ if $\|\x^\star\|_2\neq 1$, and, if $\|\x^\star\|_2$ is unknown, it can be reliably estimated by $\sqrt{\frac{1}{m}\sum_{j=1}^mY_j}$ \cite{CLS15, WGE17}.

The setup for the simulations is as follows. We generate a $10$-sparse signal vector $\x^\star\in \R^{50000}$ by first drawing $\x^\star\sim \gauss(\mathbf{0}, \mathbf{I}_{50000})$, then setting $49990$ random entries of $\x^\star$ to zero, and finally normalizing the vector to $\|\x^\star\|_2=1$. We sample $m=1000$ Gaussian measurement vectors $\A_j \sim \gauss(\mathbf{0}, \mathbf{I}_{50000})$ and generate phaseless measurements $Y_j = (\A_j^T\x^\star)^2$. 

We evaluate the relative error as well as the relative Bregman divergence from the solution set $\{\pm \x^\star\}$, given by
\begin{equation*}
\frac{\operatorname{dist}(\X(t), \x^\star)}{\|\x^\star\|_2}, \qquad \text{and} \qquad \frac{\operatorname{dist}_{\Phi}(\x^\star, \X(t))}{\|\x^\star\|_2}.
\end{equation*}
Figure \ref{fig:figure1} (left) shows that HWF exhibits the behavior suggested by our theory, namely that, after an initial warm-up stage, convergence is linear up to a precision depending on $\beta$. As we decrease the parameter $\beta$ from $10^{-6}$ to $10^{-14}$, the length of the initial warm-up stage increases as $\log\frac{1}{\beta}$. Note that in this first stage we see repeated drops and plateaus. The use of the hypentropy mirror map means that coordinates $X_i(t)$ that are small in magnitude can only change slowly, as then also the inverse Hessian $(\nabla^2 \Phi(\X))^{-1}_{ii} = (X_i^2+\beta^2)^{\frac{1}{2}}$ is small, cf.\ (\ref{eq:mirror_descent_continuous}). Informally, each drop corresponds to one coordinate $X_i(t)$ becoming large, while each plateau corresponds to all ``large'' coordinates reaching a stable state, at which point $\X(t)$ barely moves as the other ``small'' coordinates only change very slowly. This is followed by a second stage where the error decreases linearly up to a precision which depends polynomially on $\beta$ (note the log-scale of the $y$-axis). Finally, the convergence slows down after this precision has been reached. In order to reach a given precision $\epsilon>0$, it is preferable to choose the parameter $\beta$ sufficiently small so that convergence is linear up to precision $\epsilon$. This way, we avoid the final slow convergence stage, while the number of iterations spent in the initial warm-up stage only increases as $\log\frac{1}{\beta}$.

The right plot of Figure \ref{fig:figure1} shows a similar behavior when we consider the Bregman divergence $D_{\Phi}(\x^\star,\X(t))$. The only difference is in the first stage, where the Bregman divergence decreases at a constant rate without plateaus. As $\beta$ decreases, more iterations are spent in the first stage because the initial Bregman divergence $D_{\Phi}(\x^\star, \X(0))$ increases like $\log \frac{1}{\beta}$.
\begin{figure}[t]
  \centering
  \includegraphics[width=0.865\textwidth]{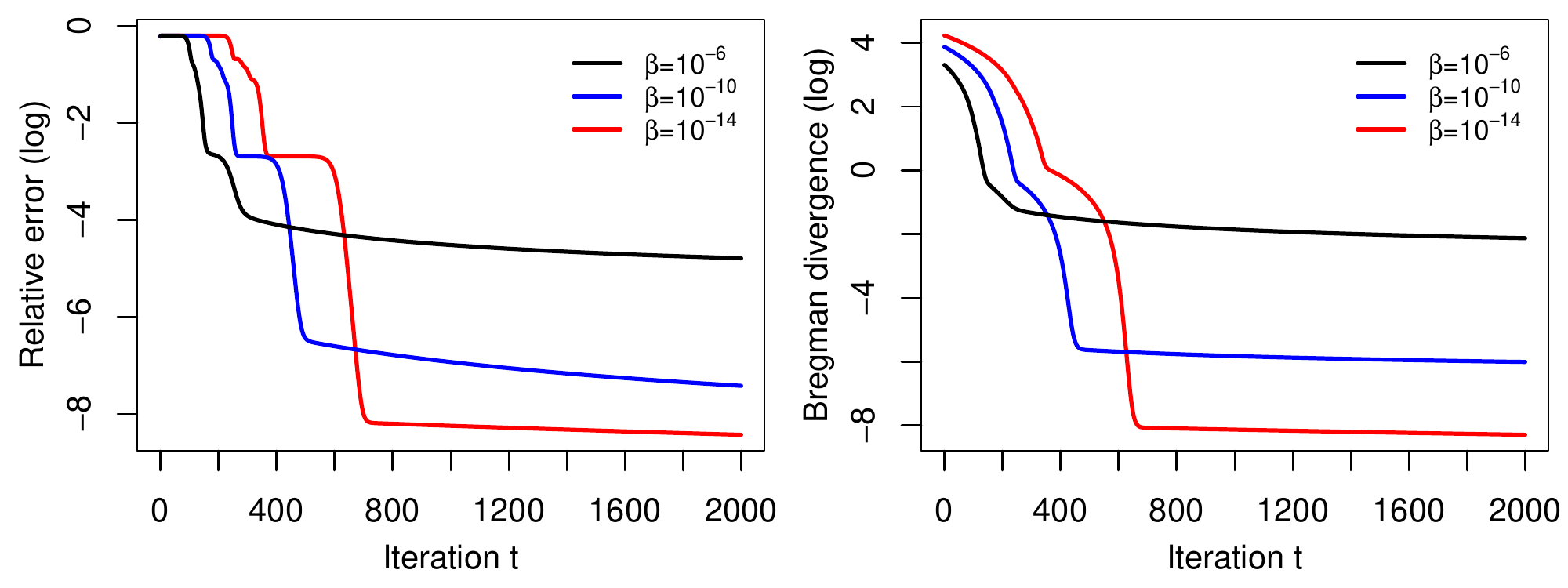}
  \caption{Convergence behavior of HWF with $n=50000$, $m=1000$ and $k=10$. Left: relative error (log-scale) of HWF for $\beta = 10^{-6}$ (black line), $\beta = 10^{-10}$ (blue line) and $\beta = 10^{-14}$ (red line). Right: Bregman divergence (log-scale) for the same values of $\beta$. In all experiments, the same signal vector $\x^\star$ and measurement vectors $\{\A_j\}$ were used.}
	\label{fig:figure1}
\end{figure}

\section{Conclusion}
\label{section:conclusion}
We provided a convergence analysis of continuous-time mirror descent applied to sparse phase retrieval. We proved that, equipped with the hypentropy mirror map, mirror descent recovers any $k$-sparse signal $\x^\star\in \R^n$ with $x^\star_{min} = \Omega(1/\sqrt{k})$ from $\widetilde{\O}(k^2)$ Gaussian measurements. This yields a simple algorithm, which, unlike most existing methods, does not require thresholding steps or added regularization terms to enforce sparsity. Further, as HWF can be recovered as a discrete-time first-order approximation to the mirror descent algorithm we analyzed, our results provide a principled theoretical understanding of HWF. In particular, our continuous-time analysis suggests how the initialization size in HWF affects convergence, and that choosing the initialization size sufficiently small can result in far fewer iterations being necessary to reach any given precision $\epsilon>0$. We leave a full theoretical investigation of HWF, with a proper discussion on step-size tuning, for future work.
\section*{Acknowledgments}
Fan Wu is supported by the EPSRC and MRC through the OxWaSP CDT programme (EP/L016710/1).

\bibliography{references}
\bibliographystyle{plainnat}

\clearpage
\appendix

{\LARGE\textbf{Appendix}}
\vspace{3mm}

The appendix is organized as follows.

In Appendix \ref{appendix:a}, we show the equivalence between continuous-time mirror descent equipped with the hypentropy mirror map and the exponentiated gradient algorithm described in Section 5.

In Appendix \ref{appendix:b}, we prove Lemma \ref{lemma1}.

In Appendix \ref{appendix:c}, we state supporting lemmas used in the proof of Theorem \ref{theorem}.

In Appendix \ref{appendix:d}, we prove Theorem \ref{theorem}.

In Appendix \ref{appendix:e}, we prove the supporting lemmas stated in Appendix \ref{appendix:c}.

In Appendix \ref{appendix:f}, we state and proof technical lemmas used in the proofs of Theorem \ref{theorem} and the supporting lemmas of Appendix \ref{appendix:c}.

Throughout the appendix, we will write $[n] = \{1,...,n\}$ for any natural number $n\in \mathbb{N}$. For a vector $\mathbf{u}\in \R^n$ and an index $i\in [n]$, we write $\mathbf{u}_{-i} = (u_l)_{l\neq i}\in \R^{n-1}$ for the vector with the $i$-th coordinate removed. Recall that, for any set $\S\subseteq [n]$, we write $\mathbf{u}_\S = (u_i)_{i\in\S}\in\R^{|\S|}$, and $x^\star_{min} = \min_{i:x^\star_i\neq 0} |x^\star_i|/\|\x^\star\|_2$ for the magnitude of the minimum non-zero component of the signal vector $\x^\star$.

Recall that the empirical risk is given by
\begin{equation*}
F(\x) = \frac{1}{4m}\sum_{j=1}^m\left((\A_j^T\x)^2 - (\A_j^T\x^\star)^2\right)^2,
\end{equation*}
and its gradient by
\begin{equation*}
\nabla F(\x) = \frac{1}{m}\sum_{j=1}^m\left((\A_j^T\x)^2 - (\A_j^T\x^\star)^2\right)(\A_j^T\x)\A_j.
\end{equation*}
The population gradient can be computed, since by definition $\A_j\sim\gauss (\mathbf{0},\mathbf{I}_n)$ i.i.d., as
\begin{align*}
\nabla f(\x)_i = \E[\nabla F(\x)_i] &= \E\left[(\A_j^T\x)^3A_{ji} - (\A_j^T\x)(\A_j^T\x^\star)^2A_{ji}\right] \\
&= \E\left[A_{ji}^4 \left(x_i^3 - x_i(x^\star_i)^2\right) + A_{ji}^2\bigg(\sum_{l\neq i}x_i(3A_{jl}^2x_l^2 - A_{jl}^2(x^\star_l)^2 - x^\star_i A_{jl}^2x_lx^\star_l\bigg)\right] \\
&= \left(3\|\x\|_2^2 - \|\x^\star\|_2^2\right)x_i - 2\left(\x^T\x^\star\right)x^\star_i,
\end{align*}
so we have
\begin{equation*}
\nabla f(\x) = \left(3\|\x\|_2^2 - \|\x^\star\|_2^2\right)\x - 2\left(\x^T\x^\star\right)\x^\star.
\end{equation*}

\section{Equivalence of mirror descent and EG$\pm$ in continuous time}
\label{appendix:a}
The following calculations closely follow the proof of Theorem 24 of \cite{GHS20}, which shows the equivalence of discrete-time versions of mirror descent and EG$\pm$. 

Recall that mirror descent is defined by the equation
\begin{equation}
\label{eq:mirror_descent_continuous}
\frac{d}{dt}\X(t) = -\left(\nabla^2\Phi(\X(t))\right)^{-1}\nabla F(\X(t)).
\end{equation}
with initialization
\begin{align}
\label{eq:initialization}
X_i(0) = \begin{cases}
\frac{1}{\sqrt{3}}\cdot \sqrt{\frac{1}{m}\sum_{j=1}^mY_j} \qquad & i=i_0\\
0 & i\neq i_0
\end{cases},
\end{align}
and EG$\pm$ in continuous time is defined by 
\begin{equation}
\label{eq:exponentiated_gradient_continuous}
\begin{gathered}
\X(t) = \mathbf{U}(t) - \mathbf{V}(t), \\
\frac{d}{dt}\mathbf{U}(t) = -\mathbf{U}(t)\odot \nabla F(\X(t)), \qquad \frac{d}{dt}\mathbf{V}(t) = \mathbf{V}(t)\odot \nabla F(\X(t)), 
\end{gathered}
\end{equation}
with initialization 
\begin{align}\label{eq:initialization2}
U_i(0) = \begin{cases} \frac{\hat{\theta}}{2\sqrt{3}} + \sqrt{\frac{\hat{\theta}^2}{12}+ \frac{\beta^2}{4}} \quad &i=i_0 \\ \frac{\beta}{2} & i\neq i_0 \end{cases}, \qquad V_i(0) = \begin{cases} -\frac{\hat{\theta}}{2\sqrt{3}} + \sqrt{\frac{\hat{\theta}^2}{12}+ \frac{\beta^2}{4}} \quad &i=i_0 \\ \frac{\beta}{2} & i\neq i_0 \end{cases},
\end{align}
where we write $\hat{\theta} = \sqrt{\frac{1}{m}\sum_{j=1}^mY_j}$ for the estimate of the size $\|\x^\star\|_2$ and $i_0$ is the same as in (\ref{eq:initialization}). A quick calculation shows that the initializations (\ref{eq:initialization}) and (\ref{eq:initialization2}) are equivalent, as $\X(0) =\mathbf{U}(0) - \mathbf{V}(0)$.

The equivalence of (\ref{eq:mirror_descent_continuous}) and (\ref{eq:exponentiated_gradient_continuous}) is a result of the fact that the product $\mathbf{U}(t)\odot \mathbf{V}(t)$ is a constant independent of $t$:
\begin{equation*}
\frac{d}{dt}\mathbf{U}(t)\odot \mathbf{V}(t) = - \mathbf{U}(t)\odot \nabla F(\X(t)) \odot \mathbf{V}(t) + \mathbf{U}(t)\odot \nabla F(\X(t)) \odot \mathbf{V}(t) = \mathbf{0},
\end{equation*}
hence, for EG$\pm$ (\ref{eq:exponentiated_gradient_continuous}) with initialization (\ref{eq:initialization2}), we have $U_i(t)V_i(t) = U_i(0)V_i(0) = \frac{\beta^2}{4}$ for all $i\in [n]$, and recalling that $U_i(t)>0$, solving for the positive root of the quadratic equation gives
\begin{equation*}
X_i(t) = U_i(t) - V_i(t) = U_i(t) - \frac{\beta^2}{4U_i(t)} \hspace{1mm}\Rightarrow\hspace{1mm} U_i(t) = \frac{X_i(t) + \sqrt{X_i(t)^2 + \beta^2}}{2}.
\end{equation*}
The same computation yields $V_i(t) = (-X_i(t) + \sqrt{X_i(t)^2+\beta^2})/2$, which leads to
\begin{equation*}
U_i(t) + V_i(t) = \sqrt{X_i(t)^2 + \beta^2}.
\end{equation*}
Thus, the evolution of $\X(t)$ in the exponentiated gradient algorithm (\ref{eq:exponentiated_gradient_continuous}) is described by
\begin{align*}
\frac{d}{dt}X_i(t) &= \frac{d}{dt}U_i(t) - \frac{d}{dt}V_i(t) \nonumber\\
&= - (U_i(t) + V_i(t)) \nabla F(\X(t))_i \\
&= - \sqrt{X_i^2 + \beta^2} \cdot \nabla F(\X(t))_i \\
&= - [(\nabla^2\Phi(\X(t)))^{-1}\nabla F(\X(t))]_i,
\end{align*}
for all $i\in [n]$, which is exactly the definition of mirror descent (\ref{eq:mirror_descent_continuous}) with parameter $\beta$.

\section{Proof of Lemma \ref{lemma1}}
\label{appendix:b}
\begin{proof}[Proof of Lemma \ref{lemma1}]
Using the identiy $\operatorname{arcsinh}(x) = \log (x + \sqrt{1+x^2})$, we can compute
\begin{align*}
D_{\Phi}(\x^\star,\x) &= \sum_{i=1}^n x^\star_i \operatorname{arcsinh}\left(\frac{x^\star_i}{\beta}\right) - \sqrt{(x^\star_i)^2 + \beta^2} -  x_i \operatorname{arcsinh}\left(\frac{x_i}{\beta}\right) + \sqrt{x_i^2 + \beta^2} \nonumber\\
&\quad - (x^\star_i - x_i) \operatorname{arcsinh}\left(\frac{x_i}{\beta}\right)\nonumber\\
&= \sum_{i=1}^n g_i(x_i),
\end{align*}
where 
\begin{equation*}
g_i(x_i) = \sqrt{x_i^2 + \beta^2} - \sqrt{(x_i^\star)^2 + \beta^2} - x^\star_i \log\frac{x_i + \sqrt{x_i^2 + \beta^2}}{x_i^\star + \sqrt{(x_i^\star)^2 + \beta^2}}. 
\end{equation*}
We begin by showing the bound (\ref{eq:lemma1ub}). For any $i\in [n]$, we have $g_i(x^\star_i) = 0$ and
\begin{align}\label{eq:gderiv}
g_i'(x_i) &=\frac{x_i}{\sqrt{x_i^2+\beta^2}} - x^\star_i\frac{1 + x_i/\sqrt{x_i^2+\beta^2}}{x_i+\sqrt{x_i^2 + \beta^2}} \nonumber \\
&= \frac{x_i-x^\star_i}{\sqrt{x_i^2+\beta^2}}.
\end{align}
Now, since we assume $x_ix^\star_i\ge 0$ and $|x_i|\ge \frac{1}{2}|x^\star_i|$, we have $\sqrt{z^2 + \beta^2} \ge \frac{1}{2}x^\star_{min} \ge \frac{c}{2\sqrt{k}}$ for any $z\in (x_i,x^\star_i)$ (or $(x^\star_i, x_i)$ if $x^\star_i < x_i$). Hence, we can bound  (using the convention $\int_a^bf(x)dx = -\int_b^af(x)dx$ when $a>b$)
\begin{equation*}
g_i(x_i) = g_i(x^\star_i) + \int_{x^\star_i}^{x_i}g'(z)dz < \int_{x^\star_i}^{x_i}\frac{2\sqrt{k}}{c}(z - x^\star_i)dz= \frac{\sqrt{k}}{c}(x_i-x^\star_i)^2.
\end{equation*}
Summing over $i\in \S$, this gives
\begin{equation*}
\sum_{i\in \S} g_i(x_i) \le \frac{\sqrt{k}}{c}\|\x_\S - \x^\star_\S\|_2^2
\end{equation*}
as claimed. On the other hand, if $x^\star_i=0$, we have 
\begin{equation*}
g_i(x_i) = \sqrt{x_i^2+\beta^2} - \beta \le |x_i|,
\end{equation*}
which gives
\begin{equation*}
\sum_{i\notin \S} g_i(x_i) \le \|x_{\S^c}\|_1,
\end{equation*}
and completes the proof of the bound (\ref{eq:lemma1ub}) in Lemma \ref{lemma1}. The other bound (\ref{eq:lemma1lb}) can be shown similarly using (\ref{eq:gderiv}). We have $\sqrt{z^2 + \beta^2} \le \sqrt{\max\{\|\x\|_\infty^2, \|\x^\star\|_\infty^2\}+\beta^2}$ for all $z\in(x_i,x_i^\star)$ (or $(x^\star_i,x_i)$ if $x^\star_i<x_i$), which gives
\begin{align*}
g(x_i) = g(x^\star_i) + \int_{x_i}^{x_i^\star}g'(z)dz &> \int_{x_i}^{x_i^\star}\left(\max\{\|\x\|_\infty^2, \|\x^\star\|_\infty^2\}+\beta^2\right)^{-\frac{1}{2}}(z - x^\star_i)dz \\
&= \frac{1}{2}\left(\max\{\|\x\|_\infty^2, \|\x^\star\|_\infty^2\}+\beta^2\right)^{-\frac{1}{2}}(x_i-x^\star_i)^2.
\end{align*}
Summing over all $i$ gives the bound (\ref{eq:lemma1lb}) of Lemma \ref{lemma1}.
\end{proof}

\section{Supporting lemmas for Theorem \ref{theorem}}
\label{appendix:c}
In this section, we provide three supporting lemmas characterizing the behavior of mirror descent, which will be useful to prove Theorem \ref{theorem}. For the sake of notational simplicity, we will assume $\|\x^\star\|_2 = 1$. The general case $\|\x^\star\|_2 \neq 1$ immediately follows by replacing $\x^\star$, $(\A_j^T\x^\star)^2$ and $\beta$ by $\x^\star/\|\x^\star\|_2$, $(\A_j^T\x^\star)^2/\|\x^\star\|_2^2$ and $\beta/\|\x^\star\|_2$, respectively, in what follows.

Lemmas \ref{lemma:support1} and \ref{lemma:support2} show that the empirical gradient $\nabla F(\x)$ is close to its expectation $\nabla f(\x)$ in a suitable sense.
\begin{lemma}\label{lemma:support1}
Let $\x^\star\in \R^n$ be a $k$-sparse vector with $\|\x^\star\|_2 = 1$, and let $\S = \{1\le i\le n: x^\star_i\neq 0\}$ be its support. Let $\{A_{ji}\}_{j\in [m], i\in [n]}$ be a collection of i.i.d.\ $\gauss(0,1)$ random variables. For any $0\le\delta\le 1$, define $\mathcal{X}= \{\x\in\R^n : \|\x\|_2\le 1, \|\x_{\S^c}\|_1\le \delta \}$.
Then, for any constant $\gamma>0$, there is a constant $c_1(\gamma)>0$, such that if $m\ge c_1(\gamma)\max\{k^2\log^2n,\; \log^5 n\}$, then
\begin{align*}
\| \nabla F(\x) - \nabla f(\x)\|_{\infty} \le \gamma\frac{\|\x_\S\|_1}{k} + c_3\delta \quad \text{for all } \x\in\mathcal{X} \label{eq:support1}
\end{align*}
holds with probability at least $1-c_2n^{-10}$. Here, $c_2$ and $c_3$ are universal constants (independent of $\gamma, c_1$).
\end{lemma}

\begin{lemma}\label{lemma:support2}
Let $\x^\star\in \R^n$ be a $k$-sparse vector with $\|\x^\star\|_2 = 1$, and let $\S = \{1\le i\le n: x^\star_i\neq 0\}$ be its support. Let $\{A_{ji}\}_{j\in [m], i\in [n]}$ be a collection of i.i.d.\ $\gauss(0,1)$ random variables. For any $0\le\delta\le 1$, define $\mathcal{X}= \{\x\in\R^n : \|\x\|_2\le 1, \|\x_{\S^c}\|_1\le \delta \}$. Then, for any constant $\gamma >0$, there is a constant $c_1(\gamma)>0$ such that if $m\ge c_1(\gamma)\max\{k^2\log^2 n, \;\log^5n\}$, then
\begin{align*}
&\big|\langle \nabla F(\x), \x-\x^\star\rangle - \langle \nabla f(\x), \x-\x^\star\rangle\big| \\
\le &\min \left\{\gamma\frac{\|\x_S\|_1}{\sqrt{k}} + c_3\delta,\; \gamma\left(\|\x_\S-\x^\star_\S\|_2^2 + \delta^2\right)\right\} \quad \text{ for all } \x\in\mathcal{X} 
\end{align*}
holds with probability at least $1-c_2n^{-10}$. Here, $c_2$ and $c_3$ are universal constants (independent of $\gamma, c_1$).
\end{lemma}

The next lemma characterizes the region to which the trajectory of mirror descent is confined.
\begin{lemma} \label{lemma:support3}
Let $\x^\star\in \R^n$ be any $k$-sparse vector with $\|\x^\star\|_2=1$ and $x^\star_{min}\ge c/\sqrt{k}$ for some constant $c>0$, and let $\S=\{1\le i\le n: x^\star_i\neq 0\}$ be its support. There exist universal constants $c_1,c_2,c_3>0$ such that the following holds. Let $m \ge c_1\max\{k^2\log^2n,\;\log^5n\}$, and let $\X(t)$ be given by the continuous-time mirror descent equation (1) with mirror map (4) and initialization (5). Assume that there is a $T>0$ such that
\begin{equation*}
\|\X_{\S^c}(t)\|_1\le \delta,
\end{equation*}
for all $t\le T$ and a constant $\delta\le c_3/n$.
Then, there is a $\xi \in \{-1,+1\}$, such that, for all $t>0$,
\begin{align}
\xi \cdot X_i(t)x^\star_i &\ge 0 \quad \text{for all } i\in [n], \label{eq:claim1}\\
\frac{1}{3} - 3\sqrt{\frac{\log n}{m}} \le \|\X(t)\|_2^2 &\le 2, \label{eq:claim2}\\
\sqrt{3}\cdot 2\left|\X(t)^T\x^\star\right| &\ge 3\|\X(t)\|_2^2 - 1, \label{eq:claim3}
\end{align}
holds with probability at least $1-c_2n^{-10}$.
\end{lemma}

\section{Proof of Theorem \ref{theorem}}
\label{appendix:d}
In this section, we make the same assumptions for notational simplicity as in Appendix \ref{appendix:c}. 
The following inequalities will be useful throughout the proof.
Assuming $\|\X_{\S^c}(t)\|\le \delta$, we have, by (\ref{eq:claim2}) of Lemma \ref{lemma:support3},
\begin{equation}\label{eq:bound_size}
\|\X_\S(t)\|_1 = \|\X(t)\|_1 - \|\X_{\S^c}(t)\|_1 \ge \left(\frac{1}{3} - 3\sqrt{\frac{\log n}{m}}\right)^{-\frac{1}{2}} - \delta \ge \frac{1}{2},
\end{equation}
and, using (\ref{eq:claim1}) of Lemma \ref{lemma:support3}, we can bound
\begin{equation}\label{eq:bound_innerproduct}
\X(t)^T\x^\star \ge \|\X_\S(t)\|_1x^\star_{min} \ge \frac{c}{2\sqrt{k}}.
\end{equation}
Before proving Theorem \ref{theorem}, it will be helpful to first consider the population dynamics, which highlights the main ideas of the proof of Theorem \ref{theorem}. That is, we first assume that we had access to the population gradient $\nabla f$, or in other words that $m=\infty$. 

The proof of Theorem \ref{theorem} relies on the identity
\begin{equation}\label{eq:bregman_derivative}
\frac{d}{dt}D_{\Phi}(\x^\star,\X(t)) = - \Big\langle\nabla^2\Phi(\X(t))\cdot\frac{d}{dt}\X(t), \x^\star - \X(t)\Big\rangle = -\Big\langle\nabla F(\X(t)), \X(t) - \x^\star \Big\rangle,
\end{equation}
to show that the Bregman divergence $D_{\Phi}(\x^\star, \X(t))$ decreases as claimed.

\textbf{Initial Bregman divergence}\\
Recalling the initialization (5), we can bound the initial Bregman divergence at $t=0$ by
\begin{align}
\label{eq:bregman_divergence_initial}
D_{\Phi}(\x^\star,\X(0)) &= \sum_{i\in \S\backslash \{i_0\}} \beta - \sqrt{(x_i^\star)^2 + \beta^2} - x^\star_i \log\frac{\beta}{x_i^\star + \sqrt{(x_i^\star)^2 + \beta^2}} \nonumber\\
&\quad + \sqrt{\hat{\theta}^2/3 + \beta^2} - \sqrt{(x_i^\star)^2 + \beta^2} - x^\star_i \log\frac{\hat{\theta}/\sqrt{3} + \sqrt{\hat{\theta}^2/3 + \beta^2}}{x_i^\star + \sqrt{(x_i^\star)^2 + \beta^2}} \nonumber\\
&\le \sum_{i\in\S\backslash{i_0}}|x^\star_i|\log \frac{1}{\beta} + \underbrace{\beta - \sqrt{(x^\star_i)^2+\beta^2} + |x^\star_i|\log \left(|x^\star_i| + \sqrt{(x^\star_i)^2 + \beta^2}\right)}_{\le 0} + 1 \nonumber\\
&\le \|\x^\star\|_1\log \frac{1}{\beta} + 1
\end{align}
where we write $\hat{\theta} = \sqrt{\sum_{j=1}^mY_j/m}$ for the estimate of the signal size $\|\x^\star\|_2$, and used the fact that the function $x\log\frac{\beta}{x+\sqrt{x^2+\beta^2}}$ is symmetric, that is $x\log\frac{\beta}{x+\sqrt{x^2+\beta^2}} = -x\log\frac{\beta}{-x+\sqrt{x^2+\beta^2}}$.

\textbf{Bounding $-\langle\nabla f(\X(t)), \X(t) -\x^\star \rangle$}\\
We can compute
\begin{equation*}
-\big\langle\nabla f(\X(t)), \X(t) - \x^\star \big\rangle = \left[3\left(\|\X(t)\|_2^2 - 1\right) + 2\left(\X(t)^T\x^\star\right)\right]\left(\X(t)^T\x^\star\right) - \left(3\|\X(t)\|_2^2 - 1\right)\|\X(t)\|_2^2.
\end{equation*}
To bound this quantity, we distinguish two cases.
\begin{itemize}
\item \textbf{Case 1: $\|\X(t)\|_2^2 \le \frac{2}{5}$}\\
In this case, we use the fact that $3\|\X(t)\|_2^2-1 \ge -9\sqrt{\frac{\log n}{m}}$ by Lemma \ref{lemma:support3} to bound
\begin{align}\label{eq:population_dynamics1}
-\big\langle \nabla f(\X(t)), \X(t) - \x^\star \big\rangle&\le \left[3\left(\|\X(t)\|_2^2-1\right) + 2\left(\X(t)^T\x^\star\right)\right]\left(\X(t)^T\x^\star\right) - \frac{18}{5}\sqrt{\frac{\log n}{m}} \nonumber \\
&\le \left(-\frac{9}{5} + 2\sqrt{\frac{2}{5}}\right)\left(\X(t)^T\x^\star\right) - \frac{18}{5}\sqrt{\frac{\log n}{m}} \nonumber \\
& \le -\frac{c}{2\sqrt{k}}\|\X_\S(t)\|_1, 
\end{align}
for $m\ge c'k\log n$ with $c'$ large enough, where for the last line we used (\ref{eq:bound_innerproduct}).

\item \textbf{Case 2: $\|\X(t)\|_2^2 > \frac{2}{5}$}\\
In this case, we use $2(\X(t)^T\x^\star) = \|\X(t)\|_2^2 + \|\x^\star\|_2^2 - \|\X(t)-\x\|_2^2$ to bound
\begin{align*}
-\big\langle \nabla f(\X(t)), \X(t) - \x^\star \big\rangle \le \left[4\|\X(t)\|_2^2 -2 - \Delta(t) \right]\left(\X(t)^T\x^\star\right) - \left(3\|\X(t)\|_2^2 - 1\right) \|\X(t)\|_2^2,
\end{align*}
where we write $\Delta(t) = \|\X(t)-\x^\star\|_2^2$.
If $4\|\X(t)\|_2^2 -2 - \Delta(t)\le 0$, this is upper bounded by
\begin{equation}
\label{eq:population_dynamics20}
-\big\langle \nabla f(\X(t)), \X(t) - \x^\star \big\rangle \le -\frac{2}{25}.
\end{equation}
Otherwise, we can bound, using $\X(t)^T\x^\star \le \|\X(t)\|_2\|\x^\star\|_2$,
\begin{align}\label{eq:population_dynamics2}
\big\langle \nabla f(\X(t)), \x^\star-\X(t) \big\rangle &\le \underbrace{-3\|\X(t)\|_2^4 + 4\|\X(t)\|_2^3 + \|\X(t)\|_2^2 - 2\|\X(t)\|_2}_{\le 0} - \Delta (t) \|\X(t)\|_2 \nonumber\\
&\le - \sqrt{\frac{2}{5}} \cdot \Delta (t).
\end{align}
\end{itemize}
The inequalities (\ref{eq:population_dynamics1})--(\ref{eq:population_dynamics2}) can be used to show that the Bregman divergence $D_{\Phi}(\x^\star,\X(t))$ decreases as claimed. In particular, we can use Lemma \ref{lemma1} to bound $\|\X(t)-\x^\star\|_2^2$ in terms of $D_{\Phi}(\x^\star, \X(t))$, and inequality (\ref{eq:population_dynamics2}) yields a bound leading to linear convergence if $\|\X_{\S^c}(t)\|_1$ is negligibly small compared to $\|\X_\S(t)-\x^\star_\S\|_2^2$. In order to make the proof of Theorem \ref{theorem} rigorous, we need to replace the population gradient $\nabla f$ by the empirical gradient $\nabla F$ in the outline we provided above.
\begin{proof}[Proof of Theorem \ref{theorem}]
Guided by the analysis of the population dynamics and the fact that Lemma \ref{lemma1} plays a central role in bounding (\ref{eq:population_dynamics2}) in terms of the Bregman divergence $D_{\Phi}(\x^\star,\X(t))$, we divide the analysis of the convergence of mirror descent into two stages, bounded by
\begin{align*}
T_1(\beta) &= \inf\left\{t>0 : \min_{i\in \S} \frac{|X_i(t)|}{|x^\star_i|} > \frac{1}{2} \right\}, \text{ and}\\
T_2(\delta) &= \inf\left\{t>0: D_{\Phi}(\x^\star,\X(t)) \le 2\delta \right\},
\end{align*}
respectively, where $\delta = c_3n\beta^{\alpha}$ for a constant $\alpha\in[\frac{1}{2},1)$. Note that we can choose any $\alpha\in[\frac{1}{2},1)$ provided $c_1$ is large enough. In particular, choosing $\alpha = \frac{1}{2}$ recovers the case stated in Theorem \ref{theorem}. Allowing different values for $\alpha$ shows that, as remarked in Section 4, the dependence $\delta= \Omega(n\sqrt{\beta})$ and hence the requirement $\beta\le \O(n^{-4})$ is not sharp. The important property is that $\delta$ depends polynomially on $\beta$. In the following, we will omit the dependency on $\beta$ and $\delta$ for notational simplicity.

We consider the stages (i) $t\le T_1$ and (ii) $T_1< t \le T_2$. Note that we have $T_2>T_1$, because if there is an index $i\in \S$ with $|X_i(t)|<\frac{1}{2}|x^\star_i|$, then we also have $\|\X(t)-\x^\star\|_2^2 > \frac{c^2}{4k}$. This implies that $D_{\Phi}(\x^\star,\X(t))>2\delta$, since in this case we have, by Lemma \ref{lemma1}, 
\begin{equation*}
D_{\Phi}(\x^\star,\X(t)) \ge \frac{1}{2\sqrt{\max\{\|\X(t)\|_\infty^2, \|\x^\star\|_\infty^2\} + \beta^2}} \|\X(t)-\x^\star\|_2^2 \ge \frac{1}{2\sqrt{2 + \beta^2}} \cdot \frac{c^2}{4k},
\end{equation*}
where we used that $\|\X(t)\|_\infty\le \|\X(t)\|_2 \le \sqrt{2}$ by Lemma \ref{lemma:support3}.

In both stages, we will (a) bound the length of the stage by using (\ref{eq:bregman_derivative}) to bound $T_i$, and (b) show that off-support coordinates $\|\X_{\S^c}(t)\|_1$ stay sufficiently small. Throughout the proof we will assume that the inequalities in Lemmas \ref{lemma:support1}--\ref{lemma:support3} are satisfied, which happens with probability at least $1-c_4n^{-10}$.

\textbf{Stage (i), part (a): $t\le T_1$, bound $T_1$}\\
Assume for now that we have already shown $\|\X_{\S^c}(t)\|_1< \delta_1 = n\beta^{\frac{1+\alpha}{2}}$ for all $t\le T_1$. We have already computed the rate (\ref{eq:bregman_derivative}) at which the Bregman divergence $D_{\Phi}(\x^\star,\X(t))$ decreases under the assumption that we have access to the population gradient $\nabla f$ in (\ref{eq:population_dynamics1}) and (\ref{eq:population_dynamics2}). Now, we need to bound (\ref{eq:bregman_derivative}) with the empirical gradient $\nabla F$.
\begin{itemize}
\item \textbf{Case 1: $\|\X(t)\|_2^2\le \frac{2}{5}$}\\
By Lemma \ref{lemma:support2}, we have
\begin{equation*}
\left|\big\langle \nabla F(\X(t)), \X(t) - \x^\star \big\rangle - \big\langle \nabla f(\X(t)), \X(t) - \x^\star \big\rangle\right| \le  \frac{c}{8\sqrt{k}}\|\X_\S(t)\|_1 + c_6\delta_1,
\end{equation*}
since $m\ge c_1\max\{k^2\log^2n,\;\log^5 n\}$, where $c_6$ is the universal constant of Lemma \ref{lemma:support2}. We also have
\begin{equation*}
c_6\delta_1 = c_6n\beta^{\frac{1+\alpha}{2}}\le \frac{c}{8\sqrt{k}}\|\X_\S(t)\|_1
\end{equation*}
for any $\alpha\ge 0$, where we used $\beta \le c_2/n^4$ and (\ref{eq:bound_size}). 
This shows that 
\begin{equation*}
\left|\big\langle \nabla F(\X(t)), \X(t) - \x^\star \big\rangle - \big\langle \nabla f(\X(t)), \X(t) - \x^\star \big\rangle\right| \le  \frac{c}{4\sqrt{k}}\|\X_\S(t)\|_1.
\end{equation*}
Together with (\ref{eq:population_dynamics1}), this bound leads to
\begin{align} \label{eq:bound1}
\frac{d}{dt}D_{\Phi}(\x^\star, \X(t)) &\le -\big\langle \nabla f(\X(t)), \X(t) - \x^\star \big\rangle  \nonumber\\
&\quad + \left|\big\langle \nabla F(\X(t)), \X(t) - \x^\star \big\rangle - \big\langle \nabla f(\X(t)), \X(t) - \x^\star \big\rangle\right| \nonumber\\
&\le - \frac{c}{4\sqrt{k}}\|\X_\S(t)\|_1 \nonumber\\
&\le - \frac{c}{8\sqrt{k}},
\end{align}
where for the last line we used (\ref{eq:bound_size}). 

\item \textbf{Case 2: $\|\X(t)\|_2^2\ge \frac{2}{5}$}\\
As in the previous case, we can bound the difference $\langle \nabla F(\X(t)) - \nabla f(\X(t)), \X(t) - \x^\star \rangle$ using Lemma \ref{lemma:support2}. Recalling (\ref{eq:population_dynamics20}) and (\ref{eq:population_dynamics2}), we obtain
\begin{equation}\label{eq:bound21}
\frac{d}{dt}D_{\Phi}(\x^\star, \X(t)) = -\big\langle \nabla F(\X(t)), \X(t) - \x^\star \big\rangle \le - \frac{1}{25}
\end{equation}
if $4\|\X(t)\|_2^2 -2 - \|\X(t)-\x^\star\|_2^2\le 0$, and
\begin{equation}\label{eq:bound2}
\frac{d}{dt}D_{\Phi}(\x^\star, \X(t))= -\big\langle \nabla F(\X(t)), \X(t) - \x^\star \big\rangle \le -\frac{1}{2}\|\X(t) - \x^\star\|_2^2
\end{equation}
if $4\|\X(t)\|_2^2 -2 - \|\X(t)-\x^\star\|_2^2> 0$, where we used the second bound in Lemma \ref{lemma:support2} together with the fact that $\|\X(t)-\x^\star\|_2^2 > \frac{c^2}{4k} > \delta_1^2$ to bound
\begin{align*}
\left|\big\langle \nabla F(\X(t)) - \nabla f(\X(t), \; \X(t) - \x^\star \big\rangle \right| &\le \frac{1}{16}\left(\|\X(t)-\x^\star\|_2^2 + \delta_1^2\right)\\
&\le \bigg(\sqrt{\frac{2}{5}}-\frac{1}{2}\bigg)\|\X(t) - \x^\star\|_2^2.
\end{align*}
\end{itemize}
We can now bound $T_1$.
Define
\begin{equation*}
T_0 = \inf\left\{t>0: \|\X(t)\|_2^2>\frac{2}{5}\right\},
\end{equation*}
as the time until which Case 1 holds.
Then, the bound (\ref{eq:bound1}) from Case 1 shows that
\begin{equation*}
\min\{T_1,T_0\} \le \frac{8\sqrt{k}D_{\Phi}(\x^\star,\X(0))}{c} \le \frac{8}{c} k\log \frac{1}{\beta} + \frac{8\sqrt{k}}{c} \le c_5k\log\frac{1}{\beta},
\end{equation*}
where $c_5\ge \frac{8}{c} + \frac{8}{c\sqrt{k}\log \frac{1}{\beta}}$ and we used (\ref{eq:bregman_divergence_initial}).
If $T_1<T_0$, then we have bounded $T_1$ as desired.

Otherwise, we can use the bounds (\ref{eq:bound21}) and (\ref{eq:bound2}) from Case 2 to control $T_1 - T_0$. The first bound (\ref{eq:bound21}) can apply at most for $t\le 25D_{\Phi}(\x^\star, \X(0))\le 25\sqrt{k}\log \frac{1}{\beta} + 25$, where we again used (\ref{eq:bregman_divergence_initial}).

As (\ref{eq:bound2}) depends on the quantity $\|\X(t)-\x^\star\|_2^2$, we need to show that this $\ell_2$-distance is sufficiently large if the Bregman divergence $D_{\Phi}(\x^\star,\X(t))$ is large.

To this end, define
\begin{equation*}
S(t) := \left\{i\in \S: \frac{|X_i(t)|}{|x^\star_i|} < \frac{1}{2}\right\}.
\end{equation*}
With this, we can bound
\begin{equation}\label{eq:boundbreg1}
\|\X(t)-\x^\star\|_2^2 \ge \sum_{i\in S(t)} \left(\frac{1}{2}x^\star_i\right)^2 + \sum_{i\notin S(t)}(X_i(t) - x^\star_i)^2 \ge \frac{c^2}{4k}|S(t)| + \sum_{i\notin S(t)}(X_i(t) - x^\star_i)^2.
\end{equation}
Following the same computation as in the proof of Lemma \ref{lemma1}, we get
\begin{align}\label{eq:boundbreg2}
D_{\Phi}(\x^\star,\X(t)) \le \sum_{i\in S(t)} |x^\star_i| \log \frac{1}{\beta} + \frac{\sqrt{k}}{c} \sum_{i\notin S(t)}(X_i(t) - x^\star_i)^2 + \|\X_{\S^c}(t)\|_1.
\end{align}
Because $\|\x^\star\|_\infty\le \|\x^\star\|_2=1$, we have $\sum_{i\in S(t)}|x^\star_i|\le |S(t)|$. Further, since $t< T_1$, we also have $\|\X_{\S^c}(t)\|_1\le \delta_1\le \frac{3}{4}(x^\star_{min})^2\le \sum_{i\notin S(t)}(X_i(t)-x^\star_i)^2$. With this, we can combine (\ref{eq:bound2}), (\ref{eq:boundbreg1}) and (\ref{eq:boundbreg2}) and bound
\begin{equation*}
\frac{d}{dt}D_{\Phi}(\x^\star,\X(t)) \le -\frac{1}{2} \|\X(t)-\x^\star\|_2^2 \le - \frac{c^2}{8k\log \frac{1}{\beta}} D_{\Phi}(\x^\star,\X(t)),
\end{equation*}
which shows that $D_{\Phi}(\x^\star,\X(t))$ decreases linearly at the rate $c^2/(8k\log \frac{1}{\beta})$, as long as $T_0 < t < T_1$. Now, we have
\begin{equation*}
D_{\Phi}(\x^\star,\X(T_0)) \le D_{\Phi}(\x^\star,\X(0)) \le \sqrt{k}\log \frac{1}{\beta}+1,
\end{equation*}
and, for $t<T_1$,
\begin{equation*}
\frac{c^2}{4k}\le \|\X(t)-\x^\star\|_2^2 \le 2\sqrt{2+\beta^2}\cdot D_{\Phi}(\x^\star,\X(t))\le 3D_{\Phi}(\x^\star,\X(t)),
\end{equation*}
where for the second inequality we used $\|\X(t)\|_\infty\le \sqrt{2}$ by Lemma \ref{lemma:support3} and the bound (\ref{eq:lemma1lb}) of Lemma \ref{lemma1}.
This implies
\begin{equation*}
(T_1-T_0) \le \frac{8k\log\frac{1}{\beta}}{c^2} \log \left(\frac{12}{c^2} k^{\frac{3}{2}}\log \frac{1}{\beta} + \frac{12}{c^2}k\right) \le \frac{c_5}{2}k\log\frac{1}{\beta} \cdot \log\left(k\log \frac{1}{\beta}\right)
\end{equation*}
for $c_5\ge \frac{24}{c^2} + \frac{16}{c^2}\log\frac{24}{c^2}$.

\textbf{Stage (i), part (b): $t\le T_1$, bound $\|\X_{\S^c}(t)\|_1$}\\
The main idea to controlling $\|\X_{\S^c}(t)\|_1$ is as follows: we will show that $X_j(t)$ can only grow at a comparatively slower rate than $X_i(t)$ for $j\notin \S$, $i\in \S$. We will show that both coordinates grow comparably to exponentials, and use the fact that, for any fixed $\epsilon>0$, the gap between $\beta(1+2\epsilon)^t$ and $\beta(1+\epsilon)^t$ can be made arbitrarily large by choosing $t$ large and $\beta$ small enough. Recall that
\begin{equation*}
\frac{d}{dt}X_i(t) = -\sqrt{X_i(t)^2 + \beta^2} \cdot \nabla F(\X(t))_i.
\end{equation*}
By Lemma \ref{lemma:support3}, we have $\sqrt{3}\cdot 2(\X(t)^T\x^\star) \ge 3\|\X(t)\|_2^2 - 1$. For any $i\in \S$ with $x^\star_i>0$ and $X_i(t)\le \frac{1}{2}x^\star_i$, we can compute
\begin{equation*}
\nabla f(\X(t))_i = 3(\|\X(t)\|_2^2 - 1)X_i(t) - 2(\X(t)^T\x^\star)x^\star_i\le \left(\frac{\sqrt{3}}{2}-1\right)\cdot 2(\X(t)^T\x^\star)x^\star_i \le - \frac{c\|\X_\S(t)\|_1}{4k}.
\end{equation*}
As before, Lemma \ref{lemma:support1} gives
\begin{equation*}
|\nabla F(\X(t))_i - \nabla f(\X(t))_i| \le \frac{c\|\X_\S(t)\|_1}{4k},
\end{equation*}
for $c_7\delta \le \frac{c}{8k}\|\X_\S(t)\|_1$. This gives
\begin{equation*}
\nabla F(\X(t))_{i} \le - \frac{c\|\X_\S(t)\|_1}{8k},
\end{equation*}
which means that $\frac{d}{dt}X_i(t)>0$. The analogous result holds for coordinates $i\in \S$ with $x^\star_i<0$, so in other words, once we have $|X_i(t_0)|\ge \frac{1}{2} |x^\star_i|$ for some $t_0>0$, then we must also have $|X_i(t)|\ge \frac{1}{2}|x^\star_i|$ for all $t\ge t_0$.

With this, we can define $i_1\in \S$ to be the last coordinate which crosses this threshold, that is for which $\frac{|X_i(t)|}{|x^\star_i|}\ge\frac{1}{2}$. By definition, we have $|X_{i_1}(t)| \le \frac{1}{2} |x^\star_{i_1}|$ for all $t\le T_1$. We can assume without loss of generality that $x^\star_{i_1}>0$.

For any $j\notin \S$, we have  
\begin{equation*}
\nabla f(\X(t))_j \ge 0
\end{equation*}
if $x_j(t)\ge 0$, and Lemma \ref{lemma:support1} gives
\begin{equation*}
\nabla F(\X(t))_j \ge -\frac{1-\alpha}{2\sqrt{2}}\cdot\frac{c\|\X_\S(t)\|_1}{8k}.
\end{equation*}
The analogous result holds for $X_j(t)<0$, which shows that, for any $j\notin \S$,
\begin{equation}
\label{eq:stage1b}
|\nabla F(\X(t))_j| \le \frac{1-\alpha}{2\sqrt{2}} |\nabla F(\X(t))_{i_1}|
\end{equation}
holds for all $t\le T_1$.

Loosely speaking, $X_{i_1}(t)$ and $X_j(t)$ both grow exponentially, but at different (time-varying) rates. By the definition of $i_1$, we have $X_{i_1}(T_1) = \frac{1}{2}x^\star_{i_1}$, and $X_j(t)$ can be made arbitrarily small by choosing a sufficiently small parameter $\beta$.

To make this rigorous, rescale time by a monotonically increasing function $\tilde{t}:\R_+\rightarrow \R_+$ such that $-\frac{d}{dt}\tilde{t}(t) \cdot \nabla F(\X(t))_{i_1} = c$ is constant, so that $\widetilde{\X}(t) = \X(\tilde{t}(t))$ satisfies, for all $t\ge 0$ with $\tilde{t}(t)\le T_1$,
\begin{align*}
\widetilde{X}_{i_1}(0) &= \beta  \\
\frac{d}{dt}\widetilde{X}_{i_1}(t) &= \sqrt{\widetilde{X}_{i_1}(t)^2 + \beta^2} \cdot c.
\end{align*}
Recalling the bound (\ref{eq:stage1b}), we have, for $j\notin \S$,
\begin{align*}
\left|\widetilde{X}_j(0)\right| &\le \beta  \\
\left|\frac{d}{dt}\widetilde{X}_j(t)\right| &\le \frac{1-\alpha}{2\sqrt{2}}\sqrt{\widetilde{X}_j(t)^2 + \beta^2} \cdot c.
\end{align*}
Since $\sqrt{\widetilde{X}_{i_1}(t)^2 + \beta^2} \ge \widetilde{X}_{i_1}(t)$, we have
\begin{equation*}
\widetilde{X}_{i_1}(t) \ge \beta \exp(ct) \quad \Rightarrow \quad \exp(ct) \le \frac{\widetilde{X}_{i_1}(t)}{\beta}.
\end{equation*}
Similarly, since $\sqrt{\widetilde{X}_j(t)^2 + \beta^2} \le \sqrt{2} \widetilde{X}_j(t)$ for $\widetilde{X}_j(t)\ge \beta$, we can bound
\begin{equation*}
\widetilde{X}_j(t) \le \beta \exp\left(\frac{1-\alpha}{2}ct\right) \le \beta \left(\frac{\widetilde{X}_{i_1}(t)}{\beta}\right)^{\frac{1-\alpha}{2}} \le \beta^{\frac{1+\alpha}{2}},
\end{equation*}
where we used the fact that, since $\tilde{t}(t)\le T_1$, we have $\widetilde{X}_{i_1}(t)\le \frac{1}{2}x^\star_{i_1}\le 1$.
As this holds for every $j\notin \S$, we have, for all $t\le T_1$,
\begin{equation*}
\|\X_{\S^c}(t)\|_1 \le n\beta^{\frac{1+\alpha}{2}} = \delta_1
\end{equation*}

\textbf{Stage (ii), part (a): $T_1<t\le T_2$, bound $D_{\Phi}(\x^\star,\X(t))$}\\
As before, assume for now that we have already shown $\|\X_{\S^c}(t)\|_1\le \delta$ for all $t\le T_2$. In this stage we have $t>T_1$, so $|X_i(t)|\ge \frac{1}{2}|x^\star_i|$ for all $i=1,...,n$. By Lemma \ref{lemma:support3} we also have $X_i(t)x^\star_i\ge 0$ (since we assumed $x^\star_{i_0}>0$), so the assumptions for inequality (\ref{eq:lemma1ub}) of Lemma \ref{lemma1} are satisfied.
Then, we have for all $t\le T_2$, by the definition of $T_2$,
\begin{equation*}
\|\X(t)-\x^\star\|_2^2 \ge \frac{c}{\sqrt{k}}\left(D_{\Phi}(\x^\star, \X(t))- \delta \right) \ge \frac{c}{\sqrt{k}}\delta. 
\end{equation*}
As in the previous stage, we can use this bound together with the second inequality of Lemma \ref{lemma:support2} to show the bound (\ref{eq:bound2}). We can also apply Lemma \ref{lemma1} to obtain
\begin{equation*}
D_{\Phi}(\x^\star,\X(t)) \le \frac{\sqrt{k}}{c}\|\X_\S(t)-\x^\star_\S\|_2^2 + \|\X_{\S^c}(t)\|_1 \le \frac{2\sqrt{k}}{c}\|\X(t)-\x^\star\|_2^2,
\end{equation*}
where we used $\|\X_{\S^c}(t)\|_1 \le \delta \le \frac{\sqrt{k}}{c}\|\X(t)-\x^\star\|_2^2$. In particular, we have $D_{\Phi}(\x^\star,\X(T_1)) \le \frac{6\sqrt{k}}{c}$, where we used $\|\X(t)\|_2^2\le 2$ by Lemma \ref{lemma:support3}. With this, inequality (\ref{eq:bound2}) reads
\begin{equation*}
\frac{d}{dt}D_{\Phi}(\x^\star,\X(t)) \le -\frac{c}{4\sqrt{k}}D_{\Phi}(\x^\star,\X(t)).
\end{equation*}
Hence, we have for $T_1\le t\le T_2$,
\begin{align*}
D_{\Phi}(\x^\star,\X(t)) &\le D_{\Phi}(\x^\star,\X(T_1)) \cdot \exp\left(-\frac{c}{4\sqrt{k}}(t-T_1)\right) \\
&\le \frac{6\sqrt{k}}{c}\cdot \exp\left(-\frac{c}{4\sqrt{k}}(t-T_1)\right).
\end{align*}
Further, recalling that $D_{\Phi}(\x^\star, \X(t)) > 2\delta$ for $t\le T_2$, we can bound
\begin{equation*}
(T_2-T_1) \le \frac{4\sqrt{k}}{c}\log \frac{6\sqrt{k}}{2c\delta}.
\end{equation*}

\textbf{Stage (ii), part (b): $T_1<t\le T_2$, bound $\|\X_{\S^c}(t)\|_1$}\\
Recall that $|X_j(t)|\le \beta^{\frac{1+\alpha}{2}}$ for all $j\notin \S$ and $t\le T_1$. Further, as before we can use Lemma \ref{lemma:support1} to show
\begin{equation*}
|\nabla F(\X(t))_j| \le \frac{1-\alpha}{2\alpha} \cdot \frac{c\|\X_\S(t)\|_1}{8k} \le \frac{c(1-\alpha)}{8\alpha\sqrt{2k}},
\end{equation*}
where we used $\|\X_\S(t)\|_1\le \sqrt{k}\|\X_\S(t)\|_2\le \sqrt{2k}$. Noting $\sqrt{x^2+\beta^2}\le \sqrt{2}x$ for $x\ge \beta$, we can bound, for $t\le T_2$,
\begin{align*}
|X_j(t)| &\le \beta^{\frac{1+\alpha}{2}}\exp\left(\frac{c(1-\alpha)}{8\alpha\sqrt{k}}(T_2-T_1)\right) \\
&\le \beta^{\frac{1+\alpha}{2}}\cdot\exp\left(\frac{c(1-\alpha)}{8\alpha\sqrt{k}} \cdot \frac{4\sqrt{k}}{c}\log \frac{6\sqrt{k}}{2c\delta}\right) \\
&\le \beta^{\frac{1+\alpha}{2}}\cdot\left(\frac{6\sqrt{k}}{2c\delta}\right)^{\frac{1-\alpha}{2\alpha}} \\
 &\le \frac{\delta}{n},
\end{align*}
where for the second inequality we used $\log (1+x)\le x$ for $x>0$, and for the last inequality we used the definition $\delta = (\frac{6}{2c})^{\frac{1-\alpha}{1+\alpha}}n\beta^\alpha$. This completes the proof that $\|\X_{\S^c}(t)\|_1\le \delta$ for all $t\le T_2$.
\end{proof}

\section{Proof of supporting lemmas}
\label{appendix:e}
In this section, we prove the supporting lemmas stated in Appendix \ref{appendix:c}.
\begin{proof}[Proof of Lemma \ref{lemma:support1}]
For any $i\in [n]$, we will bound the difference $|\nabla F(\x)_i - \nabla f(\x)_i|$; the result then follows by taking a union bound. First, we write $\mathbf{w}\in \R^n$ for the vector $\x_\S$ padded with zeroes, that is $w_i = x_i$ for $i\in \S$ and $w_i = 0$ otherwise. Recall that $\nabla f(\x) = \E[\nabla F(\x)]$. Then, we can decompose
\begin{align}\label{eq:split}
|\nabla F(\x)_i - \nabla f(\x)_i| &\le |\nabla F(\x)_i - \nabla F(\mathbf{w})_i| + |\nabla F(\mathbf{w})_i - \E[\nabla F(\mathbf{w})_i]| \nonumber\\
&\quad + |\E[\nabla F(\mathbf{w})_i] - \E[\nabla F(\x)_i]|,
\end{align}
and we will bound the three terms separately.

\textbf{Step 1: Bound the term $|\nabla F(\x)_i - \nabla F(\mathbf{w})_i|$}\\
We begin by bounding the first term of (\ref{eq:split}) by $\frac{c_3}{2}\delta$. Recall that
\begin{equation*}
\nabla F(\x)_i = \frac{1}{m}\sum_{j=1}^m\left((\A_{j,\S}^T\x_\S + \A_{j,\S^c}\x_{\S^c})^3 - (\A_{j,\S}^T\x_\S + \A_{j,\S^c}^T\x_{\S^c})(\A_j^T\x^\star)^2\right)A_{ji}.
\end{equation*}
Then, we have
\begin{align}\label{eq:split2}
\nabla F(\x)_i - \nabla F(\mathbf{w})_i &= \frac{3}{m}\sum_{j=1}^m A_{ji}(\A_{j,\S}^T\x_\S)^2(\A_{j,\S^c}^T\x_{\S^c}) + \frac{3}{m}\sum_{j=1}^m A_{ji}(\A_{j,\S}^T\x_\S)(\A_{j,\S^c}^T\x_{\S^c})^2 \nonumber\\
&\quad + \frac{1}{m}\sum_{j=1}^m A_{ji}(\A_{j,\S^c}^T\x_{\S^c})^3 - \frac{1}{m}\sum_{j=1}^m A_{ji}(\A_{j,\S^c}^T\x_{\S^c})(\A_j^T\x^\star)^2.
\end{align}
These four terms can be bounded as follows: for the first term, we have
\begin{align*}
\frac{1}{m}\sum_{j=1}^m A_{ji}(\A_{j,\S}^T\x_\S)^2(\A_{j,\S^c}^T\x_{\S^c}) &= \sum_{l\notin \S}x_l \cdot \frac{1}{m}\sum_{j=1}^mA_{ji}A_{jl}(\A_{j,\S}^T\x_{\S})^2\\
&\le \|\x_{\S^c}\|_1 \cdot \max_{l\notin \S}\left|\frac{1}{m}\sum_{j=1}^mA_{ji}A_{jl}(\A_{j,\S}^T\x_{\S})^2\right| \\
&\le \|\x_{\S^c}\|_1\cdot \max_{l\notin \S}\sqrt{\frac{1}{m}\sum_{j=1}^mA_{ji}^2A_{jl}^2}\cdot\sqrt{\frac{1}{m}\sum_{j=1}^m(\A_{j,\S}^T\x_\S)^4}, 
\end{align*}
where we used H\"{o}lder's inequality in the second and the Cauchy-Schwarz inequality in the last line. The first sum is bounded by Lemma \ref{lemma:tech3}: recalling $m\ge c_1(\gamma)k^2\log^2 n$, we have with probability at least $1-c_4n^{-13}$,
\begin{equation*}
\max_{l\notin \S}\left|\frac{1}{m}\sum_{j=1}^mA_{ji}^2A_{jl}^2\right| \le 1 + \frac{1}{k} \le 2.
\end{equation*}
By Lemma \ref{lemma:tech1} (with $t=5\sqrt{\log n}$), we have with probability $1-4n^{-12.5}$,
\begin{equation*}
\sqrt{\frac{1}{m}\sum_{j=1}^m (\A_{j,\S}^T\x_\S)^4} \le \frac{1}{\sqrt{m}}\left((3m)^{\frac{1}{4}} + \sqrt{k} + 5\sqrt{\log n}\right)^2 \le 11
\end{equation*}
for all $\x\in\mathcal{X}$, where we used $5\sqrt{\log n} \le m^{\frac{1}{4}}$, which holds if $c_1\log^5n \ge 5^4\log^2n$. 

Put together, this gives 
\begin{equation*}
\frac{1}{m}\sum_{j=1}^m A_{ji}(\A_{j,\S}^T\x_\S)^2(\A_{j,\S^c}^T\x_{\S^c}) \le c'\delta,
\end{equation*}
where $c' = 11\sqrt{2}$.
The other terms in (\ref{eq:split2}) can be bounded the same way. For instance, we can write
\begin{equation*}
\frac{1}{m}\sum_{j=1}^m A_{ji}(\A_{j,\S}^T\x_S)(\A_{j,\S^c}^T\x_{\S^c})^2 = \sum_{l\notin \S}x_l\sum_{s\notin \S}x_s\cdot \frac{1}{m}\sum_{j=1}^mA_{ji}A_{jl}A_{js}(\A_{j,\S}^T\x_\S)
\end{equation*}
and, following the same steps as before, we obtain the bounds
\begin{align*}
\frac{1}{m}\sum_{j=1}^m A_{ji}(\A_{j,\S^c}^T\x_{\S^c})(\A_j^T\x^\star)^2 &\le c' \delta,\\
\frac{1}{m}\sum_{j=1}^m A_{ji}(\A_{j,\S}^T\x_S)(\A_{j,\S^c}^T\x_{\S^c})^2 &\le  c''\delta^2,\\
\frac{1}{m}\sum_{j=1}^m A_{ji}(\A_{j,\S^c}^T\x_{\S^c})^3 &\le c'''\delta^3,
\end{align*}
for all $\x\in\mathcal{X}$ with probability $1-3(c_4+4)n^{-12.5}$.
Recall that $\delta\le 1$. For any constant $c_3$ satisfying $4c'\delta + 3c''\delta^2 + c'''\delta^3 \le \frac{c_3}{2}\delta$, we have
\begin{equation*}
|\nabla F(\x)_i - \nabla F(\mathbf{w})_i| \le \frac{c_3}{2}\delta \quad \text{for all } \x\in\mathcal{X}
\end{equation*}
with probability at least $1-\frac{c_2}{2}n^{-11}$.

\textbf{Step 2: Bound the term $|\E[\nabla F(\x)_i] - \E[\nabla F(\mathbf{w})_i]|$}\\
For the expectation, we can compute, using the Cauchy-Schwarz inequality,
\begin{align*}
\E\left[\frac{1}{m}\sum_{j=1}^m A_{ji}(\A_{j,\S}^T\x_\S)^2(\A_{j,\S^c}^T\x_{\S^c})\right] &\le \E[A_{1i}^2]^{\frac{1}{2}} \E\left[(\A_{1,\S}^T\x_\S)^4(\A_{1,\S^c}^T\x_{\S^c})^2\right]^{\frac{1}{2}} \\
&= \left(\E\left[(\A_{1,\S}^T\x_\S)^4\right]\E\left[(\A_{1,\S^c}^T\x_{\S^c})^2\right]\right)^{\frac{1}{2}} \\
&\le \sqrt{3\|\x_{\S^c}\|_2^2} \\
&\le \sqrt{3}\delta
\end{align*}
for all $\x\in\mathcal{X}$, where we used that $\A_{1,\S}^T\x_\S \sim \gauss(0,\|\x_\S\|_2^2)$ and $\A_{1,\S^c}^T\x_{\S^c} \sim \gauss(0,\|\x_{\S^c}\|_2^2)$ are independent, and $\|\x_\S\|_2\le 1$, $\|\x_{\S^c}\|_2\le \delta$. Similarly, we can bound the other terms:
\begin{align*}
\E\left[\frac{1}{m}\sum_{j=1}^m A_{ji}(\A_{j,\S^c}^T\x_{\S^c})(\A_j^T\x^\star)^2\right] &\le \sqrt{3}\delta,\\
\E\left[\frac{1}{m}\sum_{j=1}^m A_{ji}(\A_{j,\S}^T\x_\S)(\A_{j,\S^c}^T\x_{\S^c})^2\right] &\le  \sqrt{3}\delta^2,\\
\E\left[\frac{1}{m}\sum_{j=1}^m A_{ji}(\A_{j,\S^c}^T\x_{\S^c})^3\right] &\le \sqrt{15}\delta^3.
\end{align*}
This completes the proof that 
\begin{equation*}
|\E[\nabla F(\x)_i] - \E[\nabla F(\mathbf{w})_i]| \le \frac{c_3}{2}\delta.
\end{equation*}

\textbf{Step 3: Bound the term $|\nabla F(\mathbf{w})_i - \E[\nabla F(\mathbf{w})_i]|$}\\
Finally, we need to show that the term $|\nabla F(\mathbf{w})_i - \E[\nabla F(\mathbf{w})_i]|$ in (\ref{eq:split}) can be bounded by $\gamma\frac{\|\x_\S\|_1}{k}$ for all $\x\in\mathcal{X}$ and $i\in [n]$ with probability $1-\frac{c_2}{2}n^{-10}$, which then completes the proof of Lemma \ref{lemma:support1}.

We decompose $\nabla F(\mathbf{w})_i$ in a straightforward, albeit somewhat lengthy manner. We have
\begin{align*}
\nabla F(\mathbf{w})_i &= \frac{1}{m}\sum_{j=1}^m\left((\A_j^T\mathbf{w})^3 - (\A_j^T\mathbf{w})(\A_j^T\x^\star)^2\right)A_{ji} \\
&= \left(w_i^3 - w_i(x^\star_i)^2\right)\frac{1}{m}\sum_{j=1}^mA_{ji}^4 + \left(3w_i^2 - (x^\star_i)^2\right)\frac{1}{m}\sum_{j=1}^mA_{ji}^3 (\A_{j,-i}^T\mathbf{w}_{-i}) \\
&\quad - 2w_ix^\star_i\frac{1}{m}\sum_{j=1}^mA_{ji}^3(\A_{j,-i}^T\x^\star_{-i}) + 3w_i\frac{1}{m}\sum_{j=1}^mA_{ji}^2(\A_{j,-i}^T\mathbf{w}_{-i})^2 \\
&\quad - 2x^\star_i\frac{1}{m}\sum_{j=1}^mA_{ji}^2(\A_{j,-i}^T\mathbf{w}_{-i})(\A_{j,-i}^T\x^\star_{-i}) - w_i\frac{1}{m}\sum_{j=1}^mA_{ji}^2(\A_{j,-i}^T\x^\star_{-i})^2 \\
&\quad + \frac{1}{m}\sum_{j=1}^mA_{ji} (\A_{j,-i}^T\mathbf{w}_{-i})^3 - \frac{1}{m}\sum_{j=1}^mA_{ji}(\A_{j,-i}^T\mathbf{w}_{-i})(\A_{j,-i}^T\x^\star_{-i})^2 \\
&=: B_1 + B_2 + B_3 + B_4 + B_5 + B_6 + B_7 + B_8,
\end{align*}
and we will show that $|B_l-\E[B_l]|$ is small for all $l=1,...,8$. All the following statements hold with probability $1-\frac{c_2}{2}n^{-10}$ for all $\mathbf{w}\in \R^n$ with $\|\mathbf{w}\|_2\le 1$ and $\mathbf{w}_{\S^c}=\mathbf{0}$, and all $i=1,...,n$.  We write the bounds in terms of $k$ instead of $m$ using the assumption $m\ge c_1(\gamma)\max\{k^2\log^2 n, \; \log^5 n\}$.
\begin{itemize}
\item For the first term, we have $\E[B_1] = 3(w_i^3 - w_i(x^\star_i)^2)$, and 
\begin{equation*}
|B_1 - \E[B_1]| = \left|\left(w_i^3 - w_i(x^\star_i)^2\right)\left(\frac{1}{m}\sum_{j=1}^mA_{ji}^4 - 3\right)\right| \le \frac{\gamma\|\mathbf{w}\|_1}{8k}
\end{equation*}
by Lemma \ref{lemma:tech3}, where we used $|w_i^3 - w_i(x^\star_i)| \le \|\mathbf{w}\|_1$.

\item For the second term, we have $\E[B_2] = 0$, and
\begin{equation*}
|B_2| \le \left|3w_i^2 - (x^\star_i)^2\right| \|\mathbf{w}\|_1 \max_{l\neq i} \left|\frac{1}{m}\sum_{j=1}^mA_{ji}^3A_{jl}\right| \le \frac{\gamma\|\mathbf{w}\|_1}{8k}
\end{equation*}
by Lemma \ref{lemma:tech3}.

\item For the third term, we have $\E[B_3] = 0$, and
\begin{equation*}
|B_3| \le \left|2w_ix^\star_i\right|\frac{\gamma}{16k} \le \frac{\gamma\|\mathbf{w}\|_1}{8k}
\end{equation*}
by Lemma \ref{lemma:tech3}.

\item For the fourth term, we have $\E[B_4] = 3w_i\|\mathbf{w}_{-i}\|_2^2$, and
\begin{align*}
|B_4 - \E[B_4]| &\le 3|w_i|\left|\sum_{l\neq i}w_l^2\left(\frac{1}{m}\sum_{j=1}^mA_{ji}^2A_{jl}^2 - 1\right)\right| \\
&\quad+ 3|w_i| \left|\sum_{l\neq i}w_l\sum_{s\neq l,i} w_s \frac{1}{m}\sum_{j=1}^mA_{ji}^2A_{jl}A_{js}\right| \nonumber\\
&\le \frac{\gamma\|\mathbf{w}\|_1}{16k} + \frac{\gamma\|\mathbf{w}\|_1}{16k} = \frac{\gamma\|\mathbf{w}\|_1}{8k},
\end{align*}
where we used Lemma \ref{lemma:tech3} to bound the first and Lemma \ref{lemma:tech7} to bound the second term.

\item For the fifth term, we have $\E[B_5] = -2x^\star_i\sum_{l\neq i}w_lx^\star_l $, and
\begin{align*}
|B_5 - \E[B_5]| &\le \left|2x^\star_i\right|\left|\sum_{l\neq i} w_l \left(\frac{1}{m}\sum_{j=1}^mA_{ji}^2A_{jl}(\A_{j,-i}^T\x^\star_{-i}) - x^\star_l\right)\right| \nonumber\\
&\le \frac{\gamma\|\mathbf{w}\|_1}{8k}
\end{align*}
where we used H\"{o}lder's inequality and Lemma \ref{lemma:tech3}.

\item For the sixth term, we have $\E[B_6] = -w_i\|\x^\star_{-i}\|_2^2$, and
\begin{align*}
|B_6 - \E[B_6]| &\le \left|w_i\right|\frac{\gamma}{8k} \le \frac{\gamma\|\mathbf{w}\|_1}{8k}
\end{align*} 
by Lemma \ref{lemma:tech3}.

\item For the seventh term, we have $\E[B_7] = 0$, and
\begin{align*}
|B_7| &\le \frac{\gamma\|\mathbf{w}\|_1}{8k}
\end{align*} 
by Lemma \ref{lemma:tech7}.

\item Finally, for the eighth term, we have $\E[B_8] = 0$, and
\begin{align*}
|B_8|&\le \left|\sum_{l\neq i} w_l \frac{1}{m}\sum_{j=1}^mA_{ji}A_{jl}(\A_{j,-i}^T\x^\star_{-i})^2\right| \le \frac{\gamma\|\mathbf{w}\|_1}{8k}
\end{align*}
where we used H\"{o}lder's inequality and Lemma \ref{lemma:tech3}.
\end{itemize}
All in all, putting everything together we have, with probability $1-\frac{c_2}{2}n^{-10}$,
\begin{equation*}
|\nabla F(\mathbf{w})_i - \E[\nabla F(\mathbf{w})_i]| \le \gamma \frac{\|\mathbf{w}\|_1}{k},
\end{equation*}
for all $\mathbf{w}\in \R^n$ with $\|\mathbf{w}\|_2\le 1$ and $\mathbf{w}_{\S^c}=\mathbf{0}$, and all $i=1,...,n$, which completes the proof of Lemma \ref{lemma:support1}.
\end{proof}

\begin{proof}[Proof of Lemma \ref{lemma:support2}]
We begin by showing the bound $\gamma\frac{\|\x_\S\|_1}{\sqrt{k}} + c_3\delta$. The main idea is similar to the one used in the proof of Lemma \ref{lemma:support1}.

\textbf{Proof of the bound $\gamma\frac{\|\x_\S\|_1}{\sqrt{k}} + c_3\delta$}\\
Writing $\mathbf{w}\in\R^n$ for the vector $\x_\S$ padded with zeroes, i.e.\ $w_i = x_i$ for $i\in \S$ and $w_i =0$ otherwise, we have
\begin{align}
\label{eq:decomposition}
\left|\langle \nabla F(\x) - \nabla f(\x), \x-\x^\star\rangle\right| &\le \left|\langle \nabla F(\x), \x-\x^\star\rangle - \langle \nabla F(\mathbf{w}), \mathbf{w}-\x^\star\rangle \right| \nonumber\\
&\quad + \left|\langle \nabla F(\mathbf{w}), \mathbf{w}-\x^\star\rangle - \E[\langle \nabla F(\mathbf{w}), \mathbf{w}-\x^\star\rangle] \right| \nonumber\\
&\quad +\left|\E[\langle \nabla F(\mathbf{w}), \mathbf{w}-\x^\star\rangle] - \E[\langle \nabla F(\x), \x-\x^\star\rangle] \right|.
\end{align} 
To bound these three terms, we can write
\begin{align*}
\langle \nabla F(\x), \x-\x^\star\rangle  &= \frac{1}{m}\sum_{j=1}^m(\A_j^T\x)^4 - \frac{1}{m}\sum_{j=1}^m (\A_j^T\x)^3(\A_j^T\x^\star) \nonumber\\
&\quad - \frac{1}{m}\sum_{j=1}^m(\A_j^T\x)^2(\A_j^T\x^\star)^2 + \frac{1}{m}\sum_{j=1}^m(\A_j^T\x)(\A_j^T\x^\star)^3 \nonumber\\
&= g_{1,\x}(\A) - g_{2,\x}(\A) - g_{3,\x}(\A) + g_{4,\x}(\A).
\end{align*}
We will show that each of the four terms deviates by at most $\frac{\gamma}{4}\frac{\|\x\|_1}{\sqrt{k}} + \frac{c_3}{4}\delta$ from its mean. 

\textbf{Step 1: Split $g_{i,\x}$ into $\x_\S$ term and rest term}\\
We first show that the first and last term in (\ref{eq:decomposition}) are bounded by $\frac{c_3}{2}\delta$. To do that, we split each of the four terms $g_{i,\x}$ into a part which only depends on $\x_\S$ (which corresponds to $\langle \nabla F(\mathbf{w}), \mathbf{w}-\x^\star\rangle$) and a residual (which corresponds to $\langle \nabla F(\mathbf{x}), \mathbf{x}-\x^\star\rangle - \langle \nabla F(\mathbf{w}), \mathbf{w}-\x^\star\rangle$).

We only go through the computation for $g_{1,\x}$, since the other three terms can be bounded following the same steps. We have
\begin{align*}
g_{1,\x}(\A) &= \frac{1}{m} \sum_{j=1}^m(\A_{j,\S}^T\x_\S + \A_{j,\S^c}^T\x_{\S^c})^4 \\
&= \frac{1}{m}\sum_{j=1}^m(\A_{j,\S}^T\x_\S)^4 + \frac{4}{m}\sum_{j=1}^m(\A_{j,\S}^T\x_\S)^3(\A_{j,\S^c}^T\x_{\S^c}) + \frac{6}{m}\sum_{j=1}^m(\A_{j,\S}^T\x_\S)^2(\A_{j,\S^c}^T\x_{\S^c})^2 \\
&\quad + \frac{4}{m}\sum_{j=1}^m(\A_{j,\S}^T\x_\S)(\A_{j,\S^c}^T\x_{\S^c})^3 + \frac{1}{m}\sum_{j=1}^m(\A_{j,\S^c}^T\x_{\S^c})^4
\end{align*} 
The first term only depends on $\x_\S$, and we will denote it by $h_{1,\x}(\A) = \frac{1}{m}\sum_{j=1}^m(\A_{j,\S}^T\x_\S)^4$.

The other terms can all be bounded as follows:
\begin{align*}
\frac{4}{m}\sum_{j=1}^m(\A_{j,\S}^T\x_\S)^3(\A_{j,\S^c}^T\x_{\S^c}) &= 4\sum_{l\notin \S}x_l \cdot \frac{1}{m}\sum_{j=1}^mA_{jl}(\A_{j,\S}^T\x_\S)^3 \\
&\le 4\|\x_{\S^c}\|_1 \cdot \max_{l\notin \S}\left|\frac{1}{m}\sum_{j=1}^mA_{jl}(\A_{j,\S}^T\x_\S)^3\right| \\
&\le \frac{c_3}{32}\delta,
\end{align*}
where the first inequality holds by H\"{o}lder's inequality, and the second inequality holds by Lemma \ref{lemma:tech7} with probability $1-c'_2n^{-10}$.
Similarly, we can bound
\begin{align*}
\frac{6}{m}\sum_{j=1}^m(\A_{j,\S}^T\x_\S)^2(\A_{j,\S^c}^T\x_{\S^c})^2 &= 6\sum_{l\notin \S}x_l\sum_{s\notin \S}x_s\cdot \frac{1}{m} \sum_{j=1}^mA_{jl}A_{js}(\A_{j,\S}^T\x_\S)^2\\
&\le 6\|\x_{\S^c}\|_1^2\cdot \max_{l,s\notin \S}\sqrt{\frac{1}{m}\sum_{j=1}^mA_{jl}^2A_{js}^2}\cdot \sqrt{\frac{1}{m}\sum_{j=1}^m(\A_{j,\S}^T\x_\S)^4}\\
&\le \frac{c_3}{32}\delta^2,
\end{align*}
if $c_3>0$ is large enough, where we used H\"{o}lder's inequality in the second and Lemmas \ref{lemma:tech1} and \ref{lemma:tech3} in the last line. The same computation yields
\begin{align*}
\frac{4}{m}\sum_{j=1}^m(\A_{j,\S}^T\x_\S)(\A_{j,\S^c}^T\x_{\S^c})^3 &\le \frac{c_3}{32} \delta^3,  \\
\frac{1}{m}\sum_{j=1}^m(\A_{j,\S^c}^T\x_{\S^c})^4 &\le \frac{c_3}{32}\delta^4. 
\end{align*}
Putting everything together, this shows that the rest terms can be bounded by $\frac{c_3}{8}\delta$.
Bounding $g_{2,\x}$, $g_{3,\x}$ and $g_{4,\x}$ the same way, we have, with probability $1-\frac{c_2}{2}n^{-10}$,
\begin{equation*}
\left| \langle \nabla F(\mathbf{x}), \mathbf{x}-\x^\star\rangle - \langle \nabla F(\mathbf{w}), \mathbf{w}-\x^\star\rangle \right| \le \frac{c_3}{2}\delta.
\end{equation*}
For the difference in expectation, we can bound (using $\|\x\|_2\le 1$ and $\|\x_{\S^c}\|_2\le \|\x_{\S^c}\|_1\le\delta$)
\begin{align*}
\E\left[\frac{1}{m}\sum_{j=1}^m(\A_{j,\S}^T\x_\S)^3(\A_{j,\S^c}^T\x_{\S^c})\right] &= \E\left[(\A_{1,\S}^T\x_\S)^3\right]\E\left[\A_{j,\S^c}^T\x_{\S^c}\right] = 0, \\
\E\left[\frac{1}{m}\sum_{j=1}^m(\A_{j,\S}^T\x_\S)^2(\A_{j,\S^c}^T\x_{\S^c})^2\right] &= \E\left[(\A_{1,\S}^T\x_\S)^2\right]\E\left[(\A_{j,\S^c}^T\x_{\S^c})^2\right] \le \delta^2, \\
\E\left[\frac{1}{m}\sum_{j=1}^m(\A_{j,\S}^T\x_\S)(\A_{j,\S^c}^T\x_{\S^c})^3\right] &= \E\left[\A_{1,\S}^T\x_\S\right]\E\left[(\A_{j,\S^c}^T\x_{\S^c})^3\right] = 0, \\
\E\left[\frac{1}{m}\sum_{j=1}^m(\A_{j,\S^c}^T\x_{\S^c})^4\right] &= \E\left[(\A_{j,\S^c}^T\x_{\S^c})^4\right] \le 3\delta^4.
\end{align*}
Repeating this for $g_{2,\x}$, $g_{3,\x}$ and $g_{4,\x}$ shows 
\begin{equation*}
\left| \E[\langle \nabla F(\mathbf{x}), \mathbf{x}-\x^*\rangle] - \E[\langle \nabla F(\mathbf{w}), \mathbf{w}-\x^*\rangle] \right| \le \frac{c_3}{2}\delta.
\end{equation*}

\textbf{Step 2: Bound the term $|\langle \nabla F(\mathbf{w},\mathbf{w}-\x^\star\rangle - \E[\langle \nabla F(\mathbf{w}), \mathbf{w}-\x^\star\rangle] |$}\\
What is left to do is to bound the second term in (\ref{eq:decomposition}), i.e.\ we need to show that the terms $h_{i,\x}$ concentrate around their respective expectations.

From here on, we will write $\x,\x^\star\in\R^k$ for the $k$-dimensional vectors $\x_\S$ and $\x^\star_\S$ respectively, as $w_j=x^\star_j=0$ for $j\notin \S$, so we can ignore the off-support coordinates for notational simplicity.

Let $N_{\epsilon'}$ be an $\epsilon'$-net of the unit sphere in $\R^k$, where $\epsilon' = c_4\gamma n^{-3}$. We will first show that we can bound $|h_{i,\x}(\A)- \E[h_{i,\x}(\A)]|\le \frac{\gamma}{8}\frac{\|\x\|_1}{\sqrt{k}}$ for every $\x\in N_{\epsilon'}$ via concentration of Lipschitz functions for Gaussian random variables. Then, we will extend this bound to every $\x\in \R^k$ with $\|\x\|_2\le 1$.

As the functions $h_{i,\x}$ are not globally Lipschitz continuous, we cannot directly apply Theorem \ref{thm:ref1}. We will first bound the Lipschitz constant of $h_{i,\x}$ restricted to the set $\mathcal{A}$ defined as the intersection of the sets defined in Lemma \ref{lemma:tech2} and Lemma \ref{lemma:tech4}, and then extend this restricted function to a function $\tilde{h}_{i,\x}$ on the entire space such that $\tilde{h}_{i,\x}$ is globally Lipschitz continuous. We can then apply Theorem \ref{thm:ref1} to $\tilde{h}_{i,\x}$, which also provides a high probability bound for $h_{i,\x}$, since by construction $h_{i,\x}(\A)=\tilde{h}_{i,\x}(\A)$ with high probability.

\textbf{Step 2a: Bound the Lipschitz constant of $h_{i,\x}$ restricted to $\mathcal{A}$}\\
Let $\mathcal{A}$ be defined as the intersection of the sets defined in Lemma \ref{lemma:tech2} (with $t = 5\sqrt{\log n}$) and Lemma \ref{lemma:tech4}. By the aforementioned Lemmas, we have $\P[\mathcal{A}^c]\le c_5n^{-12}$. Since $\mathcal{A}$ is convex, the Lipschitz constant of $h_{i,\x}$ restricted to $\mathcal{A}$ is bounded by the norm of its gradient $\|\nabla h_{i,\x}\|_2$ by the mean-value theorem. For any $\{a_{ji}\}\in\mathcal{A}$, we can compute the following.
\begin{itemize}
\item $\frac{\partial}{\partial a_{jl}} h_{1,\x}(\mathbf{a}) = \frac{4}{m}x_l(\mathbf{a}_j^T\x)^3$. We can bound the Lipschitz constant by the squareroot of 
\begin{align*}
\|\nabla h_{1,\x}(\mathbf{a})\|_2^2 = \frac{16}{m^2}\sum_{l=1}^kx_l^2\sum_{j=1}^m(\mathbf{a}_j^T\x)^6 &\le \frac{16}{m^2}\left((15m)^{\frac{1}{6}} + 2\sqrt{2\log (2k)}\|\x\|_1 + 5\sqrt{\log n}\right)^6 \nonumber\\
&\le c_6 \frac{\|\x\|_1^2\log k}{m}
\end{align*}
where we used $\|\x\|_2\le 1$, $\|\x\|_1\le \sqrt{k}$, $m \ge c_1 \max\{k^2 \log^2 n, \;\log^5 n\}$ and Lemma \ref{lemma:tech2}.

\item $\frac{\partial}{\partial a_{jl}} h_{2,\x}(\mathbf{a}) = \frac{1}{m}(x_l^\star(\mathbf{a}_j^T\x)^3 + 3x_l (\mathbf{a}_j^T\x)^2(\mathbf{a}_j^T\x^\star))$. Hence, we can bound the Lipschitz constant by the squareroot of
\begin{align*}
\|\nabla h_{2,\x}(\mathbf{a})\|_2^2 &= \frac{2}{m^2}\sum_{l=1}^k\sum_{j=1}^m(x^\star_l)^2(\mathbf{a}_j^T\x)^6 + 9 x_l^2(\mathbf{a}_j^T\x)^4(\mathbf{a}_j^T\x^\star)^2 \\
&\le \frac{2}{m^2}\left((15m)^{\frac{1}{6}} + 2\sqrt{2\log (2k)}\|\x\|_1 + 5\sqrt{\log n} \right)^6 \\
&\quad + \frac{18}{m} \sqrt{\frac{1}{m}\sum_{j=1}^m(\mathbf{a}_j^T\x)^8} \cdot \sqrt{\frac{1}{m}\sum_{j=1}^m(\mathbf{a}_j^T\x^\star)^4}\\
&\le c_7 \frac{\|\x\|_1^2\log k}{m},
\end{align*}
where we use Lemma \ref{lemma:tech2} to bound the sum $\frac{1}{m}\sum_{j=1}^m(\mathbf{a}_j^T\x)^8$ and Lemma \ref{lemma:tech4} to bound $\frac{1}{m}\sum_{j=1}^m(\mathbf{a}_j^T\x^*)^4$.

\item $\frac{\partial}{\partial a_{ji}} h_{3,\x}(\mathbf{a}) = \frac{2}{m}(x_l^\star(\mathbf{a}_j^T\x)^2(\mathbf{a}_j^T\x^\star) + x_l (\mathbf{a}_j^T\x)(\mathbf{a}_j^T\x^\star)^2)$. Hence, we can bound the Lipschitz constant by the squareroot of
\begin{align*}
\|\nabla h_{3,\x}(\mathbf{a})\|_2^2 &= \frac{4}{m^2}\sum_{l=1}^k\sum_{j=1}^m2(x^\star_l)^2(\mathbf{a}_j^T\x)^4(\mathbf{a}_j^T\x^\star)^2 + 2 x_l^2(\mathbf{a}_j^T\x)^2(\mathbf{a}_j^T\x^\star)^4 \\
&\le \frac{8}{m}\sqrt{\frac{1}{m}\sum_{j=1}^m(\mathbf{a}_j^T\x)^8} \cdot \sqrt{\frac{1}{m}\sum_{j=1}^m(\mathbf{a}_j^T\x^\star)^4} \\
&\quad + \frac{8}{m} \sqrt{\frac{1}{m}\sum_{j=1}^m(\mathbf{a}_j^T\x)^4} \cdot \sqrt{\frac{1}{m}\sum_{j=1}^m(\mathbf{a}_j^T\x^\star)^8}\\
&\le c_8 \frac{\|\x\|_1^2\log k}{m}
\end{align*}
again by Lemmas \ref{lemma:tech2} and \ref{lemma:tech4}.

\item $\frac{\partial}{\partial a_{ji}} h_{4,\x}(\mathbf{a}) = \frac{1}{m}(3x_l^\star(\mathbf{a}_j^T\x)(\mathbf{a}_j^T\x^\star)^2 + x_l(\mathbf{a}_j^T\x^\star)^3)$. Hence, we can bound the Lipschitz constant by the squareroot of
\begin{align*}
\|\nabla h_{4,\x}(\mathbf{a})\|_2^2 &= \frac{2}{m^2}\sum_{l=1}^k\sum_{j=1}^m9(x^\star_l)^2(\mathbf{a}_j^T\x)^2(\mathbf{a}_j^T\x^\star)^4 + x_l^2(\mathbf{a}_j^T\x^\star)^6 \\
&\le \frac{18}{m}\sqrt{\frac{1}{m}\sum_{j=1}^m(\mathbf{a}_j^T\x)^4} \cdot \sqrt{\frac{1}{m}\sum_{j=1}^m(\mathbf{a}_j^T\x^\star)^8} + \frac{2}{m} \cdot \frac{1}{m}\sum_{j=1}^m(\mathbf{a}_j^T\x^\star)^6\\
&\le c_9 \frac{\|\x\|_1^2\log k}{m}
\end{align*}
again by Lemmas \ref{lemma:tech2} and \ref{lemma:tech4}.
\end{itemize}

\textbf{Step 2b: Construct a globally Lipschitz continuous extension of $h_{i,\x}$}\\
We only go through the following steps for the first term $h_{1,\x}$, as the proofs for the other three terms follow the exact same steps.

We have shown that, since $\mathcal{A}$ is a convex set, $h_{1,\x}$ is Lipschitz continuous with Lipschitz constant $Lip(h_{1,\x}) =  \sqrt{c_6 \frac{\|\x\|_1^2\log k}{m}}$ on $\mathcal{A}$.
Consider the Lipschitz extension of $h_{1,\x}$ (this is the one-dimensional case of the Kirszbraun theorem)
\begin{equation*}
\tilde{h}_{1,\x}(\mathbf{a}) = \inf_{\mathbf{a'}\in \mathcal{A}} (h_{1,\x}(\mathbf{a}') + Lip(h_{1,\x})\cdot \|\mathbf{a} - \mathbf{a}'\|_2).
\end{equation*} 
We will show that $\tilde{h}_{1,\x}$ concentrates around its mean, which is different from the mean of $h_{1,\x}$.
Since $h_{1,\x}$ and $\tilde{h}_{1,\x}$ differ only on $\mathcal{A}^c$ (which has probability less than $c_5n^{-12}$), we can bound, using the Cauchy-Schwarz inequality,
\begin{equation*}
\E[|h_{1,\x}(\A)|\Eins_{\mathcal{A}^c}(\A)] \le \sqrt{\E[h_{1,\x}(\A)^2]}\cdot \sqrt{\E[\Eins_{\mathcal{A}^c}(\A)]} \le \frac{c'}{n^6},
\end{equation*}
where we used the fact that $\E[h_{1,\x}(\A)^2] \le 3 + 105/m$. Similarly, since $\tilde{h}_{1,\x}(\mathbf{a}) \le Lip(h_{1,\x})\cdot \|\mathbf{a}\|_2$, 
\begin{equation*}
\E[|\tilde{h}_{1,\x}(\A)|\Eins_{\mathcal{A}^c}(\A)] \le Lip(h_{1,\x})\underbrace{\sqrt{\E[\|\A\|_2^2]}}_{=\sqrt{mk}}\cdot \sqrt{\E[\Eins_{\mathcal{A}^c}(\A)]} \le \frac{c''}{n^5}.
\end{equation*}
All in all, this shows that
\begin{equation*}
\left|\E[h_{1,\x}(\A)] - \E[\tilde{h}_{1,\x}(\A)]\right| \le \frac{c_{10}}{n^5}.
\end{equation*}
Finally, using the triangle inequality and Theorem \ref{thm:ref1}, we have
\begin{equation*}
\P\left[\left|\tilde{h}_{1,\x}(\A) - \E[h_{1,\x}(\A)]\right| > \frac{\gamma}{8}\frac{\|\x\|_1}{\sqrt{k}} \right] \le 2\exp\left(- \frac{(\frac{\gamma}{8}\frac{\|\x\|_1}{\sqrt{k}} - \frac{c_{10}}{n^5})^2}{2c_6\frac{\|\x\|_1^2\log k}{m}}\right) \le 2\exp\left(-c_{11} k \log n\right)
\end{equation*}
for a constant $c_{11} \le \frac{\gamma^2c_1}{128c_6} - \frac{\gamma c_1c_{10}}{8c_6n^5}$.

\textbf{Step 2c: Union bound over $\x\in N_{\epsilon'}$}\\
Taking the union bound over all $\x\in N_{\epsilon'}$, which has cardinality bounded by $(3/\epsilon')^k$, we have
\begin{equation*}
\P\left[\left|\tilde{h}_{1,\x}(\A) - \E[h_{1,\x}(\A)]\right| > \frac{\gamma}{8}\frac{\|\x\|_1}{\sqrt{k}} \text{ for some } \x\in N_{\epsilon'}\right] \le 2\exp\left(-c_{11} k \log n + k\log \frac{3}{\epsilon'}\right).
\end{equation*}
Since $h_{1,\x} = \tilde{h}_{1,\x}$ for any $\x$ on $\mathcal{A}$, this implies
\begin{equation*}
\P\left[\left|h_{1,\x}(\A) - \E[h_{1,\x}(\A)]\right| > \frac{\gamma}{8}\frac{\|\x\|_1}{\sqrt{k}} \text{ for some } \x\in N_{\epsilon'}\right] \le 2\exp\left(-c_{12} k \log n \right) + n^{-12},
\end{equation*} 
for a constant $c_{12} \le c_{11}-3 - \frac{1}{k}\log\frac{3}{c_4\gamma}$, as we have $\epsilon' = c_4\gamma n^{-3}$.

\textbf{Step 3: From $\epsilon'$-net to the full sphere}\\
Next, we show that $|h_{1,\x}(\A) - \E[h_{1,\x}(\A)]|\le \frac{\gamma}{4}\frac{\|\x\|_1}{\sqrt{k}}$ for any $\x\in \R^k$ with $\|\x\|_2=1$; the case $\|\x\|_2<1$ follows by considering the vector $\x/\|\x\|_2$ and rescaling.
                                              
For any $\x\in \R^k$ with $\|\x\|_2=1$, let $\x'\in N_{\epsilon'}$ with $\|\x-\x'\|_2 \le \epsilon' = c_4\gamma n^{-3}$. Then,
\begin{align*}
\left|h_{1,\x}(\A) - \E[h_{1,\x'}(\A)]\right| &\le \left|h_{1,\x}(\A) - h_{1,\x'}(\A)\right| +\left|h_{1,\x'}(\A) - \E[h_{1,\x'}(\A)]\right| \\
&\quad +\left|\E[h_{1,\x'}(\A)] - \E[h_{1,\x}(\A)]\right|.
\end{align*}
The first and third term can be bounded using the indentity $a^4-b^4 = (a^2 + b^2)(a+b)(a-b)$:
\begin{align*}
\left|h_{1,\x}(\A) - h_{1,\x'}(\A)\right| &= \left|\frac{1}{m}\sum_{j=1}^m ((\A_j^T\x)^2 + (\A_j^T\x')^2) (\A_j^T(\x+\x')) (\A_j^T(\x-\x')) \right| \\
&\le \max_j 2\|\A_j\|_2^2 \cdot 2\|\A_j\|_2 \cdot \|\A_j\|_2 \|\x-\x'\|_2 \\
&\le 4 (\sqrt{k} + 5\sqrt{\log n})^4 \|\x-\x'\|_2 \\
&\le \frac{\gamma}{16} \frac{\|\x\|_1}{\sqrt{k}},
\end{align*}
with probability $1-mn^{-12.5}$, where we used the fact that the norm $\|\cdot \|_2$ is 1-Lipschitz and applied Theorem \ref{thm:ref1} to bound the term $\|\A_j\|_2$, and for the last inequality we used $\|\x-\x^*\|_2 \le \epsilon'$.
For the expectation, the same argument yields
\begin{align*}
\left|\E[h_{1,\x}(\A)] - \E[h_{1,\x'}(\A)]\right| &\le \E\left[2\|\A_j\|_2^2 \cdot 2\|\A_j\|_2 \cdot \|\A_j\|_2 \|\x-\x'\|_2\right] \nonumber\\
&= 4 (3k + k(k-1)) \|\x-\x'\|_2 \nonumber\\
&\le \frac{\gamma}{16} \frac{\|\x\|_1}{\sqrt{k}}.
\end{align*}
This completes the proof of 
\begin{equation*}
\P\left[\left|h_{1,\x}(\A) - \E[h_{1,\x}(\A)]\right| < \frac{\gamma}{4}\frac{\|\x\|_1}{\sqrt{k}} \text{ for all } \x\in \R^k \text{ with } \|\x\|_2= 1\right] \ge 1 - \frac{c_2}{16}n^{-10},
\end{equation*}
if $m\le n^{2.5}$. The other case $m>n^{2.5}$ is simpler and can be shown following the same steps, but requires writing the probabilities in terms of $m$ instead of $n$ in the technical lemmas below, and we omit the details to avoid repetition.

Repeating the same steps for the terms $h_{2,\x}$, $h_{3,\x}$ and $h_{4,\x}$ shows that
\begin{align*}
&\P\left[\left|\langle \nabla F(\x), \x-\x^*\rangle - \E\left[\langle \nabla F(\x), \x-\x^*\rangle\right]\right|\le \gamma\frac{\|\x\|_1}{\sqrt{k}} \text{ for all } \x\in \R^k \text{ with } \|\x\|_2= 1\right] \\
\ge& 1 - \frac{c_2}{4}n^{-10}.
\end{align*}
Finally, the bound also holds for any vector $\x$ with $\|\x\|_2<1$ by considering $\x/\|\x\|_2$, and noting that each of the four terms, which make up $\langle \nabla F(\x), \x-\x^*\rangle$, scale at least linearly in $\|\x\|_2$.

\textbf{Proof of the bound $\gamma(\|\x_\S-\x_\S^\star\|_2^2 +\delta^2)$}\\
As the proof of this bound follows essentially the same steps as before, we only give a brief outline of the main ideas.
We can parametrize $\mathbf{z} = \x - \x^\star$, so that
\begin{equation*}
\langle \nabla F(\x), \x - \x^\star\rangle = \frac{1}{m}\sum_{j=1}^m(\A_j^T\mathbf{z})^4 + \frac{3}{m}\sum_{j=1}^m(\A_j^T\mathbf{z})^3(\A_j^T\x^\star) + \frac{2}{m}\sum_{j=1}^m (\A_j^T\mathbf{z})^2(\A_j^T\x^\star)^2
\end{equation*}
Note that, since $\x^\star_{S^c} = \mathbf{0}$, we have $\|\mathbf{z}_{\S^c}\|_1 = \|\x_{\S^c}\|_1 \le \delta$. 

We will show that the last term is close to its expectation. The other two terms can be controlled the same way (albeit easier because of the higher order dependence on $\mathbf{z}$). We have
\begin{align*}
\frac{1}{m}\sum_{j=1}^m (\A_j^T\mathbf{z})^2(\A_j^T\x^\star)^2 &= \frac{1}{m}\sum_{j=1}^m (\A_{j,\S}^T\mathbf{z}_\S)^2(\A_j^T\x^\star)^2 + \frac{2}{m}\sum_{j=1}^m (\A_{j,\S}^T\mathbf{z}_\S)(\A_{j,\S^c}^T\mathbf{z}_{\S^c})(\A_j^T\x^\star)^2 \\
&\quad + \frac{1}{m}\sum_{j=1}^m (\A_{j,\S^c}^T\mathbf{z}_{\S^c})^2(\A_j^T\x^\star)^2\\
&=: B_1 + B_2 + B_3.
\end{align*}
The same computation as in Step 2 shows that
\begin{equation*}
|B_1 - \E[B_1]| \le \frac{\gamma}{12} \|\mathbf{z}_\S\|_2^2
\end{equation*}
holds with probability $1-\frac{c_2}{18}n^{-10}$. A similar computation to Step 1 yields, with probability $1-\frac{c_2}{18}n^{-10}$,
\begin{align*}
|B_2 - \E[B_2]| &= 2\left| \sum_{l\notin \S}z_l \sum_{i\in \S}z_i \cdot \frac{1}{m}\sum_{j=1}^mA_{ji}A_{jl}(\A_j^T\x^\star)^2 \right| \\
&\le \|\mathbf{z}_{\S^c}\|_1\cdot \|\mathbf{z}_\S\|_1\cdot \frac{\gamma}{12k} \\
&\le \frac{\gamma}{12\sqrt{k}} \|\mathbf{z}_\S\|_2\cdot \delta,
\end{align*}
where we used H\"{o}lder's inequality, Lemma \ref{lemma:tech3} (with $m\ge c_1(\gamma)k^2\log^2 n$) and $\|\mathbf{z}_\S\|_1\le \sqrt{k}\|\mathbf{z}_\S\|_2$. The same argument gives
\begin{equation*}
|B_3-\E[B_3]| \le \frac{\gamma}{12k}\delta^2
\end{equation*}
with probability $1-\frac{c_2}{18}n^{-10}$. Combining these bounds, we have, with probability $1-\frac{c_2}{6}n^{-10}$,
\begin{equation*}
\left|\frac{2}{m}\sum_{j=1}^m (\A_j^T\mathbf{z})^2(\A_j^T\x^\star)^2 - \E\left[\frac{2}{m}\sum_{j=1}^m (\A_j^T\mathbf{z})^2(\A_j^T\x^\star)^2\right]\right| \le \frac{\gamma}{3} \left(\|\mathbf{z}_\S\|_2^2 + \delta^2\right),
\end{equation*}
where we used the inequality $2ab\le a^2 + b^2$.
Following the same steps, we also have, with probability $1- 2\frac{c_2}{6}n^{-10}$,
\begin{align*}
\left|\frac{3}{m}\sum_{j=1}^m(\A_j^T\mathbf{z})^3(\A_j^T\x^\star) - \E\left[\frac{3}{m}\sum_{j=1}^m(\A_j^T\mathbf{z})^3(\A_j^T\x^\star)\right]\right| &\le  \frac{\gamma}{3} \left(\|\mathbf{z}_\S\|_2^3 + \delta^3\right), \\
\left|\frac{1}{m}\sum_{j=1}^m(\A_j^T\mathbf{z})^4 - \E\left[\frac{1}{m}\sum_{j=1}^m(\A_j^T\mathbf{z})^4\right]\right| &\le  \frac{\gamma}{3} \left(\|\mathbf{z}_\S\|_2^4 + \delta^4\right),
\end{align*} 
which completes the proof that, with probability $1-\frac{c_2}{2}n^{-10}$, we have
\begin{equation*}
|\langle \nabla F(\x), \x - \x^\star\rangle - \E[\langle \nabla F(\x), \x - \x^\star\rangle]| \le \gamma\left(\|\x_\S-\x_\S^\star\|_2^2 + \delta^2\right) \quad\text{ for all } \x\in \mathcal{X}.
\end{equation*}
\end{proof}

\begin{proof}[Proof of Lemma \ref{lemma:support3}]
Throughout this proof, we will assume $\|\x^\star\|_2=1$ for notational simplicity. The general case $\|\x^\star\|_2\neq 1$ immediately follows by writing $\x^\star = \|\x^\star\|_2\cdot \frac{\x^\star}{\|\x^\star\|_2}$ and $\X(t) = \|\x^\star\|_2\cdot \frac{\X(t)}{\|\x^\star\|_2}$ in what follows. 

We will show that the three inequalities (\ref{eq:claim1}), (\ref{eq:claim2}) and (\ref{eq:claim3}) are satisfied by showing that, as long as all inequalities are satisfied, neither can be violated first. 

Let $i_0$ be the index for which we have the non-zero initialization $X_{i_0}(0) >0$. As we can only recover the signal $\x^\star$ up to a global sign from phaseless measurements, we can assume without loss of generality that $x^\star_{i_0} >0$, since we can otherwise replace $\x^\star$ by $-\x^\star$ in the proof below. That is, we need to show (\ref{eq:claim1}) with $\xi = +1$.

\textbf{Step 1: (\ref{eq:claim1}) continues to hold as long as (\ref{eq:claim2}) holds}\\
We prove this inequality by contradiction. Define $t_1 = \inf\{t\ge 0: X_i(t)x^\star_i <0 \text{ for some } i\}$ as the first time inequality (\ref{eq:claim1}) is violated. Assume that $t_1<T$, and let $i$ be the index for which $X_i(t)x^\star_i<0$ first occurs. Clearly, this is only possible for a coordinate $i\in \S$, and by continuity we must have $X_i(t_1) = 0$.

Without loss of generality, assume that $x^\star_i>0$. We will show that
\begin{equation*}
\frac{d}{dt}X_i(t_1) = -\sqrt{X_i(t_1)^2 + \beta^2} \cdot \nabla F(\X(t_1))_i > 0,
\end{equation*}
that is $X_i(t)$ must become positive for $t$ close enough to $t_2$, which is a contradiction to the definition of $t_1$, and hence we must have $t_1 \ge T$.

We can bound, since both (\ref{eq:claim1}) and (\ref{eq:claim2}) hold at $t_1$, 
\begin{equation}\label{eq:innerproduct}
\X(t_1)^T\x^* \ge \|\X_\S(t_1)\|_1 x^\star_{min} \ge \|\X_\S(t_1)\|_1\frac{c}{\sqrt{k}},
\end{equation}
and hence
\begin{equation*}
\nabla f(\X(t_1))_i = -2(\X(t)^T\x^\star)x^\star_i \le - 2\|\X_\S(t_1)\|_1\frac{c^2}{k},
\end{equation*}
where we used $x^\star_i \ge x^\star_{min} \ge \frac{c}{\sqrt{k}}$.

As we assume $t_1 < T$, we can use Lemma \ref{lemma:support1} to bound
\begin{equation*}
|\nabla F(\X(t_1))_i - \nabla f(\X(t_1))_i| \le 0.1\frac{c^2\|\X_\S(t_1)\|_1}{k} \le \frac{1}{20}|\nabla f(\X(t_1))_i|
\end{equation*}
with probability $1-c_2n^{-10}$ if $c_1$ is sufficiently large, since by assumption $\|\X_{\S^c}(t_1)\|_1\le \delta\le \frac{c_3}{n}$.
Hence, we have $\nabla F(\X(t_1))_i < 0$, which implies $\frac{d}{dt}X_i(t_1)>0$ and contradicts the definition of $t_1$, and we must have $t_1 \ge T$.

\textbf{Step 2: (\ref{eq:claim2}) continues to hold as long as (\ref{eq:claim1}) and (\ref{eq:claim3}) hold}\\
We first show the lower bound in (\ref{eq:claim2}). For $t=0$, 
we have, by standard concentration for sub-exponential random variables, $\frac{1}{m}\sum_{j=1}^mY_j>1-9\sqrt{\frac{\log n}{m}}$ with probability $1-2n^{-10}$. Hence, the initialization (\ref{eq:initialization}) satisfies $\|\X(0)\|_2^2\ge \frac{1}{3}-3\sqrt{\frac{\log n}{m}}$.

Define $t_2 = \inf \{t\ge 0: \|\X(t)\|_2^2<\frac{1}{3} - 3\sqrt{\frac{\log n}{m}}\}$, and assume that $t_2<T$. By continuity, we must have $\|\X(t_2)\|_2^2 = \frac{1}{3} - 3\sqrt{\frac{\log n}{m}}$. We will show that
\begin{align*}
\frac{d}{dt}\|\X(t_2)\|_2^2 &= -2\sum_{i=1}^nX_i(t_2)\frac{d}{dt}X_i(t_2) = -2\sum_{i=1}^n X_i(t_2) \sqrt{X_i(t_2)^2+\beta^2}\cdot \nabla F(\X(t_2))_i
\end{align*}
is positive, which implies $\|\X(t)\|_2^2>\frac{1}{3}$ for $t$ close enough to $t_2$ and contradicts the definition of $t_2$. Hence, we must have $t_2\ge T$. To this end, recall that in the previous step we have shown
\begin{equation}\label{eq:difgradient}
|\nabla F(\X(t_2))_i - \nabla f(\X(t_2))_i| \le 0.1\|\X_\S(t_2)\|_1 \frac{c^2}{k},
\end{equation}
with probability $1-c_2n^{-10}$ for all $i\in [n]$.
Further, since $3\|\X(t_2)\|_2^2 - 1 < 0$, we have, using (\ref{eq:innerproduct}) and the assumption $\|\X_{\S^c}(t)\|_1\le \delta$,
\begin{align}\label{eq:popgradient}
\nabla f(\X(t_2))_i \begin{cases} \le -2\|\X_\S(t_2)\|_1\frac{c^2}{k} \quad & x^\star_i>0 \\ \in (-\delta, \delta) & x^\star_i=0  \\ \ge 2\|\X_\S(t_2)\|_1\frac{c^2}{k} & x^\star_i<0 \end{cases}
\end{align}
In order to bound $\frac{d}{dt}\|\X(t_2)\|_2^2$, we write
\begin{align*}
\left|\sum_{i\notin \S} X_i(t_2) \sqrt{X_i(t_2)^2+\beta^2}\cdot \nabla F(\X(t_2))_i\right| \le \|\X_{\S^c}(t_2)\|_1 \le \delta,
\end{align*}
where we used $|\sqrt{X_i(t_2)^2+\beta^2}\cdot \nabla F(\X(t_2))_i|\le 1$. Using $\|\X_\S(t_2)\|_1 \ge \frac{1}{2}$, we have
\begin{align*}
-\sum_{i\in \S} X_i(t_2) \sqrt{X_i(t_2)^2+\beta^2}\cdot \nabla F(\X(t_2))_i \ge\sum_{i\in \S} X_i(t_2)^2 \cdot \frac{c^2}{2k} \ge 0.15 \frac{c^2}{k},
\end{align*}
where we used $\|\X_\S(t_2)\|_2^2 = \|\X(t_2)\|_2^2 - \|\X_{\S^c}(t_2)\|_2^2 \ge 0.3$. This shows that we must have $\frac{d}{dt}\|\X(t_2)\|_2^2 >0$, which contradicts the definition of $t_2$, and hence we must have $t_2 \ge T$.

The upper bound in (\ref{eq:claim2}) is an immediate consequence of (\ref{eq:claim3}): by the Cauchy-Schwarz inequality we have $\X(t)^T\x^\star\le \|\X(t)\|_2$, so, for $3\|\X(t)\|_2^2 - 1 > 0$, we have
\begin{equation*}
\frac{2\|\X(t)\|_2}{3\|\X(t)\|_2^2-1} \ge \frac{2(\X(t)^T\x^\star)}{3\|\X(t)\|_2^2-1} \ge \frac{1}{\sqrt{3}} \quad \Rightarrow \quad \|\X(t)\|_2 \le \frac{1+\sqrt{2}}{\sqrt{3}} < \sqrt{2}
\end{equation*}
by solving the quadratic form.

\textbf{Step 3: (\ref{eq:claim3}) continues to hold as long as (\ref{eq:claim1}) and (\ref{eq:claim2}) hold}\\
The proof of (\ref{eq:claim3}) follows the same recipe as the two previous proofs, although the calculations are more complicated. When $3\|\X(t)\|_2^2-1 \le 0$, there is nothing to show. Otherwise, we can consider the ratio $R(t) = \frac{2(\X(t)^T\x^\star)}{3\|\X(t)\|_2^2 - 1}$ and show that it is bounded by $\frac{1}{\sqrt{3}}$ for all $t\le T$. At $t=0$, we can show, as in Step 2, that $\|\X(0)\|_2^2\le \frac{1}{3} + 3\sqrt{\frac{\log n}{m}}$, so, together with $\X(0)^T\x^\star\ge \|\X_\S(0)\|_1x^\star_{min}\ge \frac{c}{2\sqrt{k}}$, this implies $R(0) \ge \frac{1}{\sqrt{3}}$.

 Let $t_3 = \inf\{t\ge 0: R(t) < \frac{1}{\sqrt{3}}\}$ and assume $t_3<T$ as before. For notational simplicity, we will omit the argument $t_3$ in $X_i(t_3)$ in what follows. We can compute
\begin{align}\label{eq:diffratio}
\frac{d}{dt}R(t_3) = \sum_{i=1}^n-\sqrt{X_i^2 + \beta^2}\cdot \nabla F(\X)_i \cdot \frac{2x^\star_i(3\|\X\|_2^2-1) - 2(\X^T\x^\star)\cdot 6X_i}{(3\|\X\|_2^2-1)^2}.
\end{align}
Since we assume $R(t_3) = \frac{1}{\sqrt{3}}$, we have, for $X_i,x^\star_i>0$,
\begin{align}\label{eq:ratiosign}
2x^\star_i(3\|\X\|_2^2-1) - 2(\X^T\x^\star)\cdot 6X_i &> 0 \nonumber\\
\Leftrightarrow \hspace{53mm}  X_i &< \frac{1}{\sqrt{3}}x^\star_i, 
\end{align}
and the analogous result for $X_i,x^*_i<0$. In order to show that $\frac{d}{dt}R(t_3) >0$, the idea is to show that coordinates with small magnitude $|X_i| < \frac{1}{\sqrt{3}}|x^\star_i|$ are increasing in magnitude, and conversely coordinates with large magnitude $|X_i|>\frac{1}{\sqrt{3}}|x^\star_i|$ are decreasing in magnitude. To this end, we split the coordinates $i\in [n]$ into five subsets:
\begin{align*}
\S^c &= \{i\in [n]: x^\star_i = 0\} \\
\S_1 &= \left\{i\in \S: |X_i| < \left(\frac{1}{\sqrt{3}} - 0.1\right) |x^\star_i|\right\} \\
\S_2 &= \left\{i\in \S: \left(\frac{1}{\sqrt{3}} - 0.1\right) |x^\star_i| \le |X_i| < \left(\frac{1}{\sqrt{3}} + 0.1\right) |x^\star_i|\right\} \\
\S_3 &= \left\{i\in \S: \left(\frac{1}{\sqrt{3}} + 0.1\right) |x^\star_i| \le |X_i| < \frac{2}{\sqrt{3}} |x^\star_i|\right\} \\
\S_4 &= \left\{i\in \S: |X_i| \ge \frac{2}{\sqrt{3}} |x^\star_i|\right\} 
\end{align*}
We will bound the sum (\ref{eq:diffratio}) on each of these five sets.
\begin{itemize}
\item For $\S^c$, we have
\begin{equation*}
\left|\sum_{i\in \S^c} \sqrt{X_i^2 + \beta^2}\cdot \nabla F(\X)_i \cdot \cdot 2(\X^T\x^\star)\cdot 6X_i\right| \le \frac{2c^2\delta^2}{k}\|\X_\S\|_1\cdot (\X^T\x^\star),
\end{equation*}
where we used the fact that $\|\X_{\S^c}\|_2^2 \le \|\X_{\S^c}\|_1^2\le \delta^2$ and that $|\nabla F(\X)_i|\le 0.1\|\X_\S\|_1\frac{c^2}{k}$ by (\ref{eq:difgradient}).

\item For an $i\in \S_1$ with $x^\star_i>0$, we have
\begin{equation*}
\nabla f(\X)_i = (3\|\X\|_2^2 - 1)X_i - 2(\X^T\x^\star)x^\star_i \le -0.2\cdot \sqrt{3} \|\X_\S\|_1\frac{c^2}{k},
\end{equation*}
where we used that $\sqrt{3}\cdot 2(\X^T\x^\star) = 3\|\X\|_2^2-1$ and $X_i < (\frac{1}{\sqrt{3}} - 0.1)x^\star_i$. Together with the bound (\ref{eq:difgradient}), this shows that $\nabla F(\X)_i <0$.

 Similarly, for $i\in \S_1$ with $x^\star_i<0$, we can show $\nabla F(\X)_i >0$, and hence, recalling (\ref{eq:ratiosign}),
\begin{equation*}
\sum_{i\in \S_1} -\sqrt{X_i^2 + \beta^2}\cdot \nabla F(\X)_i \cdot \left(2x^\star_i(3\|\X\|_2^2-1) - 2(\X^T\x^\star)\cdot 6X_i\right) \ge 0,
\end{equation*}
as each summand is non-neagtive.

\item For $\S_3$, we can use the same argument as for $\S_1$ to show that
\begin{equation*}
\sum_{i\in \S_3} -\sqrt{X_i^2 + \beta^2}\cdot \nabla F(\X)_i \cdot \left(2x^\star_i(3\|\X\|_2^2-1) - 2(\X^T\x^\star)\cdot 6X_i\right) \ge 0.
\end{equation*}

\item For $\S_2$, we need to show that the sum  
\begin{align*}
\sum_{i\in \S_2}-\sqrt{X_i^2 + \beta^2}\cdot \nabla F(\X)_i \cdot \left(2x^\star_i(3\|\X\|_2^2-1) - 2(\X^T\x^\star)\cdot 6X_i\right)
\end{align*}
is bounded from below. Let $i\in \S_2$ with $x^\star_i>0$. If $X_i < \frac{1}{\sqrt{3}}x^\star_i$ and $\nabla F(\X)_i<0$, or if $X_i > \frac{1}{\sqrt{3}}x^\star_i$ and $\nabla F(\X)_i>0$, then the summand is non-negative, and there is nothing to show. Therefore, take an $i\in \S_2$ with $X_i < \frac{1}{\sqrt{3}}x^\star_i$ and $\nabla F(\X)_i>0$. We can bound
\begin{align*}
\sqrt{X_i^2 + \beta^2} &\le (1+\beta) |X_i|, \\ 
\nabla f(\X)_i &<0 ,\\
\nabla F(\X)_i &\le \nabla f(\X)_i + |\nabla f(\X)_i - \nabla F(\X)_i|\le 0.1 \|\X_\S\|_1\frac{c^2}{k}, 
\end{align*}
where we used the bound (\ref{eq:difgradient}). Recalling the definition of $\S_2$, we have
\begin{align*}
2x^\star_i(3\|\X\|_2^2-1) - 2(\X^T\x^\star)\cdot 6X_i &= 12(\X^T\x^\star) \left(\frac{1}{\sqrt{3}}x^\star_i - X_i\right) \le 3(\X^T\x^\star)X_i.
\end{align*}
Putting this together, we have
\begin{align*}
&-\sqrt{X_i^2 + \beta^2}\cdot \nabla F(\x)_i \cdot \left(2x^\star_i(3\|\X\|_2^2-1) - 2(\X^T\x^\star)\cdot 6X_i\right)\\
 \ge &-\frac{0.3(1+\beta)c^2}{k} \|\X_\S\|_1\cdot (\X^T\x^\star) \cdot X_i^2.
\end{align*}
Together with the analogous bound for the case $X_i > \frac{1}{\sqrt{3}}x^\star_i$ and $\nabla F(\X)_i<0$, this yields
\begin{align*}
&\sum_{i\in \S_2}-\sqrt{X_i^2 + \beta^2}\cdot \nabla F(\X)_i \cdot \left(2x^\star_i(3\|\X\|_2^2-1) - 2(\X^T\x^\star)\cdot 6X_i\right) \nonumber\\
\ge &-\frac{0.3(1+\beta)c^2}{k} \|\X_\S\|_1\cdot (\X^T\x^\star) \cdot \|\X_{\S_2}\|_2^2.
\end{align*}

\item Finally, for $\S_4$ we can bound
\begin{align*}
\sqrt{X_i^2 + \beta^2} &\ge |X_i|, \\ 
|\nabla f(\X)_i| &= 2(\X^T\x^\star) \left|\sqrt{3}X_i-x^\star_i\right| \ge 2(\X^T\x^\star)|x_i^\star| \ge 2\|\X_\S\|_1\frac{c^2}{k},\\
|\nabla F(\X)_i| &\ge |\nabla f(\X)_i| - |\nabla f(\X)_i - \nabla F(\X)_i|\ge 1.9 \|\X_\S\|_1\frac{c^2}{k}, 
\end{align*}
where we used the bound (\ref{eq:difgradient}). Recalling the definition of $\S_4$, we have
\begin{align*}
|2x^\star_i(3\|\X\|_2^2-1) - 2(\X^T\x^\star)\cdot 6X_i| &= 12(\X^T\x^\star) \left|\frac{1}{\sqrt{3}}x^\star_i - X_i\right| \ge 6(\X^T\x^\star)X_i.
\end{align*}
Putting everything together, we can bound
\begin{align*}
&\sum_{i\in \S_2}-\sqrt{X_i^2 + \beta^2}\cdot \nabla F(\X)_i \cdot \left(2x^\star_i(3\|\X\|_2^2-1) - 2(\X^T\x^\star)\cdot 6X_i\right) \nonumber\\
\ge &\frac{11.4c^2}{k} \|\X_\S\|_1\cdot (\X^T\x^\star) \cdot \|\X_{\S_4}\|_2^2
\end{align*}
\end{itemize}
Putting these five sums together, we have shown that $\frac{d}{dt}R(t_3)>0$ if we can show that
\begin{equation*}
\frac{11.4c^2}{k} \|\X_\S\|_1\cdot (\X^T\x^\star) \cdot \|\X_{\S_4}\|_2^2\ge\left(\frac{0.3(1+\beta)c^2}{k} \|\X_{\S_2}\|_2^2 + \frac{2c^2\delta^2}{k}\right) \|\X_\S\|_1\cdot (\X^T\x^\star) 
\end{equation*}
Since $\delta \le c_3/n$ is sufficiently small, this reduces to showing
\begin{equation}\label{eq:s2s4}
11\|\X_{\S_4}\|_2^2 \ge 0.3(1+\beta)\|\X_{\S_2}\|_2^2.
\end{equation}
Now, we can rearrange the equality $R(t) = \frac{1}{\sqrt{3}}$ to obtain
\begin{align*}
&3\left(\|\X_{\S_1}\|_2^2 + \|\X_{\S_2}\|_2^2 + \|\X_{\S_3}\|_2^2 + \|\X_{\S_4}\|_2^2 + \|\X_{\S^c}\|_2^2\right) - 1 \\
=& \sqrt{3}\cdot 2\left(\X_{\S_1}^T\x^\star_{\S_1} + \X_{\S_2}^T\x^\star_{\S_2} + \X_{\S_3}^T\x^\star_{\S_3} + \X_{\S_4}^T\x^\star_{\S_4}\right).
\end{align*} 
By definition, we have $3X_i^2 < \sqrt{3}\cdot 2X_ix_i^\star$ for $i\in S_1\cup S_2\cup S_3 \cup S_4$, i.e.\ $3\|\X_{\S_j}\|_2^2 < \sqrt{3}\cdot 2\X_{\S_j}^T\x^\star_{\S_j}$ for $j=1,...,4$. Further, we have $\|\X_{\S^c}\|_2^2 \le \delta$, so
\begin{equation*}
3\|\X_{\S_4}\|_2^2 - 2\sqrt{3}\cdot \X_{\S_4}^T\x^\star_{\S_4} > 1-3\delta + \left(2\sqrt{3}\cdot \X_{\S_2}^T\x^\star_{\S_2} - 3\|\X_{\S_2}\|_2^2\right). 
\end{equation*}
By the definition of $\S_2$, we have 
\begin{equation*}
2\sqrt{3}\cdot X_ix^\star_i - 3X_i^2 \ge \left(\frac{2\sqrt{3}}{1/\sqrt{3} + 0.1} - 3\right)X_i^2 \ge 2.2X_i^2
\end{equation*}
for $i\in \S_2$. Since $\X_{\S_4}^T\x_{\S_4}^* > 0$ by (\ref{eq:claim1}), this gives
\begin{equation*}
3\|\X_{\S_4}\|_2^2 > 1- 3\delta + 2.2 \|\X_{\S_2}\|_2^2,
\end{equation*}
which shows that (\ref{eq:s2s4}) holds, thus completing the proof of (\ref{eq:claim3}).
\end{proof}

\section{Technical lemmas}
\label{appendix:f}
In this section, we collect technical lemmas and concentration bounds used to prove the supporting Lemmas \ref{lemma:support1}--\ref{lemma:support3}.

\begin{theorem}\label{thm:ref1}(Proposition 34 \cite{V12})  
Let $g:\R^n\rightarrow \R$ be a Lipschitz continuous function with Lipschitz constant $K$, i.e.\ $|g(\x) - g(\mathbf{y})|\le K\|\x-\mathbf{y}\|_2$ for all $\x,\mathbf{y}\in\R^n$. Let $\A\in\R^n$ be a standard normal random vector. Then, for any $\lambda>0$, we have
\begin{equation*}
\P\left[|g(\A) - \E[g(\A)]| \ge \lambda \right] \le 2\exp\left(-\frac{\lambda^2}{2K^2}\right)
\end{equation*}
\end{theorem}

\begin{theorem}\label{thm:ref2}(Theorems 3.6, 3.7 \cite{CL06})  
Let $X_i$ be independent random variables satisfying $|X_i|\le M$ for $i\in [n]$. Let $X = \sum_{i=1}^nX_i$ and $\|X\| = \sqrt{\sum_{i=1}^n\E[X_i^2]}$. Then, we have
\begin{equation*}
\P\left[|X - \E[X]| \ge \lambda \right] \le \exp\left(-\frac{\lambda^2}{2(\|X\|^2 + M\lambda/3)}\right)
\end{equation*}
\end{theorem}

We state the following Lemma from \cite{CLM16} without proof; the first and last inequality were not shown in \cite{CLM16}, but can be done the same way as in the proof of Lemma A.5 in \cite{CLM16}. Convexity follows from the convexity of the operator norm.
\begin{lemma}\label{lemma:tech1} (Lemma A.5 \cite{CLM16})
Let $\{A_{ji}\}_{j\in [m], i\in [k]}$ be a collection of i.i.d.\ $\gauss(0,1)$ random variables. Let $t>0$, and let $\mathcal{A}\subseteq \R^{m\times k}$ be the set consisting of all $\{a_{ji}\}\in \R^{m\times k}$ satisfying the following:
\begin{align}
\|\mathbf{a}\|_{2\rightarrow 2} &\le \sqrt{m} + \sqrt{k} + t \label{eq:tech1_1}\\
\|\mathbf{a}\|_{2\rightarrow 4} &\le (3m)^{\frac{1}{4}} + \sqrt{k} + t \label{eq:tech1_2}\\
\|\mathbf{a}\|_{2\rightarrow 6} &\le (15m)^{\frac{1}{6}} + \sqrt{k} + t \label{eq:tech1_3}\\
\|\mathbf{a}\|_{2\rightarrow 8} &\le (105m)^{\frac{1}{8}} + \sqrt{k} + t \label{eq:tech1_4}
\end{align}
where we write
\begin{equation*}
\|\mathbf{a}\|_{2\rightarrow p} = \sup_{\|x\|_2 \le 1} \|\mathbf{a}\x\|_p.
\end{equation*}
Then, we have $P[\{A_{ji}\}\in \mathcal{A}] \ge 1 - 4\exp(-t^2/2)$. Further, the set $\mathcal{A}$ is convex.
\end{lemma}

The following lemma is a slight modification of the previous result.
\begin{lemma}\label{lemma:tech2}
Let $\{A_{ji}\}_{j\in [m], i\in [k]}$ be a collection of i.i.d.\ $\gauss(0,1)$ random variables. Let $t>0$, and let $\mathcal{A}\subseteq \R^{m\times k}$ be the set consisting of all $\{a_{ji}\}\in \R^{m\times k}$ satisfying the following:
\begin{align}
\left(\sum_{j=1}^m (\mathbf{a}_j^T\x)^4\right)^{\frac{1}{4}} &\le (3m)^{\frac{1}{4}} + 2\sqrt{2\log (2k)}\|\x\|_1 + t  && \text{ for all } \x\in\R^k \text{ with } \|\x\|_2= 1 \label{eq:tech2_1} \\
\left(\sum_{j=1}^m (\mathbf{a}_j^T\x)^6\right)^{\frac{1}{6}} &\le (15m)^{\frac{1}{6}} + 2\sqrt{2\log (2k)}\|\x\|_1 + t && \text{ for all } \x\in\R^k \text{ with } \|\x\|_2= 1 \label{eq:tech2_2} \\
\left(\sum_{j=1}^m (\mathbf{a}_j^T\x)^8\right)^{\frac{1}{8}} &\le (105m)^{\frac{1}{8}} + 2\sqrt{2\log (2k)}\|\x\|_1 + t && \text{ for all } \x\in\R^k \text{ with } \|\x\|_2= 1 \label{eq:tech2_3} 
\end{align}
Then, we have $\P[\{A_{ji}\}\in\mathcal{A}]\ge 1 - 3\left\lceil \sqrt{k}\right\rceil\exp(-t^2/2)$. Further, the set $\mathcal{A}$ is convex.
\end{lemma}

\begin{lemma}\label{lemma:tech3}
Let $\x^\star\in \R^n$ be a $k$-sparse vector with $\|\x^\star\|_2=1$, and let $\S=\{1\le i\le n: x^\star_i\neq 0\}$ be its support. Let $\{A_{ji}\}_{j\in [m], i\in [n]}$  be a collection of i.i.d.\ $\gauss(0,1)$ random variables. There exist constants $C,c_1>0$, such that if $m\ge c_1\log^5n$, then
\begin{align}
\left|\frac{1}{m}\sum_{j=1}^mA_{ji}^2A_{js}A_{jl} - \E\left[A_{1i}^2A_{1s}A_{1l}\right]\right|&\le C\sqrt{\frac{\log n}{m}} && \text{for all } i,l,s\in [n], \label{eq:tech3_1}\\
\left|\frac{1}{m}\sum_{j=1}^mA_{ji}^2A_{jl}(\A_{j,-i}^T\x^\star_{-i}) - \E\left[A_{1i}^2A_{1l}(\A_{1,-i}^T\x^\star_{-i})\right]\right| &\le C\sqrt{\frac{\log n}{m}} && \text{for all } i,l\in[n], \label{eq:tech3_2} \\
\left|\frac{1}{m}\sum_{j=1}^mA_{ji}A_{jl}(\A_{j,-i}^T\x^\star_{-i})^2 - \E\left[A_{1i}A_{1l}(\A_{1,-i}^T\x^\star_{-i})^2\right]\right| &\le C\sqrt{\frac{\log n}{m}} && \text{for all } i,l\in[n], \label{eq:tech3_3} \\
\left|\frac{1}{m}\sum_{j=1}^mA_{ji}^8 - 105\right| &\le C\sqrt{\frac{\log^5 n}{m}} && \text{for all } i\in[n] \label{eq:tech3_4}, 
\end{align}
holds with probability at least $1-c_2n^{-13}$, where $c_2$ is a universal constant (independent of $C, c_1$), 
\end{lemma}

\begin{lemma}\label{lemma:tech4}
Let $\x^\star\in \R^n$ be a $k$-sparse vector with $\|\x^\star\|_2=1$, and let $\{A_{ji}\}_{j\in [m], i\in [n]}$  be a collection of i.i.d.\ $\gauss(0,1)$ random variables. There exists constants $C, c_1>0$ such that the following holds. Let $\mathcal{A}\subseteq \R^{m\times n}$ be the set consisting of all $\{a_{ji}\}\in\R^{m\times n}$ satisfying the following:
\begin{align*}
\left|\frac{1}{m}\sum_{j=1}^m(\mathbf{a}_j^T\x^\star)^4 - 3\right| &\le C\sqrt{\frac{\log n}{m}} \\
\left|\frac{1}{m}\sum_{j=1}^m(\mathbf{a}_j^T\x^\star)^6 - 15\right| &\le C\sqrt{\frac{\log^3 n}{m}} \\
\left|\frac{1}{m}\sum_{j=1}^m(\mathbf{a}_j^T\x^\star)^8 - 105\right| &\le C\sqrt{\frac{\log^5 n}{m}}.
\end{align*}
Then, if $m\ge c_1\log^5n$, we have $\P[\{A_{ji}\}\in \mathcal{A}] \ge 1 - c_2n^{-12}$, where $c_2$ is a universal constant (independent of $C, c_1$). Further, the set $\mathcal{A}$ is convex.
\end{lemma}

\begin{lemma}\label{lemma:tech5}
Let $\{A_{ji}\}_{j\in [m], i\in [k]}$ be a collection of i.i.d.\ $\gauss(0,1)$ random variables. There exist constants $C,c_1,c_2>0$ such that the following holds. Let $m\ge c_1\max\{k^2\log^2 n, \; \log^5n\}$, where $n\ge k$ is any natural number. Then, the set $\mathcal{A}\subseteq \R^{m\times k}$ defined by
\begin{equation*}
\mathcal{A} := \left\{\{a_{ji}\}\in \R^{m\times k}: \frac{1}{m}\sum_{j=1}^m\left(\sum_{l\in\mathcal{L}}a_{jl}\right)^4\le C|\mathcal{L}|^2 \text{ for all } \mathcal{L}\subseteq[k] \right\}
\end{equation*}
satisfies $\P[\{A_{ji}\}\in\mathcal{A}]\ge 1 - c_2n^{-11}$.
Further, the set $\mathcal{A}$ is convex.
\end{lemma}

\begin{lemma}\label{lemma:tech6}
Let $\{A_{ji}\}_{j\in [m], i\in [n]}$ be a collection of i.i.d.\ $\gauss(0,1)$ random variables, and let $\mathcal{S}\subset [n]$ be a subset of cardinality $|\mathcal{S}| = k$. There exist constants $c_1, c_2>0$ such that, for any constant $c_{\epsilon}>0$, there is a $C>0$ such that the following holds. Let $m\ge c_1\max\{k^2\log n, \;\log^5n \}$, and let $\mathcal{A}\subseteq \R^{m\times n}$ be the set consisting of all $\{a_{ji}\}\in\R^{m\times n}$ satisfying the following: 
\begin{itemize}
\item For any $i\in [n]$ and any subset $\mathcal{K}\subseteq \S\backslash \{i\}$, let $N_{\epsilon'}^{\mathcal{K}}$ be an $\epsilon'$-net of the unit sphere in $\R^{\mathcal{K}}$, where $\epsilon' = c_{\epsilon}/k$. Then, for any $\mathcal{L}\subseteq \mathcal{K}$ with $|\mathcal{L}|\ge \frac{1}{2}|\mathcal{K}|$, and $\x\in N_{\epsilon'}^{\mathcal{K}}$,
\begin{equation}\label{eq:tech6_1}
\left|\sum_{l\in \mathcal{L}}\frac{1}{m}\sum_{j=1}^ma_{ji}a_{jl}\left(\mathbf{a}_{j,-i}^T\x\right)^2\right| \le C |\mathcal{L}|\frac{\log n}{\sqrt{m}}.
\end{equation}

\item For any $i\in [n]$ and any subset $\mathcal{K}\subseteq \S\backslash \{i\}$, let $N_{\epsilon'}^{\mathcal{K}}$ be an $\epsilon'$-net of the unit sphere in $\R^{\mathcal{K}}$, where $\epsilon' = c_{\epsilon}/k$. Then, for any $\mathcal{L}\subseteq \mathcal{K}$ with $|\mathcal{L}|\ge \frac{1}{2}|\mathcal{K}|$, and $\x\in N_{\epsilon'}^{\mathcal{K}}$,
\begin{equation}\label{eq:tech6_2}
\left|\sum_{l\in \mathcal{L}}\frac{1}{m}\sum_{j=1}^ma_{ji}^2\left(\mathbf{a}_{j,-\{l,i\}}^T\x_{-l}\right)a_{jl}\right| \le C |\mathcal{L}|\sqrt{\frac{\log n}{m}}.
\end{equation}
\end{itemize}
Then, we have $\P[\{A_{ji}\} \in \mathcal{A}] \ge 1-c_2n^{-10}$.
\end{lemma}

\begin{lemma}\label{lemma:tech7}
Let $\{A_{ji}\}_{j\in [m], i\in [n]}$  be a collection of i.i.d.\ $\gauss(0,1)$ random variables, and let $\mathcal{S}\subset [n]$ be a subset of cardinality $|\mathcal{S}| = k$. Define $\mathcal{X} = \{\x\in \R^n: \x_{\S^c} = \mathbf{0}, \|\x\|_2\le 1\}$. There exist constants $C, c_1 >0$, such that if $m\ge c_1\max\{k^2\log^2n, \;\log^5n\}$, then for all $i\in [n]$, 
\begin{align}
\left|\frac{1}{m}\sum_{j=1}^m A_{ji}(\A_{j,-i}^T\x_{-i})^3\right| &\le C\|\x\|_1 \frac{\log n}{\sqrt{m}}, \quad &&\text{for all } \x\in \mathcal{X}, \text{ and} \label{eq:tech7_1} \\
\left|\sum_{l\neq i} x_l \sum_{s\neq l,i}x_s \frac{1}{m}\sum_{j=1}^mA_{ji}^2A_{jl}A_{js}\right| &\le C\|\x\|_1 \sqrt{\frac{\log n}{m}}&&\text{for all } \x\in\mathcal{X} \label{eq:tech7_2}
\end{align}
holds with probability at least $1-c_2n^{-10}$, where $c_2$ is a universal constant (independent of $C,c_1$).
\end{lemma}

In the following, we prove Lemmas \ref{lemma:tech2}--\ref{lemma:tech7}.
\begin{proof}[Proof of Lemma \ref{lemma:tech2}]
The proof follows that of Lemma A.5 of \cite{CLM16} closely. Since the proof of the three inequalities is the same, we only show (\ref{eq:tech2_3}). 

Define
\begin{equation*}
\|\mathbf{a}\|_{2\rightarrow 8, \lambda} = \max \{\|\mathbf{a}\x\|_8 : \|\x\|_2 = 1, \|\x\|_1\le \lambda\}.
\end{equation*}
We will show that $\|\A\|_{2\rightarrow 8, \lambda}\le (105m)^{\frac{1}{8}} + \sqrt{2\log (2k)}\lambda + t$ with probability $1-2\exp(-t^2/2)$. The result then follows by taking the union bound over $\lambda = 1,...,\left\lceil \sqrt{k} \right\rceil$, since $\left\lceil \|\x\|_1 \right\rceil \le 2\|\x\|_1$ for $\|\x\|_1\ge 1$.

Define $X_{\mathbf{u},\mathbf{v}} = \langle \A \mathbf{u}, \mathbf{v} \rangle$ on the set
\begin{equation*}
T_{\lambda} = \{(\mathbf{u},\mathbf{v}) : \mathbf{u}\in \R^k, \|\mathbf{u}\|_2= 1, \|\mathbf{u}\|_1\le \lambda, \mathbf{v}\in \R^m, \|\mathbf{v}\|_{8/7} \le 1\}.
\end{equation*}
Then, by H\"{o}lder's inequality, we have $\|\A\|_{2\rightarrow 8, \lambda} = \max_{(\mathbf{u},\mathbf{v})\in T_{\lambda}} X_{\mathbf{u}, \mathbf{v}}$.

Define $Y_{\mathbf{u}, \mathbf{v}} = \langle \mathbf{g}, \mathbf{u}\rangle + \langle \mathbf{h}, \mathbf{v} \rangle$, where $\mathbf{g}\in \R^k$ and $\mathbf{h}\in\R^m$ are independent standard normal random vectors.

We have, for any $(\mathbf{u},\mathbf{v}),(\mathbf{u'},\mathbf{v'})\in T_{\lambda}$, 
\begin{equation*}
\E[|X_{\mathbf{u},\mathbf{v}} - X_{\mathbf{u'},\mathbf{v'}}|^2] = \|\mathbf{v}\|_2^2 + \|\mathbf{v'}\|_2^2 - 2 \langle \mathbf{u}, \mathbf{u}'\rangle \langle \mathbf{v}, \mathbf{v}'\rangle,
\end{equation*}
and 
\begin{equation*}
\E[|Y_{\mathbf{u},\mathbf{v}} - Y_{\mathbf{u'},\mathbf{v'}}|^2] = 2 + \|\mathbf{v}\|_2^2 + \|\mathbf{v'}\|_2^2 - 2 \langle \mathbf{u}, \mathbf{u}'\rangle - 2\langle \mathbf{v}, \mathbf{v}'\rangle,
\end{equation*}
where we used $\|\mathbf{u}\|_2= \|\mathbf{u}'\|_2 = 1$. Therefore, we have
\begin{equation*}
\E[|Y_{\mathbf{u},\mathbf{v}} - Y_{\mathbf{u'},\mathbf{v'}}|^2] - \E[|X_{\mathbf{u},\mathbf{v}} - X_{\mathbf{u'},\mathbf{v'}}|^2] = 2(1 - \langle \mathbf{u}, \mathbf{u}'\rangle ) (1 - \langle \mathbf{v}, \mathbf{v}'\rangle ) \ge 0,
\end{equation*}
where we used $\|\mathbf{v}\|_2 \le \|\mathbf{v}\|_{8/7} \le 1$ and $\|\mathbf{v'}\|_2 \le \|\mathbf{v'}\|_{8/7} \le 1$. By Proposition 33 of \cite{V12}, this implies
\begin{equation*}
\E\left[\max_{(\mathbf{u},\mathbf{v})\in T_{\lambda}} X_{\mathbf{u}, \mathbf{v}}\right] \le \E\left[\max_{(\mathbf{u},\mathbf{v})\in T_{\lambda}} Y_{\mathbf{u}, \mathbf{v}}\right],
\end{equation*}
and hence
\begin{align*}
\E\left[\|\mathbf{A}\|_{2\rightarrow 8, \lambda}\right] &\le \E\left[\max_{(\mathbf{u},\mathbf{v})\in T_{\lambda}} Y_{\mathbf{u}, \mathbf{v}}\right] \\
&\le \E\left[\max_{(\mathbf{u},\mathbf{v})\in T_{\lambda}} \|\mathbf{g}\|_{\infty}\|\mathbf{u}\|_1 + \|\mathbf{h}\|_8\|\mathbf{v}\|_{8/7}\right] \\
&\le \sqrt{2\log (2k)}\lambda + (105m)^{\frac{1}{8}},
\end{align*}
where we used H\"{o}lder's inequality in the second line, and for the last inequality the bound
\begin{align*}
\exp(t\E[\|\mathbf{g}\|_{\infty}]) &\le \E[\exp(t\|\mathbf{g}\|_{\infty})]\le \sum_{i=1}^k \E[\exp(t|g_i|)]\le 2k\exp(t^2/2)
\end{align*}
with $t = \sqrt{2\log (2k)}$.

Finally, $\|\cdot \|_{2\rightarrow 8,\lambda}$ is a $1$-Lipschitz function: let $\mathbf{a}, \mathbf{b} \in \R^{m\times k}$ and, without loss of generality, $\|\mathbf{a}\|_{2\rightarrow 8, \lambda}\ge\|\mathbf{b}\|_{2\rightarrow 8, \lambda}$, then
\begin{align*}
\|\mathbf{a}\|_{2\rightarrow 8, \lambda}-\|\mathbf{b}\|_{2\rightarrow 8, \lambda} & = \max_{\|\x\|_2=1, \|\x\|_1\le \lambda} \|\mathbf{a}\x\|_8 - \max_{\|\mathbf{y}\|_2=1, \|\mathbf{y}\|_1\le \lambda} \|\mathbf{b}\mathbf{y}\|_8 \\
&\le \max_{\|\x\|_2=1, \|\x\|_1\le \lambda} \|\mathbf{a}\x\|_8 - \|\mathbf{b}\x\|_8 \\
&\le \max_{\|\x\|_2=1, \|\x\|_1\le \lambda} \|(\mathbf{a} - \mathbf{b})\x\|_8 \\
&\le \max_{\|\x\|_2=1, \|\x\|_1\le \lambda} \|(\mathbf{a} - \mathbf{b})\x\|_2 \\
&\le \|\mathbf{a}-\mathbf{b}\|_F,
\end{align*}
where we write $\|\cdot\|_F$ for the Frobenius norm and used the fact that it is an upper bound to the $\ell_2$-operator norm. Hence, an application of Theorem \ref{thm:ref1} yields
\begin{equation*}
\P\left[\|\A\|_{2\rightarrow 8, \lambda} < \sqrt{2\log (2k)}\lambda + (105m)^{\frac{1}{8}} + t \right] \ge 1 - \exp(-t^2/2).
\end{equation*}
Taking the union bound over $\lambda = 1,...,\left\lceil \sqrt{k}\right\rceil$ completes the proof that (\ref{eq:tech2_3}) holds with probability $1 - \left\lceil \sqrt{k}\right\rceil\exp(-t^2/2)$. The inequalities (\ref{eq:tech2_1}) and (\ref{eq:tech2_2}) can be show the same way.

Finally, convexity of $\mathcal{A}$ follows from an application of the Minkowski inequality. Let $\mathbf{a}, \mathbf{b} \in \R^{m\times k}$ satisfy (\ref{eq:tech1_3}), and $\alpha \in (0,1)$. Then, for any $\x\in\R$ with $\|\x\|_2=1$, 
\begin{align*}
\left(\sum_{j=1}^m \left(\alpha \mathbf{a}^T\x + (1-\alpha)\mathbf{b}^T\x\right)^8\right)^{\frac{1}{8}} &\le \alpha \left(\sum_{j=1}^m \left(\mathbf{a}^T\x\right)^8\right)^{\frac{1}{8}} + (1-\alpha) \left(\sum_{j=1}^m \left(\mathbf{b}^T\x\right)^8\right)^{\frac{1}{8}}\\
&\le (105m)^\frac{1}{8} + 2\sqrt{2\log (2k)} \|\x\|_1 + t.
\end{align*}
Since (\ref{eq:tech1_1}) and (\ref{eq:tech1_2}) can be shown the same way, this means that also $\alpha\mathbf{a} + (1-\alpha)\mathbf{b}\in \mathcal{A}$.
\end{proof}

\begin{proof}[Proof of Lemma \ref{lemma:tech3}]
The proof of Lemma \ref{lemma:tech3} relies on Theorem \ref{thm:ref2} and the following truncation trick, which allows us to consider bounded random variables. We begin by showing the inequality (\ref{eq:tech3_1}).

Let $i,l,s\in [n]$.
Writing $B_j = \left\{\max\{|A_{ji}|, |A_{jl}|, |A_{js}|\}\le \sqrt{64\log n}\right\}$ for the event that the terms are suitably bounded, we split $A_{ji}^2A_{jl}A_{js} = A_{ji}^2A_{jl}A_{js} \cdot \Eins(B_j) + A_{ji}^2A_{jl}A_{js} \cdot (1 - \Eins(B_j)) = Y_j + Z_j$, where by $\Eins(\cdot)$ we denote the indicator function. We will show that each of the two terms concentrates around its mean. 

Since $|Y_j|\le 64^2\log^2n$ is bounded, we can apply Theorem \ref{thm:ref2}. To this end, we compute 
\begin{equation*}
\sqrt{\sum_{j=1}^m\frac{1}{m^2}\E\left[A_{ji}^4A_{jl}^2A_{js}^2\Eins(B_j)\right]} \le \sqrt{\frac{105}{m}}.
\end{equation*}
Then, we have, as $m\ge c_1\log^5 n$ and for $C>0$ sufficiently large,
\begin{equation*}
\P\left[\left|\frac{1}{m}\sum_{j=1}^mY_j - \E\left[Y_j\right]\right| > \frac{C}{2}\sqrt{\frac{\log n}{m}}\right] \le 2\exp\left(- \frac{C^2\log n/(4m)}{2(\frac{105}{m} + 64^2C\frac{\log^{5/2}n}{6m^{3/2}})}\right) \le 2n^{-16}.
\end{equation*}
For the second term $Z_j$, we can use the Chebyshev inequality: we can compute
\begin{align*}
\operatorname{Var}\left(\frac{1}{m}\sum_{j=1}^mZ_j\right) &\le \frac{1}{m}\E\left[A_{1i}^4A_{1l}^2A_{1s}^2\cdot\Eins(B_j^c)\right] \nonumber\\
&\le \frac{1}{m}\sqrt{\E\left[A_{1i}^{8}A_{1l}^4A_{1s}^4\right] \cdot \P\left[\max\{|A_{1i}|, |A_{1l}|, |A_{1s}|\}>\sqrt{64\log n}\right]} \nonumber\\
&\le \frac{c'}{mn^{16}},
\end{align*}
for a constant $c'>0$, and hence, by the Chebyshev inequality,
\begin{equation*}
\P\left[\left|\frac{1}{m}\sum_{j=1}^mZ_j - \E\left[Z_j\right]\right| > \frac{C}{2}\sqrt{\frac{\log n}{m}}\right] \le \frac{\frac{c'}{mn^{16}}}{\frac{C^2\log n}{4m}} = \frac{4c'}{C^2\log n} n^{-16} \le n^{-16},
\end{equation*}
for $C \ge \sqrt{\frac{4c'}{\log n}}$.

This completes the proof that 
\begin{equation*}
\P\left[\left|\frac{1}{m}\sum_{j=1}^mA_{ji}^2A_{jl}A_{js} - \E\left[A_{1i}^2A_{1l}A_{1s}\right]\right| > C\sqrt{\frac{\log n}{m}}\right] \le 3n^{-16}.
\end{equation*}
Taking the union bound over all $i,l,s\in [n]$ shows that (\ref{eq:tech3_1}) holds with probability at least $1-3n^{-13}$.

The inequalities (\ref{eq:tech3_2})--(\ref{eq:tech3_4}) can be shown following the same steps. In (\ref{eq:tech3_4}) we need to control higher order terms, which is why get the additional logarithmic factors. As the proof of each bound follows exactly the same steps, we omit the details to avoid repetition.
\end{proof}

\begin{proof}[Proof of Lemma \ref{lemma:tech4}]
Using the fact that $\A_j^T\x^\star \sim \gauss(0, \|\x^\star\|_2^2)$ are independent random variables, this lemma can be shown following the same steps as in the proof of Lemma \ref{lemma:tech3}. The convexity of $\mathcal{A}$ follows, as in Lemma \ref{lemma:tech2}, from the Minkowski inequality. 
\end{proof}

\begin{proof}[Proof of Lemma \ref{lemma:tech5}]
We show this Lemma via induction over the size $|\mathcal{L}|$. For any $1\le L\le k$, define
\begin{equation*}
\mathcal{A}_L = \left\{\{a_{ji}\}\in \R^{m\times k}: \frac{1}{m}\sum_{j=1}^m\left(\sum_{l\in\mathcal{L}}a_{jl}\right)^4\le C|\mathcal{L}|^2 \text{ for all } \mathcal{L}\subseteq[k] \text{ with } |\mathcal{L}| = L \right\}.
\end{equation*}
We will show that $\P[\{A_{ji}\}\in\mathcal{A}_L]\ge 1 - c_3Ln^{-13}$ for some constant $c_3>0$, from which the result $\P[\{A_{ji}\}\in\mathcal{A}] \ge 1-c_2n^{-11}$ follows by taking the union bound over all possible sizes $L$.

As in Lemma \ref{lemma:tech2}, convexity of $\mathcal{A}_L$ follows from the Minkowski inequality. As the intersection of convex sets, $\mathcal{A}$ is also convex.

\underline{Base case}: $L=1$

This is true by the bound (\ref{eq:tech3_2}) of Lemma \ref{lemma:tech3}.

\underline{Induction step:} $L\rightarrow L+1$

Let $\mathcal{L}\subseteq [k]$ be any subset of coordinates with $|\mathcal{L}|=L$, and let $s\in [k]\backslash \mathcal{L}$. We have
\begin{align}\label{eq:indstep}
\frac{1}{m}\sum_{j=1}^m\left(\sum_{l\in \mathcal{L}} A_{jl} + A_{js}\right)^4 &= \frac{1}{m}\sum_{j=1}^mA_{js}^4 + \frac{4}{m}\sum_{j=1}^m\left(\sum_{l\in \mathcal{L}} A_{jl}\right) A_{js}^3 + \frac{6}{m}\sum_{j=1}^m\left(\sum_{l\in \mathcal{L}} A_{jl}\right)^2 A_{js}^2 
\nonumber\\
&\quad + \frac{4}{m}\sum_{j=1}^m\left(\sum_{l\in \mathcal{L}} A_{jl}\right)^3 A_{js} + \frac{1}{m}\sum_{j=1}^m\left(\sum_{l\in \mathcal{L}} A_{jl}\right)^4 
\end{align}
We need to show that, with probability $1-c_3(L+1)n^{-13}$, the sum (\ref{eq:indstep}) is bounded by $c(L+1)^2$ for all $\mathcal{L}\subseteq [k]$ with $|\mathcal{L}|=L$ and all $s\in [k]\backslash\mathcal{L}$.

By the induction hypothesis, we can bound the last term with probability $1-c_3Ln^{-13}$:
\begin{equation*}
\frac{1}{m}\sum_{j=1}^m\left(\sum_{l\in \mathcal{L}} A_{jl}\right)^4 \le C|\mathcal{L}|^2.
\end{equation*}
By Lemma \ref{lemma:tech3}, we can bound, with probability at least $1-c_4n^{-13}$ the first, second and third term in (\ref{eq:indstep}) as follows: the first term is bounded by 
\begin{equation*}
\frac{1}{m}\sum_{j=1}^mA_{js}^4 \le 4,
\end{equation*}
 the second term by (recall that $m\ge c_1k^2\log^2 n$)
\begin{equation*}
\frac{4}{m}\sum_{j=1}^m\left(\sum_{l\in \mathcal{L}} A_{jl}\right) A_{js}^3 = 4\sum_{l\in\mathcal{L}} \frac{1}{m} \sum_{j=1}^mA_{jl}A_{js}^3 \le c_5\frac{|\mathcal{L}|}{k} \le c_5,
\end{equation*}
and the third term in (\ref{eq:indstep}) can be bounded, using the Cauchy-Schwarz inequality,
\begin{align*}
\frac{6}{m}\sum_{j=1}^m\left(\sum_{l\in \mathcal{L}} A_{jl}\right)^2 \cdot A_{js}^2 &\le 6\sqrt{\frac{1}{m}\sum_{j=1}^m\left(\sum_{l\in \mathcal{L}} A_{jl}\right)^4}\cdot\sqrt{\frac{1}{m}\sum_{j=1}^mA_{js}^4} \\
&\le 6 \sqrt{C|\mathcal{L}|^2}\cdot \sqrt{4} \\
&= 12\sqrt{C} |\mathcal{L}|,
\end{align*}
where we used the induction hypothesis in the second line. 

For the fourth term in (\ref{eq:indstep}), we proceed as in the proof of Lemma \ref{lemma:tech6}: for a fixed $s\in[k]$, we condition on $\A_{\cdot s} = \mathbf{a}_{\cdot s}$ and show that the fourth term in (\ref{eq:indstep}) is suitably bounded with high probability for a fixed $\mathcal{L}$ via concentration of Lipschitz functions of Gaussian random variables. Next, we take a union bound over all possible choices for $\mathcal{L}$, integrate over the condition on $\A_{\cdot s}$ using the formula
\begin{equation*}
\P[B] = \int \P[B|\A_{\cdot s} = \mathbf{a}_{\cdot s}] \mu(\mathbf{a}_{\cdot s})d\mathbf{a}_{\cdot s},
\end{equation*}
where $B$ is any event and by $\mu$ we denote the standard normal density (since $\{A_{ji}\}_{j\in [m], i\in [k]}$ are i.i.d.), and finally take another union bound over all $s\in[k]$. 

\textbf{Step 1: Bound the fourth term in (\ref{eq:indstep}) conditioned on $\mathcal{L}$, $s$ and $\A_{\cdot s}$}\\
Fix an $s\in[k]$, a vector $\mathbf{a}_{\cdot s}$ satisfying $\max_j |a_{js}|\le 6\sqrt{\log n}$, a subset $\mathcal{L}\subseteq [k]\backslash \{s\}$, and consider the function
\begin{equation*}
h_{\mathcal{L}}(\mathbf{a}) = \frac{1}{m}\sum_{j=1}^ma_{js}\left(\sum_{l\in\mathcal{L}}a_{jl}\right)^3.
\end{equation*}
The idea is to apply concentration of Lipschitz functions of Gaussian random variables to show that the fourth term in (\ref{eq:indstep}), $h_{\mathcal{L}}(\A)$ is close to its expectation $\E[h_\mathcal{L}(\A)] = 0$. However, as $h_{\mathcal{L}}$ is not globally Lipschitz continuous, we cannot directly apply Theorem \ref{thm:ref1}. Similar to the proof of Theorem 3 of \cite{F18}, we will first restrict $h_{L}$ to a high probability event, where $h_{L}$ is Lipschitz continuous. Then, we extend this restricted function to a function $\tilde{h}_\mathcal{L}$ on the entire space in a way such that $\tilde{h}_\mathcal{L}$ is globally Lipschitz continuous, and apply Theorem \ref{thm:ref1} to this function $\tilde{h}_\mathcal{L}$. This also provides a high probability bound for $h_{\mathcal{L}}(\A)$, since, by construction, $\tilde{h}(\A)=h_\mathcal{L}(\A)$ with high probability. 

\textbf{Step 1a: Bound the Lipschitz constant of $h_{\mathcal{L}}$ restricted to $\mathcal{A}_L$}\\ 
Restricted to $\mathcal{A}_L$, we can bound the Lipschitz constant of $h_{\mathcal{L}}$ by the norm of its gradient by the mean-value theorem, since $\mathcal{A}_L$ is a convex set. For any $l\in\mathcal{L}$, we have
\begin{equation*}
\frac{\partial}{\partial a_{jl}}h_{\mathcal{L}}(\mathbf{a}) = \frac{3}{m}a_{js}\left(\sum_{l'\in\mathcal{L}}a_{jl'}\right)^2.
\end{equation*}
Hence, we can bound
\begin{align*}
\|\nabla h_{\mathcal{L}}(\mathbf{a})\|_2^2 &= \sum_{l\in \mathcal{L}}\sum_{j=1}^m \left(\frac{\partial}{\partial a_{jl}}h_{\mathcal{L}}(\mathbf{a})\right)^2 \le \frac{9|\mathcal{L}|}{m} \cdot \frac{1}{m}\sum_{j=1}^m\max_ja_{js}^2\left(\sum_{l\in\mathcal{L}}a_{jl}\right)^4 \le \frac{324|\mathcal{L}|\log n}{m} \cdot C |\mathcal{L}|^2
\end{align*}
on $\mathcal{A}_L$, where we used the assumption $\max_j a_{js}^2\le 36\log n$ and the induction hypothesis.

\textbf{Step 1b: Construct a globally Lipschitz continuous extension of $h_{\mathcal{L}}$}\\
Consider the Lipschitz extension of the function $h_{\mathcal{L}}$ (this is the one-dimensional case of the Kirszbraun theorem):
\begin{equation*}
\tilde{h}_{\mathcal{L}}(\mathbf{a}) = \inf_{\mathbf{a'}\in \mathcal{A}_L} \left((h_{\mathcal{L}}(\mathbf{a'}) + Lip(h_{\mathcal{L}})\cdot \|\mathbf{a} - \mathbf{a'} \|_2\right) 
\end{equation*}
where we write $Lip(h_{\mathcal{L}}) = \sqrt{\frac{324C|\mathcal{L}|^3\log n}{m}}$. By the definition, we have $\tilde{h}_{\mathcal{L}} = h_{\mathcal{L}}$ on $\mathcal{A}_L$, and it follows from an application of the triangle inequality that $\tilde{h}_{\mathcal{L}}$ is globally Lipschitz continuous with Lipschitz constant $Lip(h_{\mathcal{L}})$. 

We will show that $\tilde{h}_{\mathcal{L}}$ concentrates around its mean, which is different from the mean of $h_{\mathcal{L}}$.
Since $h_{\mathcal{L}}$ and $\tilde{h}_{\mathcal{L}}$ differ only on $\mathcal{A}_L^c$, which has probability less than $c_3Ln^{-13}$, we can bound, using the Cauchy-Schwarz inequality (note that all expectations are conditioned on $\A_{\cdot s} = \mathbf{a}_{\cdot s}$, which we omit for the sake of brevity and because the $\{A_{ji}\}$ are independent),
\begin{equation*}
\E[|h_{\mathcal{L}}(\A)|\Eins_{\mathcal{A}_L^c}(\A)] \le \sqrt{\E[h_{\mathcal{L}}(\A)^2]}\cdot \sqrt{\E[\Eins_{\mathcal{A}_L^c}(\A)]} \le \frac{c'}{n^5},
\end{equation*}
where we used the fact that $\E[h_{\mathcal{L}}(\A)^2]= 15|\mathcal{L}|^3/m$. Using $\tilde{h}_{\mathcal{L}}(\mathbf{a}) \le Lip(h_{\mathcal{L}})\cdot \|\mathbf{a}\|_2$, we have
\begin{equation*}
\E[|\tilde{h}_{\mathcal{L}}(\A)|\Eins_{\mathcal{A}_L^c}(\A)] \le Lip(h_{\mathcal{L}})\underbrace{\sqrt{\E[\|\A\|_2^2]}}_{=\sqrt{mk}}\cdot \sqrt{\E[\Eins_{\mathcal{A}_L^c}(\A)]} \le \frac{c''}{n^4}.
\end{equation*}
All in all, this shows that
\begin{equation*}
\left|\E[h_{1,\x}(\A)] - \E[\tilde{h}_{1,\x}(\A)]\right| \le \frac{c_6}{n^4}.
\end{equation*}
Hence, by Theorem \ref{thm:ref1}, we have
\begin{equation*}
\P\left[|\tilde{h}_{\mathcal{L}}(\A)| > C|\mathcal{L}| \;\big|\; \A_{\cdot s} = \mathbf{a}_{\cdot s} \right] \le 2\exp\left(- \frac{(C|\mathcal{L}| - c_6/n^4)^2}{2\frac{324C|\mathcal{L}|^3\log n}{m}}\right) \le 2 \exp\left(- c_7 \frac{m}{|\mathcal{L}|\log n}\right),
\end{equation*} 
for a constant $c_7 \le \frac{C}{648} - \frac{c_6}{324n^4}$.

\textbf{Step 2: Unravel the conditions: take union bounds and integrate over $\mathbf{a_{\cdot s}}$}\\
Let
\begin{equation*}
B_s = \left\{|\tilde{h}_{\mathcal{L}}(\A)| \le C|\mathcal{L}| \text{ for all } \mathcal{L}\subseteq [k]\backslash \{s\} \text{ with } |\mathcal{L}|=L\right\}.
\end{equation*}
Since we assume $m\ge c_1k^2\log n$, we can take the union bound over all possible subsets $\mathcal{L}\subseteq [k]\backslash \{s\}$ to obtain, using the upper bound ${k\choose L} \le (\frac{ek}{L})^L$,
\begin{align*}
\P\left[B_s \;\big|\; \A_{\cdot s} = \mathbf{a}_{\cdot s}\right] &\ge 1 - 2\exp\left(- c_7 \frac{m}{|\mathcal{L}|\log n} + |\mathcal{L}|\log \frac{ek}{|\mathcal{L}|}\right) \\
&\ge 1 - 2\exp\left(-\left(c_1c_7-1\right)k\log n\right).
\end{align*}
Next, we integrate over all $\mathbf{a}_{\cdot s}$ satisfying $\max_j |a_{js}|\le 6\sqrt{\log n}$:
\begin{align*}
\P\left[B_s \right] &\ge \int_{\{\max_j |a_{js}|\le 6\sqrt{\log n}\}}\P\left[B_s \;\big|\; \A_{\cdot s} = \mathbf{a}_{\cdot s}\right]\mu(\mathbf{a}_{\cdot s})d\mathbf{a}_{\cdot s} \\
&\ge \left(1 - 2\exp\left(-\left(c_1c_6-1\right)k\log n\right)\right)\cdot \P\left[\max_j |a_{js}| \le 6\sqrt{\log n}\right],
\end{align*}
where we wrote $\mu$ for the Gaussian density. This probability is $1-(m+2)n^{-18}$ by standard Gaussian tail bounds, if $(c_1c_7-1)k\ge 18$. Finally, taking the union bound over all $s\in [k]$ gives
\begin{equation*}
\P[B_s \text{ holds for all } s=1,...,k] \ge 1 - (m+2)n^{-17}.
\end{equation*}
Finally, we have $h_{\mathcal{L}}(\mathbf{a}) = \tilde{h}_{\mathcal{L}}(\mathbf{a})$ on $\mathcal{A}_L$, which has probability at least $1-c_3Ln^{-13}$ by the induction hypothesis. This shows that 
\begin{equation*}
\left| \frac{1}{m}\sum_{j=1}^ma_{js} \left(\sum_{l\in\mathcal{L}} A_{jl}\right)^3 \right| < C|\mathcal{L}| \quad \text{for all } s\in[k], \mathcal{L}\subseteq [k]\backslash \{s\} \text{ with } |\mathcal{L}|=L
\end{equation*}
holds with probability at least $1 - c_3Ln^{-13} - (m+2)n^{-17}$. Putting everything together, this completes the induction step (\ref{eq:indstep}),
\begin{equation*}
\frac{1}{m}\sum_{j=1}^m\left(\sum_{l\in \mathcal{L}} A_{jl} + A_{js}\right)^4 \le C(|\mathcal{L}|+1)^2 \quad \text{for all } s\in[k], \mathcal{L}\subseteq [k]\backslash \{s\} \text{ with } |\mathcal{L}|=L
\end{equation*}
holds with probability at least $1 - c_3(L+1)n^{-13}$ if $m\le n^4$ The case $m>n^4$ is simpler and can be shown following the same steps, writing the probabilities in terms of $m$ instead of $n$.
\end{proof}

\begin{proof}[Proof of Lemma \ref{lemma:tech6}] In order to show $\P[\{A_{ji}\}\in \mathcal{A}] \ge 1-c_2n^{-10}$, we need to show that both (\ref{eq:tech6_1}) and (\ref{eq:tech6_2}) are satisfied with high probability.

\textbf{Proof that (\ref{eq:tech6_1}) is satisfied with high probability}\\
The idea of the proof that (\ref{eq:tech6_1}) holds with high probability is as follows: for any index $i\in [n]$, we condition on $\A_{\cdot i} = \mathbf{a}_{\cdot i}$. To this end, we use the formula
\begin{equation*}
\P[B] = \int \P[B|\A_{\cdot i}=\mathbf{a}_{\cdot i}]\mu(\mathbf{a}_{\cdot i})d\mathbf{a}_{\cdot i},
\end{equation*}
which holds for any event $B$, where we write $\mu$ for the standard normal density.

To define the event $B_i$ (making explicit the fact that we fixed an index $i$), fix any subset $\mathcal{K}\subseteq \S\backslash \{i\}$, any subset $\mathcal{L}\subseteq \mathcal{K}$ with $|\mathcal{L}| \ge \frac{1}{2}|\mathcal{K}|$, and any point $\x\in N^{\mathcal{K}}_{\epsilon'}$, and consider the function
\begin{equation*}
h_{\mathcal{K},\mathcal{L},\x}(\mathbf{a}) = \sum_{l\in \mathcal{L}} \frac{1}{m}\sum_{j=1}^m a_{ji}a_{jl}(\mathbf{a}_{j,-i}^T\x)^2.
\end{equation*} 
Then, we define $B_i$ as the following event:
\begin{align*}\label{eq:event}
B_i = \Big\{|h_{\mathcal{K},\mathcal{L},\x}(\mathbf{A})|\le C|\mathcal{L}|\frac{\log n}{\sqrt{m}} \quad &\text{for all subsets } \mathcal{K}\subseteq \S\backslash \{i\}, \mathcal{L}\subseteq \mathcal{K} \text{ with } \nonumber\\
&|\mathcal{L}|\ge \frac{1}{2}|\mathcal{K}| \text{ and } \x\in N^{\mathcal{K}}_{\epsilon'} \Big\}.
\end{align*}
To show that the conditional probability $\P[B_i|\A_{\cdot i} = \mathbf{a}_{\cdot i}]$ is large, we will show that, for any fixed $\mathcal{K}$, $\mathcal{L}$ and $\x$ as described above, $h_{\mathcal{K},\mathcal{L},\x, \mathbf{a}_{\cdot i}}(\mathbf{A})$ concentrates around its mean $\E[h_{\mathcal{K},\mathcal{L}, \x, \mathbf{a}_{\cdot i}}(\A)] = 0$.
Next, we can bound $\P[B_i|\mathbf{A}_{\cdot i} = \mathbf{a}_{\cdot i}]$ by taking union bounds over all $\x\in N^{\mathcal{K}}_{\epsilon'}$, $\mathcal{L}\subseteq \mathcal{K}$ with $|\mathcal{L}|\ge \frac{1}{2}|\mathcal{K}|$ and $\mathcal{K}\subseteq \S\backslash \{i\}$. Finally, we integrate over $\mathbf{a}_{\cdot i}$ to obtain $\P[B_i]$, and taking the union bound over all $i\in [n]$ completes the proof of (\ref{eq:tech6_1}).

\textbf{Step 1: Bound $h_{\mathcal{K},\mathcal{L},\x,\mathbf{a}_{\cdot i}}(\A)$ conditioned on $\mathcal{K},\mathcal{L},\x$ and $\A_{\cdot i}$}\\
 First, let $\mathcal{A}_1$ be as in Lemma \ref{lemma:tech5}, and $\mathcal{A}_2$ as in Lemma \ref{lemma:tech1} (with $t = 5\sqrt{\log n}$) (to be precise, $\mathcal{A}_1,\mathcal{A}_2\subset \R^{m\times n}$, and we require the projections onto $\R^{m\times \S}$ to be as in the respective Lemmas). Then, by these two Lemmas, we have $\P[\mathcal{A}_1\cap\mathcal{A}_2] \ge 1 - c_3n^{-11}$, for a constant $c_3>0$, and as the intersection of two convex sets, $\mathcal{A}_1\cap\mathcal{A}_2$ is also convex.

Fix an $i\in [n]$, a subset $\mathcal{K}\subseteq \S\backslash\{i\}$, a subset $\mathcal{L}\subseteq\mathcal{K}$ with $|\mathcal{L}|\ge \frac{1}{2}|\mathcal{K}|$ and a vector $\x\in N^{\mathcal{K}}_{\epsilon'}$.
We begin by conditioning on $\A_{\cdot i} = \mathbf{a}_{\cdot i}$, where $\max_j|a_{ji}|\le 5\sqrt{\log n}$. To simplify notation, we will omit the subscripts and write $h(\A)$.

As in the proof of Lemma \ref{lemma:tech5}, the idea is use Theorem \ref{thm:ref1} to show that $h(\A)$ is close to its expectation. However, $h(\A)$ is not globally Lipschitz continuous. We will show that, restricted to $\mathcal{A}_1\cap \mathcal{A}_2$, the function $h$ is Lipschitz continuous. Then, we extend this restricted function to a globally Lipschitz continuous function $\tilde{h}$ on the entire space, and apply Theorem \ref{thm:ref1} to this function $\tilde{h}$. By construction, we have $h(\A) = \tilde{h}(\A)$ with high probability, which yields a high probability bound for $h(\A)$.

\textbf{Step 1a: Bound the Lipschitz constant of $h$ restricted to $\mathcal{A}_1\cap\mathcal{A}_2$}\\
Restricted to $\mathcal{A}_1\cap\mathcal{A}_2$, we can bound the Lipschitz constant of $h$ by the norm of its gradient by the mean-value theorem, since $\mathcal{A}_1\cap\mathcal{A}_2$ is a convex set. We have
\begin{equation*}
\frac{\partial}{\partial a_{js}} h(\mathbf{a})= \begin{cases} 
\frac{2}{m} a_{ji}x_s\left(\mathbf{a}_{j,-i}^T\x\right) \sum_{l\in\mathcal{L}}a_{jl}& s\notin \mathcal{L}\\
\frac{2}{m} a_{ji}x_s\left(\mathbf{a}_{j,-i}^T\x\right) \sum_{l\in\mathcal{L}}a_{jl} +  \frac{1}{m}a_{ji}\left(\mathbf{a}_{j,-i}^T\x\right)^2 \quad & s\in \mathcal{L} \end{cases}
\end{equation*}
Using the inequality $(a+b)^2 \le 2(a^2+b^2)$, we have
\begin{align}\label{eq:lipconst}
\|\nabla h(\mathbf{a})\|_2^2 = \sum_{s\neq i}\sum_{j=1}^m\left(\frac{\partial}{\partial a_{js}} h(\mathbf{a})\right)^2 \le & 8\sum_{s\neq i}\sum_{j=1}^m \frac{1}{m^2} x_s^2a_{ji}^2\left(\sum_{l\in \mathcal{L}}a_{jl}\right)^2\left(\mathbf{a}_{j,-i}^T\x\right)^2 \nonumber\\
& + 2\sum_{s\in \mathcal{L}} \sum_{j=1}^m\frac{1}{m^2}a_{ji}^2\left(\mathbf{a}_{j,-i}^T\x\right)^4.
\end{align}
We bound the two sums separately. For the first sum, we have
\begin{align*}
\sum_{s\neq i}\sum_{j=1}^m \frac{1}{m^2}x_s^2a_{ji}^2\bigg(\sum_{l\in \mathcal{L}}a_{jl}\bigg)^2\left(\mathbf{a}_{j,-i}^T\x\right)^2 &= \frac{1}{m} \sum_{s\neq i}x_s^2\sum_{j=1}^m \frac{1}{m}a_{ji}^2\bigg(\sum_{l\in \mathcal{L}}a_{jl}\bigg)^2\left(\mathbf{a}_{j,-i}^T\x\right)^2 \nonumber\\
&\le \Big(\max_j a_{ji}^2\Big)\cdot\frac{1}{m} \sum_{s\neq i}x_s^2\sum_{j=1}^m \frac{1}{m}\bigg(\sum_{l\in \mathcal{L}}a_{jl}\bigg)^2\left(\mathbf{a}_{j,-i}^T\x\right)^2 \\
&\le \frac{25\log n}{m} \sum_{s\neq i}x_s^2 \sqrt{\frac{1}{m}\sum_{j=1}^m \bigg(\sum_{l\in \mathcal{L}}a_{jl}\bigg)^4\cdot\frac{1}{m}\sum_{j=1}^m \left(\mathbf{a}_{j,-i}^T\x\right)^4}
\end{align*}
where we used H\"{o}lder's inequality in the second and the Cauchy-Schwarz inequality in the last line. 
\begin{remark}
The application of H\"{o}lder's inequality in the second line is the reason for the extra $\log n$ term in the sample complexity. We expect that also $\frac{1}{m}\sum_{j=1}^ma_{ji}^4(\mathbf{a}_{j,-i}^T\x)^4 = \O(1)$, however Lemma \ref{lemma:tech1} does not apply because of the extra factor $a_{ji}^4$, and the truncation method of Lemma \ref{lemma:tech3} does not yield a probability small enough to allow for taking the union bound over all $\x\in N^{\mathcal{K}}_{\epsilon'}$.
\end{remark}
Since $\mathbf{a}\in \mathcal{A}_1$, we have, by Lemma \ref{lemma:tech5},
\begin{equation*}
\sqrt{\frac{1}{m}\sum_{j=1}^m \bigg(\sum_{l\in \mathcal{L}}a_{jl}\bigg)^4} \le c_4|\mathcal{L}|.
\end{equation*}
Since also $\mathbf{a}\in \mathcal{A}_2$, we have, by Lemma \ref{lemma:tech1},  
\begin{equation*}
\sqrt{\frac{1}{m}\sum_{j=1}^m(\mathbf{a}_{j,-i}^T\x)^4} \le \sqrt{\frac{1}{m}\|\mathbf{a}\|_{2\rightarrow 4}^4} \le 11.
\end{equation*}
Putting this together, we can bound the first sum in (\ref{eq:lipconst}),
\begin{equation*}
\sum_{s\neq i}\sum_{j=1}^m \frac{1}{m^2} x_s^2a_{ji}^2\bigg(\sum_{l\in \mathcal{L}}a_{jl}\bigg)^2\left(\mathbf{a}_{j,-i}^T\x\right)^2 \le \frac{5\log n}{m}\sum_{s\neq i}x_s^2 \cdot c_4|\mathcal{L}| \cdot 11 \le 55c_4\frac{|\mathcal{L}|\log n}{m},
\end{equation*}
where we used $\|\x\|_2= 1$.

For the second sum in (\ref{eq:lipconst}), we can write
\begin{equation*}
\sum_{s\in \mathcal{L}} \sum_{j=1}^m\frac{1}{m^2}a_{ji}^2\left(\mathbf{a}_{j,-i}^T\x\right)^4 \le \frac{5\log n}{m} |\mathcal{L}| \cdot \frac{1}{m}\sum_{j=1}^m(\mathbf{a}_{j,-i}^T\x)^4 \le 605\frac{|\mathcal{L}|\log n}{m},
\end{equation*}
where we used $\frac{1}{m}\sum_{j=1}^m(\mathbf{a}_{j,-i}^T\x)^4\le 121$, again by Lemma \ref{lemma:tech1}.

All in all, we have shown that 
\begin{equation*}
\|\nabla h(\mathbf{a})\|_2 = \left(\sum_{s\neq i}\sum_{j=1}^m\left(\frac{\partial}{\partial a_{js}} h(\mathbf{a})\right)^2\right)^{\frac{1}{2}} \le \sqrt{c_5\frac{|\mathcal{L}|\log n}{m}}
\end{equation*}
for $\mathbf{a}\in \mathcal{A}_1\cap \mathcal{A}_2$ and $c_5 = 440c_4 + 1210$.

\textbf{Step 1b: Construct a globally Lipschitz continuous extension of $h$}\\
Consider the Lipschitz extension of the function $h$ (this is the one-dimensional case of the Kirszbraun theorem):
\begin{equation*}
\tilde{h}(\mathbf{a}) = \inf_{\mathbf{a}'\in \mathcal{A}_1\cap \mathcal{A}_2} \left(h(\mathbf{a'}) + Lip(h) \cdot \|\mathbf{a}-\mathbf{a'}\|_2\right),
\end{equation*}
where we write $Lip(h) = \sqrt{c_5\frac{|\mathcal{L}|\log n}{m}}$. On $\mathcal{A}_1\cap\mathcal{A}_2$ we have $\tilde{h}=h$, and an application of the triangle inequality shows that $\tilde{h}$ is globally Lipschitz continuous with Lipschitz constant $Lip(h)$. 

We will show that $\tilde{h}$ concentrates around its mean, which is different from the mean of $h$.
Since $h$ and $\tilde{h}$ differ only on $(\mathcal{A}_1\cap\mathcal{A}_2)^c$, which has probability at most $c_3n^{-11}$, we can bound, using the Cauchy-Schwarz inequality (note that all expectations are conditioned on $\A_{\cdot i} = \mathbf{a}_{\cdot i}$, which we omit for the sake of brevity and because the $\{A_{jl}\}$ are independent),
\begin{equation*}
\E[|h(\A)|\Eins_{(\mathcal{A}_1\cap\mathcal{A}_2)^c}(\A)] \le \sqrt{\E[h(\A)^2]}\cdot \sqrt{\E[\Eins_{(\mathcal{A}_1\cap\mathcal{A}_2)^c}(\A)]} \le \frac{c'}{n^5},
\end{equation*}
where we used the bound $\E[h(\A)^2] \le \sqrt{315}k^2/m$. Using $\tilde{h}(\mathbf{a}) \le Lip(h)\cdot \|\mathbf{a}\|_2$, we have
\begin{equation*}
\E[|\tilde{h}(\A)|\Eins_{(\mathcal{A}_1\cap\mathcal{A}_2)^c}(\A)] \le Lip(h)\underbrace{\sqrt{\E[\|\A\|_2^2]}}_{=\sqrt{mk}}\cdot \sqrt{\E[\Eins_{(\mathcal{A}_1\cap\mathcal{A}_2)^c}(\A)]} \le \frac{c''}{n^3}.
\end{equation*}
All in all, this shows that
\begin{equation*}
\left|\E[h_{1,\x}(\A)] - \E[\tilde{h}_{1,\x}(\A)]\right| \le \frac{c_6}{n^3}.
\end{equation*}
Hence, by Theorem \ref{thm:ref1}, we have 
\begin{equation}
\P\left[|\tilde{h}(\A)| > C\frac{|\mathcal{L}|\log n}{\sqrt{m}} \;\big|\;\A_{\cdot i} = \mathbf{a}_{\cdot i} \right] \le 2\exp\left(- \frac{(C\frac{|\mathcal{L}|}{\sqrt{m}}\log n - c_6/n^3)^2}{2c_5\frac{|\mathcal{L}|\log n}{m}}\right) \le 2 \exp\left(- c_7 |\mathcal{L}|\log n\right),
\end{equation} 
for a constant $c_7 \le \frac{C^2}{2c_5}- \frac{cc_6}{c_5}\frac{\sqrt{m}}{n^3}$.

\textbf{Step 2: Unravel the conditions: take union bounds and integrate over $\mathbf{a}_{\cdot i}$}\\
Let
\begin{align*}
\widetilde{B}_{i, S} = \Big\{|\tilde{h}_{\mathcal{K},\mathcal{L},\x}(\A)|\le C\frac{|\mathcal{L}|\log n}{\sqrt{m}} \quad &\text{for all subsets } \mathcal{K}\subseteq \S\backslash \{i\} \text{ with } |\mathcal{K}| = S, \\
&\mathcal{L}\subseteq \mathcal{K} \text{ with } |\mathcal{L}|\ge \frac{1}{2}|\mathcal{K}| \text{ and } \x\in N^{\mathcal{K}}_{\epsilon'} \Big\},
\end{align*}
and let $\widetilde{B}_i$ be the same event without the restriction $|\mathcal{K}| = S$.
Taking union bounds over all $\x\in N^{\mathcal{K}}_{\epsilon'}$, $\mathcal{L}\subseteq \mathcal{K}$ with $|\mathcal{L}|\ge \frac{1}{2}|\mathcal{K}|$ and $\mathcal{K}\subseteq \S\backslash \{i\}$ with fixed cardinality $|\mathcal{K}|=S$, we obtain (using the upper bound ${k\choose S} \le (\frac{ek}{S})^S$)
\begin{align*}
\P\left[\widetilde{B}_{i,S} \;\big||\; \A_{\cdot i} = \mathbf{a}_{\cdot i}\right] &\ge 1 - 2\exp\left(- \frac{c_7}{2}S\log n + S\log \frac{3k}{c_{\epsilon}} + S\log 2 + S\log\frac{ek}{S}\right) \\
&\ge 1 - 2\exp\left(-c_8S\log n\right)
\end{align*}
for a constant $c_8>0$ if $c_7 \ge 4 + 2\frac{\log (6/c_{\epsilon})}{\log n}$.  
Taking the union bound over all possible choices for the cardinalities $S=1,...,k$ gives
\begin{align*}
\P[\widetilde{B}_i \;|\; \A_{\cdot i} = \mathbf{a}_{\cdot i}] &\ge 1 - \frac{2}{1 - 2\exp\left(-c_8 \log n\right)}\exp\left(-c_8\log n\right) \ge 1 - 2.1 \exp\left(-c_8\log n\right),
\end{align*}
if $n^{-c_8} \le 0.1/4.2$.
Next, we integrate over all $\mathbf{a}_{\cdot i}$ satisfying $\max_j |a_{ji}|\le 5\sqrt{\log n}$:
\begin{align*}
\P[\tilde{B}_i] &\ge \int_{\{\max_j |a_{ji}|\le 5\sqrt{\log n}\}}\P\left[B_i \;\big|\; \A_{\cdot i} = \mathbf{a}_{\cdot i} \right]\mu(\mathbf{a}_{\cdot i})d\mathbf{a}_{\cdot i} \nonumber\\
&\ge \left(1 - 2.1e^{-c_8\log n}\right)\cdot \P\left[\max_j |a_{ji}| \le 5\sqrt{\log n}\right],
\end{align*}
where we wrote $\mu$ for the Gaussian density. This probability is at least $1-4n^{-11}$ by standard Gaussian tail bounds, if $c_8 \ge 11$.

On $\mathcal{A}_1\cap\mathcal{A}_2\cap \{\max_j |a_{ji}|\le 5\sqrt{\log n}\}$, which is independent of $\mathcal{K}$, $\mathcal{L}$ and $\x$, we have $h_{\mathcal{K},\mathcal{L},\x} = \tilde{h}_{\mathcal{K},\mathcal{L},\x}$ for all subsets $\mathcal{K}\subseteq \S\backslash \{i\}$, $\mathcal{L}\subseteq \mathcal{K}$ with $|\mathcal{L}|\ge \frac{1}{2}|\mathcal{K}|$ and $\x\in N^{\mathcal{K}}_{\epsilon'}$. In this case, the bound on $\tilde{h}_{\mathcal{K},\mathcal{L},\x}(\A)$ also applies to $h_{\mathcal{K},\mathcal{L},\x}(\A)$, that is
\begin{equation*}
\P[B_i] \ge \P\left[\widetilde{B}_i \cap \mathcal{A}_1\cap \mathcal{A}_2\cap \Big\{\max_j |a_{ji}|\le 5\sqrt{\log n}\Big\}\right] \ge 1 - 4n^{-11} - c_3n^{-11} - n^{-11}\ge 1- c_2n^{-11},
\end{equation*}
for $c_2 = c_3 + 5$.
Finally, taking the union bound over all $i\in [n]$ establishes that (\ref{eq:tech6_1}) holds with probability at least $1-c_2n^{-10}$.

\textbf{Proof that (\ref{eq:tech6_2}) is satisfied with high probability}\\
The inequality (\ref{eq:tech6_2}) can be shown following the exact same steps as above for (\ref{eq:tech6_1}). It is slightly simpler because of the lower order of the expression we need to control, which is also the reason for the term $\sqrt{\log n}$ in the bound (\ref{eq:tech6_2}) instead of $\log n$. We omit the details to avoid repetition.
\end{proof}

\begin{proof}[Proof of Lemma \ref{lemma:tech7}]
We first show the bound (\ref{eq:tech7_1}). To show that (\ref{eq:tech7_1}) holds with high probability, we will show that it is satisfied for all $\{a_{ji}\}\in\mathcal{A}$, where $\mathcal{A}$ is defined as the intersection of the set defined in Lemma \ref{lemma:tech6} and the set where Lemmas \ref{lemma:tech1} (with $t = 5\sqrt{\log n}$; as $\mathcal{A}\subset \R^{m\times n}$, we assume the projection of $\mathcal{A}$ onto $\R^{m\times \S}$ to satisfy Lemma \ref{lemma:tech1}) and \ref{lemma:tech3} hold. By the aforementioned lemmas, we have $\P[\{A_{ji}\}\in\mathcal{A}]\ge 1-c_2n^{-10}$.

\textbf{Proof that (\ref{eq:tech7_1}) is satisfied with high probability}\\
Let $\{a_{ji}\}\in \mathcal{A}$. For any $i\in [n]$, write 
\begin{equation*}
g(\x) = \frac{1}{m}\sum_{j=1}^ma_{ji}\left(\mathbf{a}_{j,-i}^T\x_{-i}\right)^3.
\end{equation*}
We consider the constrained optimization problem
\begin{equation}
\label{eq:optconstr}
\begin{aligned}
&\max_{\x} g(\x) \\
\text{s.t.} & \quad \|\x\|_2^2 \le 1 \\
& \quad \|\x\|_1 \le b 
\end{aligned}
\end{equation}
for some $0\le b \le \sqrt{k}$. Our goal is to show that, if $\x'$ is an optimizer of this problem, then we have $g(\x')\le Cb \sqrt{\frac{\log n}{m}}$. By symmetry, we can obtain the lower bound $g(\x')\ge -Cb \sqrt{\frac{\log n}{m}}$ by considering the corresponding minimization problem exactly the same way. This would then complete the proof that, for all $\{a_{ji}\}\in \mathcal{A}$, inequality (\ref{eq:tech7_1}) is satisfied.

The idea of the proof is as follows: since we are maximizing a continuous function over a compact set, a global maximum is attained at a point $\x'\in\mathcal{X}$. Using the KKT conditions at $\x'$, we will bound the Lagrange multiplier $\mu^\star$ corresponding to the constraint $\|\x\|_1\le b$. This controls how the maximum attainable value $g(\x')$ increases if we relax the constraint $\|\x\|_1\le b$ (for details see e.g.\ \cite{BV04}), and integrating over $b$ gives the desired result.

\textbf{Step 1: Establish KKT conditions}\\
In order to establish that the KKT conditions are satisfied at $\x'$, we verify that the linear independence constraint qualification (LICP) (see e.g.\ \cite{P73}) holds in this problem: the gradients of all active inequality constraints are linearly independent at any point $\x$. We restrict our attention to $b\neq \sqrt{s}$, $s=1,...,k$, which is a set of measure zero and can be ignored when integrating. If only one inequality constraint is binding, there is nothing to show. Otherwise, the gradients $\frac{\partial}{\partial \x}\|\x\|_2^2 = 2\x$ and $\frac{\partial}{\partial \x}\|\x\|_1 = \operatorname{sign}(\x)$ can only be linearly dependent if $x_i\in \{-x,0,x\}$ for all $i$. But then $\|\x\|_2^2 = sx^2$ and $\|\x\|_1 = sx$ (where $s$ is the number of non-zero coordinates of $\x$), hence the two constraints cannot be simultaneously binding, because we have $b\neq \sqrt{s}$.

Let $\x'\in\mathcal{X}$ be a maximizer of (\ref{eq:optconstr}). From LICP it follows that the following KKT conditions are necessarily satisfied: writing $\mathcal{L}(\x,\lambda,\mu) = g(\x) - \lambda (\|\x\|_2^2-1) - \mu (\|\x\|_1-b)$, there exist Lagrange multipliers $\lambda^\star, \mu^\star$ satisfying 
\begin{align}
\nabla_{\x}\mathcal{L}(\x',\lambda^\star,\mu^\star) &= \mathbf{0} \quad &&\text{stationarity} \label{eq:kkt1}\\
\|\x'\|_2^2 \le 1, \quad \|\x'\|_1&\le b &&\text{primal feasibility}\label{eq:kkt2}\\
\lambda^\star\ge 0, \quad \mu^\star&\ge 0 &&\text{dual feasibility}\label{eq:kkt3}\\
\lambda^\star(\|\x'\|_2^2-1) = 0, \quad \mu^\star(\|\x'\|_1-b) &= 0 &&\text{complementary slackness}\label{eq:kkt4}
\end{align}

\textbf{Step 2: Bounding the Lagrange multiplier $\mu^\star$}\\
Using the stationarity condition (\ref{eq:kkt1}), we get
\begin{equation*}
\begin{gathered}
\frac{\partial}{\partial x_l} \mathcal{L}(\x',\lambda^\star,\mu^\star) = \frac{3}{m}\sum_{j=1}^ma_{ji}a_{jl}(\mathbf{a}_{j,-i}^T\x'_{-i})^2 - 2\lambda^\star x'_l - \mu^\star\operatorname{sign}(x'_l) = 0 \\
\Rightarrow h_l(\x') := \frac{1}{m}\sum_{j=1}^ma_{ji}a_{jl}(\mathbf{a}_{j,-i}^T\x'_{-i})^2 = \frac{2}{3}\lambda^\star x'_l + \frac{1}{3}\mu^\star\operatorname{sign}(x'_l)
\end{gathered}
\end{equation*}
From this, we see that $h_l(\x')$ has the same sign as $x'_{l}$ because of the dual feasibility condition (\ref{eq:kkt3}), and that $x'_{l}=0$ if and only if $h_l(\x') = 0$. Rearranging for $\mu^\star$, we have, for any $l$ with $x'_{l}\neq 0$,
\begin{equation}\label{eq:shadowprice}
\mu^\star = 3|h_l(\x')| - 2\lambda^\star|x'_l| \le 3|h_l(\x')|,
\end{equation}
where we again used that $\lambda^\star\ge 0$. Next, we show that
\begin{equation}\label{eq:shadowbound}
\min_{l:x'_l\neq 0} |h_l(\x')| \le \frac{C\log n}{3\sqrt{m}}.
\end{equation}
To this end, let $\mathcal{K} = \{s\in \S: x'_s\neq 0\}$, and let $\mathcal{L} = \{l\in \mathcal{K}: x'_l>0\}$ (note that we can assume $x'_i=0$, as $x_i$ does not contribute to $g(\x)$). We can assume without loss of generality that $|\mathcal{L}|\ge \frac{1}{2}|\mathcal{K}|$, as we can otherwise consider the set defined by $x'_l<0$. Further, assume for notational simplicity that $\|\x'\|_2 = 1$; otherwise, we can replace $\x'$ by $\x' / \|\x'\|_2$, which makes every $|h_l(\x')|$ larger because $\|\x'\|_2\le 1$. 

Let $N_{\epsilon'}^{\mathcal{K}}$ be an $\epsilon'$-net of the unit sphere in $\R^{\mathcal{K}}$, where $\epsilon' = C\log n/(465\sqrt{m})$. Let $\x\in N_{\epsilon'}^{\mathcal{K}}$ with $\|\x-\x'\|_2 \le \epsilon'$. Then, recalling $m\ge c_1\max\{k^2\log^2n, \; \log^5n\}$, we have by (\ref{eq:tech6_1}) of Lemma \ref{lemma:tech6},
\begin{equation*}
\sum_{l\in \mathcal{L}} h_l(\x) \le c|\mathcal{L}|\frac{\log n}{\sqrt{m}} \le \frac{C|\mathcal{L}|\log n}{6\sqrt{m}},
\end{equation*}
if $C\ge 6c$, where $c$ is the universal constant from Lemma \ref{lemma:tech6}.
By the pigeonhole principle, there must be an index $l\in \mathcal{L}$ with $h_l(\x) \le C\log n/(6\sqrt{m})$. By the Cauchy-Schwarz inequality, we have
\begingroup
\allowdisplaybreaks
\begin{align*}
|h_l(\x)-h_l(\x')| &= \left|\frac{1}{m}\sum_{j=1}^m a_{ji}a_{jl}\left(\mathbf{a}_{j,-i}^T(\x+\x')\right)\left(\mathbf{a}_{j,-i}^T(\x-\x')\right)\right| \nonumber\\
&\le \left(\frac{1}{m}\sum_{j=1}^ma_{ji}^2a_{jl}^2\left(\mathbf{a}_{j,-i}^T(\x+\x')\right)^2 \cdot \frac{1}{m}\sum_{j=1}^m\left(\mathbf{a}_{j,-i}^T(\x-\x')\right)^2\right)^\frac{1}{2} \nonumber\\
&\le \left(\frac{1}{m}\sum_{j=1}^ma_{ji}^8\right)^{\frac{1}{8}}\left(\frac{1}{m}\sum_{j=1}^ma_{jl}^8\right)^{\frac{1}{8}}\left(\frac{1}{m}\sum_{j=1}^m\left(\mathbf{a}_{j,-i}^T(\x+\x')\right)^4\right)^{\frac{1}{4}} \\
&\quad \cdot \left(\frac{1}{m}\sum_{j=1}^m\left(\mathbf{a}_{j,-i}^T(\x-\x')\right)^2\right)^\frac{1}{2}\nonumber\\
&\le 106^\frac{1}{4} \cdot \frac{(3m)^{\frac{1}{4}} + \sqrt{k} + 5\sqrt{\log n}}{m^\frac{1}{4}}\|\x+\x'\|_2 \cdot \frac{\sqrt{m} + \sqrt{k} + 5\sqrt{\log n}}{\sqrt{m}}\|\x-\x'\|_2\\
&\le \frac{C\log n}{6\sqrt{m}},
\end{align*}
\endgroup
where we used Lemmas \ref{lemma:tech1} and \ref{lemma:tech3}, and $\|\x-\x'\|\le C\log n/(465\sqrt{m})$. Together with the non-negativity of $h_l(\x')$, this completes the proof of (\ref{eq:shadowbound}).

Now, (\ref{eq:shadowprice}) together with (\ref{eq:shadowbound}) implies
\begin{equation*}
\mu^\star \le C\frac{\log n}{\sqrt{m}}.
\end{equation*}

\textbf{Step 3: Showing (\ref{eq:tech7_1})}\\
In order to show (\ref{eq:tech7_1}), define the value function
\begin{equation*}
v(b) = \max_{\x} \{g(\x): \|\x\|_2^2\le 1, \|\x\|_1\le b\}.
\end{equation*}
We use the fact that $\mu^\star$ is the shadow price of the constraint $\|\x\|_1\le b$ (for more details see e.g.\ \cite{BV04}):
\begin{equation*}
\frac{\partial}{\partial b}v(b) = \mu^\star(b).
\end{equation*}
By definition, we have $v(0) = 0$. For any $0\le b_0\le \sqrt{k}$, we have, by the fundamental theorem of calculus,
\begin{align*}
v(b_0) = \int_0^{b_0} \frac{\partial}{\partial b}v(b)db
= \int_0^{b_0} \mu^\star(b)db \le \int_0^{b_0}C\frac{\log n}{\sqrt{m}}db 
< Cb_0 \frac{\log n}{\sqrt{m}},
\end{align*}
By definition, we have $g(\x)\le v(b_0)$ for all $\x\in\mathcal{X}$ with $\|\x\|_1\le b_0$.
Hence, 
\begin{equation*}
g(\x) \le Cb \frac{\log n}{\sqrt{m}}
\end{equation*}
must hold for any $\x$ with $\|\x\|_2^2\le 1$ and $\|\x\|_1\le b$,
which completes the proof of (\ref{eq:tech7_1}).

\textbf{Proof that (\ref{eq:tech7_2}) is satisfied with high probability}\\
The proof of (\ref{eq:tech7_1}) follows the same steps as the proof of (\ref{eq:tech7_1}), using the bound (\ref{eq:tech6_2}) of Lemma \ref{lemma:tech6} to bound the shadow price $\mu^\star$. We omit the details of the proof to avoid repetition.
\end{proof}

\end{document}